\documentclass[twoside,11pt]{article}
\usepackage{amsmath,amsthm}
\usepackage[preprint,nohyperref]{jmlr2e}
\usepackage{graphicx,booktabs,array,float,wrapfig}
\usepackage{algorithm,algorithmic}
\usepackage{amsmath,amssymb,amsfonts,amsthm,mathtools,bm,mathrsfs}
\usepackage{xcolor}
\usepackage{placeins}
\usepackage{needspace}
\usepackage[hyperfootnotes=false]{hyperref}
\usepackage{url}
\hypersetup{colorlinks=true,linkcolor=blue,citecolor=blue,urlcolor=blue,
  pdftitle={Chaos Is a LADDER: Domain Generalization beyond Invariance via Reweighting},
  pdfauthor={Yuhang Jiang, Fengchuan Zhang, Sanguo Zhang, Guojun Zhu}}
\DeclareMathOperator*{\argmin}{arg\,min}

\newtheorem{Lem}{Lemma}
\newtheorem{Con}{Condition}
\newtheorem{Mth}{Theorem}
\newtheorem{Rem}{Remark}
\newtheorem{Cor}{Corollary}
\newtheorem{Pro}{Proposition}
\newtheorem{assumption}{Assumption}
\newcommand{\method}{\textnormal{\textsc{LADDER}}}

\newsavebox{\jmlrtablebox}
\newcommand{\fitwidetable}[1]{%
  \sbox{\jmlrtablebox}{#1}%
  \ifdim\wd\jmlrtablebox>\linewidth
    \resizebox{\linewidth}{!}{\usebox{\jmlrtablebox}}%
  \else
    \usebox{\jmlrtablebox}%
  \fi}
\title{Chaos Is a LADDER: Domain Generalization beyond Invariance via Reweighting}
\author{%
Yuhang Jiang\textsuperscript{1,2},
\name Fengchuan Zhang\textsuperscript{1,2},
\name Sanguo Zhang\textsuperscript{1,2},\\[3pt]
\name Guojun Zhu\textsuperscript{1,2}\thanks{Corresponding author.}\\[8pt]
\addr\textsuperscript{1}School of Mathematical Sciences,\\
\addr University of Chinese Academy of Sciences, Beijing, China\\[3pt]
\addr\textsuperscript{2}Key Laboratory of Big Data Mining and Knowledge Management,\\
\addr Chinese Academy of Sciences, Beijing, China
}
\date{}
\ShortHeadings{Chaos Is a LADDER}{Jiang, Zhang, Zhang and Zhu}
\firstpageno{1}

\begin{document}
\maketitle
\begin{abstract}
Domain generalization (DG) aims to learn from multiple source domains and generalize to unseen target domains. Most DG methods pursue invariance: they seek a causal representation whose prediction rule is invariant across domains. This principle is effective when the causal mechanism is stable, but becomes restrictive when the domain itself modulates how causal content maps to the response. In this case, directly feeding domain style into the predictor can create misleading shortcuts, since style does not by itself cause the response. Yet the apparent chaos of multiple styles can become a ladder: style can locate the unseen target domain among source domains and guide which domain-dependent prediction rules should be trusted. We propose \emph{Latent Adaptive Domain Disentanglement and Environment Reweighting} (LADDER), a fixed-model DG pipeline that learns causal/style representations, freezes the encoders, fits source-specific classifiers, and uses an unlabeled target-domain covariate set only at inference to compute weights over these fixed classifiers, with no target labels or model-state updates. We establish theoretical guarantees for source reweighting and validate LADDER on simulations, FMoW, and a location-grouped iWildCam protocol, with gains in overall and group-averaged accuracy.
\end{abstract}

\section{Introduction}

Domain generalization (DG) aims to learn predictors from multiple source domains that generalize to an unseen target domain \citep{blanchard2011domain,zhou2022survey}. This challenge arises in many applications, where the deployment environment may differ from the training environments in background, acquisition condition, geographic region, time, sensor, or population composition \citep{rothenhausler2021anchor}. For example, iWildCam images across camera-trap locations share animal categories but vary in habitat and species prevalence \citep{beery2020iwildcam}. In this setting, empirical risk minimization (ERM) over pooled source data learns a single predictor across domains, which can be suboptimal when the target domain follows a different prediction rule. The challenge is therefore both to suppress spurious domain-specific correlations and to determine when domain information should guide prediction in an unseen domain \citep{zhu2025twoerms}.

A dominant line of DG research addresses this challenge through invariance. The aim is to learn a causal representation whose relationship with the label remains invariant across source domains and transfers to unseen domains. In causal language, this line often aims to extract a domain-invariant causal representation, denoted by $\bm{Z}_c$, while removing the domain-specific style representation, denoted by $\bm{Z}_s^{(e)}$ for domain $e$ \citep{peters2016invariant,rothenhausler2021anchor}. Representative approaches include invariant risk minimization and closely related invariance regularizers \citep{arjovsky2019irm,krueger2021rex,ahuja2021iib,rame2022fishr}. Although these methods differ in implementation, they share the same target: style representation is treated mainly as a nuisance factor, and the purpose of modeling it is to learn a better invariant causal representation. This influential view is appropriate when the prediction rule is invariant across domains.

\begin{wrapfigure}{R}{0.45\linewidth}
    \centering
    \includegraphics[width=0.85\linewidth]{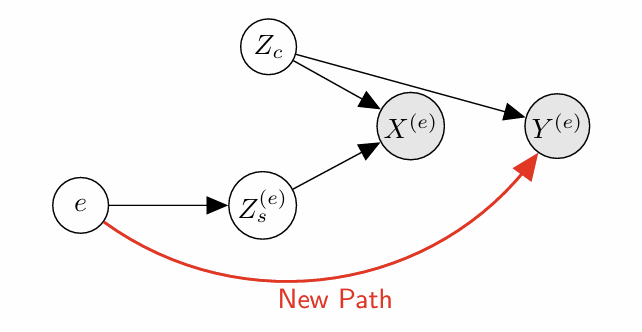}
    \caption{General structural causal model. For domain index $e$, $e\to Y^{(e)}$ denotes variation in $\eta_e(Y\mid \bm{Z}_c)$. \method{} uses domain-level style summaries for source weighting, not style as a per-sample classifier input.}
    \label{fig:scm_new_path}
\end{wrapfigure}

However, the invariance assumption can be too restrictive for broader DG tasks. Recent studies question whether invariant methods consistently outperform strong ERM baselines, especially on realistic benchmarks with complex source-target shifts \citep{gulrajani2021search,zhu2025twoerms}. One reason is that the usual invariant formulation does not distinguish per-sample style features from domain-level style distributions. Figure~\ref{fig:scm_new_path} illustrates this distinction. The per-sample style representation $\bm{Z}_s^{(e)}$ may not directly cause the label $Y^{(e)}$, and therefore feeding style into a pooled classifier can create shortcuts that impair generalization. At the same time, the conditional prediction rule $\eta_e$, i.e. $Y^{(e)}\sim\eta_e(\cdot\mid \bm{Z}_c)$, may vary across domains, so a single predictor may be insufficient. The covariate $\bm{X}^{(e)}$ is generated from both the causal representation $\bm{Z}_c$ and the style representation $\bm{Z}_s^{(e)}$, while the response $Y^{(e)}$ depends on $\bm{Z}_c$ through a domain-dependent rule. The invariant-rule setting is recovered as a special case when this $e\to Y^{(e)}$ path collapses to a common prediction rule across all domains. In this general case, domain-level style distributions guide the selection of domain-dependent prediction rules, with causal content providing the per-sample predictor input. For example, an iWildCam cloudy-forest target domain may have a prediction rule closer to that of a forest-like source domain, without using style as a direct label shortcut.

This perspective connects to work retaining useful domain-specific information. Environment-adaptive covariate selection studies when apparently spurious correlations may help, while structured DG exploits known domain groups and low-rank parameter structure \citep{zuo2026environment,li2023multidimensional}. Domain adaptation and test-time adaptation include approaches that update model parameters using target data \citep{zhang2021udasurvey,xiao2024ttasurvey}. In our DG protocol, unlabeled target covariates determine weights over source classifiers while all fitted model parameters remain fixed.
Closer to our mechanism, Best Sources Forward uses
per-sample source-membership probabilities to fuse
source-specific classifiers; DAEL collaboratively trains
source heads; Gated Domain Units form observation-dependent
ensembles; and DSNR maps an unlabeled target distribution to
a domain-specific nonparametric predictor
\citep{mancini2018best,zhou2021dael,foll2023gdu, zheng2026dsnr}. We propose \emph{Latent Adaptive Domain Disentanglement and Environment Reweighting} (LADDER), which learns
causal/style encoders, fits source classifiers, caches source
style distributions, and computes one Sinkhorn-KNN weight vector from an unlabeled target set. LADDER routes through style-distribution distances; standard MoE uses a learned routing gate \citep{shazeer2017outrageously}.

Our contributions are three-fold. First, we propose LADDER, which estimates the target prediction rule by locally combining source-specific classifiers according to the target style distribution. Second, we establish theoretical guarantees for source reweighting and quantify how source coverage, finite-sample estimation, and representation error affect target prediction. Third, we validate LADDER on simulations and two real-world DG benchmarks, showing strong overall and group-averaged accuracy under shared-stack routing controls.

\begin{figure}[!htb]
\centering
\includegraphics[width=0.9\textwidth]{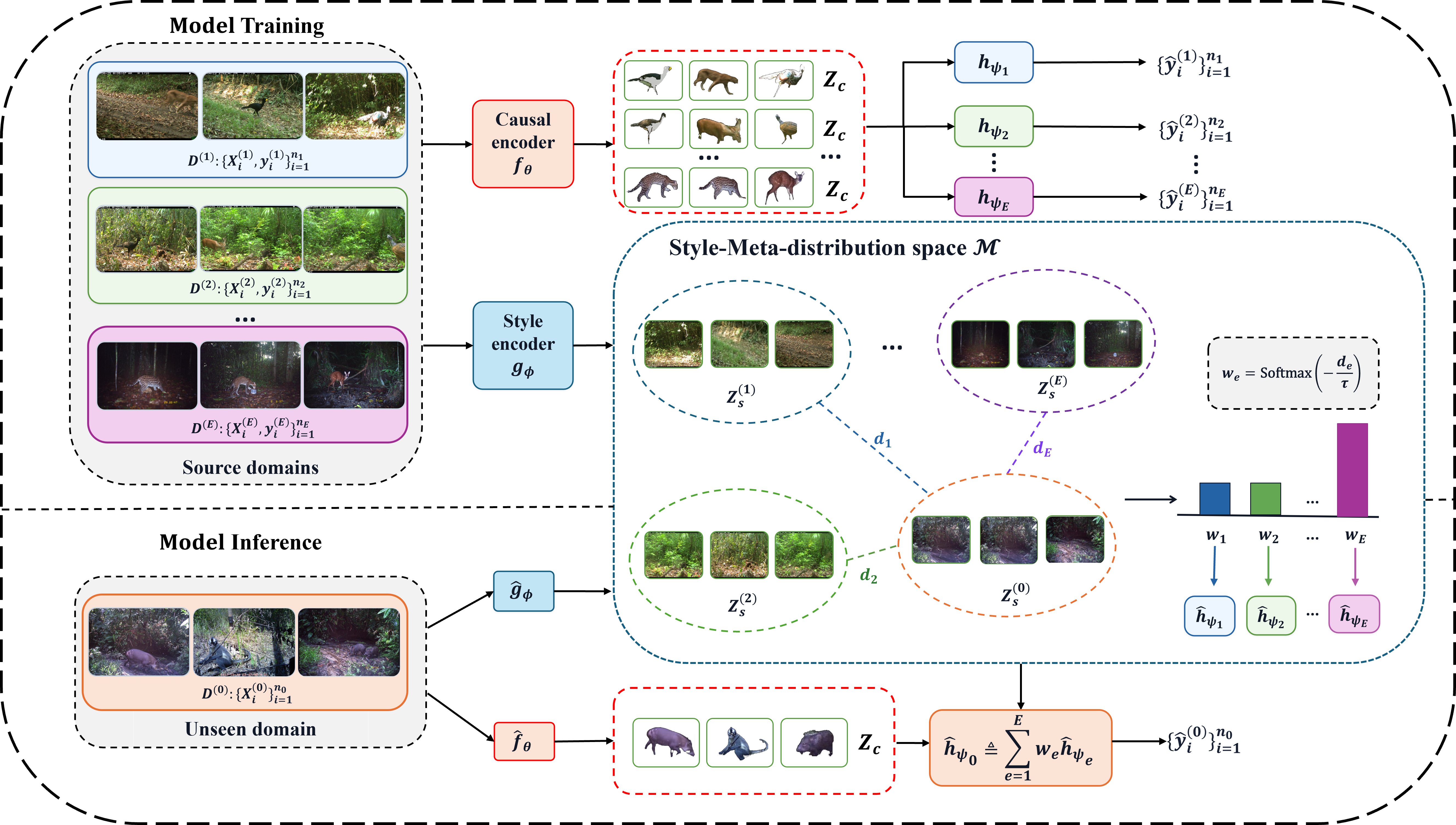}
\caption{Overview of LADDER.}
\label{fig:ladder_pipeline}
\end{figure}

\section{Methodology}\label{sec:methodology}
Figure~\ref{fig:ladder_pipeline} illustrates \method{} training and inference. Source domains are encoded into causal representations for source-specific classifiers and style representations for source-domain fingerprints. At prediction, the fixed style encoder maps unlabeled target covariates to a style summary, compares it with stored source fingerprints, and changes only weights over fixed source classifiers.

\subsection{Problem Set-up}\label{sec:problem_setup}
We consider a DG problem in which source domains are indexed by $e\in\{1,\ldots,E\}$ and an unseen target domain is indexed by $0$. Each domain induces a joint distribution $P_e$ over observations and labels $(\bm{X},Y)\in\mathcal{X}\times\mathcal{Y}$. At training time we observe labeled samples $\mathcal{D}^{(e)}=\{(\bm{x}_i^{(e)},y_i^{(e)})\}_{i=1}^{n_e}$ from the source domains. At inference, an unlabeled target covariate set $\mathcal{D}^{(0)}=\{\bm{x}_i^{(0)}\}_{i=1}^{n_0}$ is used only to estimate a style summary and source weights; encoders, classifiers, and normalization statistics remain fixed. Following the meta-distribution view of DG \citep{blanchard2011domain,zhou2022survey}, we treat domains as points in a style-distribution space $\mathcal{M}$ and let $\Pi_{\mathcal{M}}$ denote their meta-distribution.

The key structural assumption is that each observation admits two learned roles, denoted $\bm{Z}_c$ and $\bm{Z}_s^{(e)}$. The causal and style representations are optimized for label prediction and source reweighting respectively; this division requires no statistical independence between representations or between style and labels. For a domain $e$, let $\nu_e\in\mathcal{P}(\mathcal{Z}_s)$ be its domain-level style distribution on the style space $\mathcal{Z}_s$; its empirical estimate $\widehat{\nu}_e$ is the style fingerprint used below. We write observations and labels abstractly as
\[
\bm{X}^{(e)}
=G(\bm{Z}_c,\bm{Z}_s^{(e)},\bm{\varepsilon}_x),
\qquad
Y^{(e)}\sim\eta_e(\cdot\mid \bm{Z}_c),
\]
where $\bm{\varepsilon}_x$ denotes noise, $G$ is a shared observation mechanism, and $\eta_e$ is the domain-dependent conditional prediction rule.

The smoothness principle underlying \method{} is that domains close in style distribution should have similar prediction rules:
\begingroup
\begin{equation}
    d_{\mathcal{F}}(\eta_e,\eta_{e'})\le L_\eta\,\mathcal{W}_1(\nu_e,\nu_{e'}),
    \label{eq:meta_smoothness_in_method}
\end{equation}
\endgroup
where $d_{\mathcal{F}}$ measures task-level distance between conditional rules, $\mathcal{W}_1$ is the 1-Wasserstein distance, and $L_\eta$ is a smoothness constant. In the parametric example, $d_{\mathcal{F}}$ is the parameter Frobenius distance.

A pooled predictor $\widehat y=h(f_{\bm\theta}(x),g_{\bm\phi}(x))$ may use style as a shortcut. \method{} computes source weights from distances between the target style fingerprint $\widehat\nu_0=n_0^{-1}\sum_i\delta_{g_{\widehat{\bm\phi}}(x_i^{(0)})}$ and cached source fingerprints, where $\delta_{\bm{s}}$ denotes a Dirac point mass. It averages source-classifier logits on $\bm Z_c$ before Softmax (Eq.~\eqref{eq:dynamic_prediction}), yielding distribution-dependent model averaging \citep{arpit2022ensemble}.

\subsection{Model Training: Encoders and Classifiers}\label{sec:representation_learning}
Following disentanglement-based DG and information-bottleneck learning \citep{alemi2017deep,zhang2022principled}, model training learns $\bm{z}_c=f_{\bm{\theta}}(\bm{x})$ for label prediction and $\bm{z}_s=g_{\bm{\phi}}(\bm{x})$ for source reweighting, then fits source classifiers. We optimize:
\begingroup
\begin{equation}
\min_{\Theta}\mathcal{J}=\mathcal{L}_{\mathrm{cls}}+\lambda_s\mathcal{L}_{\mathrm{sd}}+\lambda_a\mathcal{L}_{\mathrm{cd}}^{\mathrm{GRL}}+\lambda_o\mathcal{L}_{\mathrm{orth}}+\lambda_r\mathcal{L}_{\mathrm{reg}}.
\label{eq:stage1_objective}
\end{equation}
\endgroup
\begingroup
Here $\Theta$ collects the representation-stage parameters and $\lambda_s,\lambda_a,\lambda_o,\lambda_r$ are fixed loss weights. Equation~\eqref{eq:stage1_objective} combines auxiliary label classification,
style-domain classification, adversarial suppression of domain information in $\bm Z_c$, cross-covariance orthogonality, and representation
regularization. Let $\rho_y$ denote the shared auxiliary label classifier,
and let $\rho_s$ and $\rho_c$ denote the auxiliary domain
classifiers on $\bm{Z}_s^{(e)}$ and $\bm Z_c$, respectively.
Writing $\widehat{\mathbb E}_e$ for the empirical average
over source domain $e$ and $\ell_{\mathrm{CE}}(\bm u,y)=-\log(\mathrm{Softmax}(\bm u)_y)$ for
cross-entropy, $\mathcal L_{\mathrm{cls}}=E^{-1}\sum_e\widehat{\mathbb E}_e\ell_{\mathrm{CE}}(\rho_y(f_{\bm\theta}(x)),y)$,
$\mathcal L_{\mathrm{sd}}
=E^{-1}\sum_e\widehat{\mathbb E}_e
\ell_{\mathrm{CE}}(\rho_s(g_{\bm\phi}(x)),e)$, and
$\mathcal L_{\mathrm{cd}}^{\mathrm{GRL}}
=E^{-1}\sum_e\widehat{\mathbb E}_e
\ell_{\mathrm{CE}}(\rho_c(R_{\lambda_a}(f_{\bm\theta}(x))),e)$,
where $R_{\lambda_a}$ is a gradient-reversal layer.

We use
$\mathcal L_{\mathrm{orth}}
=(d_cd_s)^{-1}\|\widehat{\operatorname{Cov}}
(f_{\bm\theta}(x),g_{\bm\phi}(x))\|_F^2$
to discourage redundant encoding; $d_c,d_s$ denote the branch dimensions. The generic regularizer
$\mathcal L_{\mathrm{reg}}$ may, for example, be a
variational-bottleneck Kullback--Leibler (KL) term
$\sum_{r\in\{c,s\}}d_r^{-1}D_{\mathrm{KL}}(q_r(z_r\mid x)\|p_r)$, where $q_r$ and $p_r$ are the branch posterior and prior,
but LADDER does not depend on this choice.
We then freeze both encoders and fit each final source classifier on fixed causal representations by
\[
\widehat{\bm{\psi}}_e
\in\arg\min_{\bm{\psi}_e}
\widehat{\mathbb E}_e
\ell_{\mathrm{CE}}(h_{\bm{\psi}_e}(f_{\widehat{\bm{\theta}}}(x)),y).
\]
\par
\endgroup

\subsection{Model Inference: Reweighted Ensemble}\label{sec:environment_reweighted_ensemble}
For each source domain $e=1,\ldots,E$, cache the low-dimensional style fingerprint
\begin{equation}
\widehat{\nu}_e=\frac{1}{n_e}\sum_{i=1}^{n_e}\delta_{g_{\widehat{\bm{\phi}}}(\bm{x}_i^{(e)})}.
\label{eq:style_empirical_measure}
\end{equation}
After these fingerprints are cached, source covariates need not be stored. For a new target domain, construct only the target style summary $\widehat{\nu}_0$ from target covariates, then compute $d_e$ for each source domain. Theory uses $d_e=\mathcal{W}_1(\widehat{\nu}_e,\widehat{\nu}_0)$; real-data runs use a balanced debiased Sinkhorn approximation with Euclidean ground cost (Appendix~\ref{app:implementation_details}). The rates exclude entropic and solver errors. Let $\widehat{\mathcal{N}}_K(0)$ be the $K$ source domains with the smallest distances. The target-adaptive weights are
\begin{equation}
 w_e=
 \frac{\exp(-d_e/\tau)}{\sum_{j\in\widehat{\mathcal{N}}_K(0)}\exp(-d_j/\tau)},
 \quad \text{for every } e\in\widehat{\mathcal{N}}_K(0),
\label{eq:dynamic_weight}
\end{equation}
where $1\le K\le E$ is the neighbor count and $\tau>0$ is the temperature. Set $w_e=0$ outside $\widehat{\mathcal N}_K(0)$. With few target covariates or a noisy target fingerprint,
larger $\tau$ and source neighborhoods make the rule approach
a global source-weighted average. Each source classifier outputs pre-softmax logits, and the final target logit averages these logits:
\begin{equation}
\widehat{\bm{\ell}}(\bm{x})
=
\sum_{e\in\widehat{\mathcal{N}}_K(0)}
w_e\,h_{\widehat{\bm{\psi}}_e}\!\big(f_{\widehat{\bm{\theta}}}(\bm{x})\big).
\label{eq:dynamic_prediction}
\end{equation}
Labels are predicted from $\mathrm{Softmax}(\widehat{\bm{\ell}}(\bm{x}))$; for generalized linear model (GLM) classifiers, Eq.~\eqref{eq:dynamic_prediction} equals the parameter average analyzed below. The weights $w_e$ are computed once from $\widehat{\nu}_0$ and shared within the target domain. Algorithm~\ref{alg:ladder_inference} summarizes the inference stage. Runtime and storage are reported in Appendix Tables~\ref{tab:supp_runtime} and~\ref{tab:supp_storage}.

\begin{algorithm}[H]
\caption{\method{} Inference Stage}
\label{alg:ladder_inference}
\begin{algorithmic}[1]
\setlength{\itemsep}{0pt}
\REQUIRE Fixed $f_{\widehat{\bm{\theta}}},g_{\widehat{\bm{\phi}}}$; cached source fingerprints $\{\widehat{\nu}_e\}_{e=1}^E$; source classifiers $\{h_{\widehat{\bm{\psi}}_e}\}_{e=1}^E$; target covariates $\mathcal{D}^{(0)}$; neighbor count $K$; temperature $\tau$.
\STATE Encode target styles and form $\widehat{\nu}_0=n_0^{-1}\sum_{i=1}^{n_0}\delta_{g_{\widehat{\bm{\phi}}}(\bm{x}_i^{(0)})}$.
\STATE Compute style-summary distances $d_e$ for $e=1,\ldots,E$.
\STATE Select $\widehat{\mathcal{N}}_K(0)$ as the $K$ nearest source domains.
\STATE Compute weights $\{w_e\}_{e\in\widehat{\mathcal{N}}_K(0)}$ by Eq.~\eqref{eq:dynamic_weight}.
\STATE Predict with fixed source classifiers by Eq.~\eqref{eq:dynamic_prediction}.
\RETURN Target predictions $\{\widehat{y}(\bm{x}):\bm{x}\in\mathcal{D}^{(0)}\}$.
\end{algorithmic}
\end{algorithm}

\section{Theoretical Guarantees}\label{sec:Theory}
With fixed causal/style encoders and source classifiers, we analyze how well an unlabeled target style summary selects nearby source rules. The oracle analysis uses oracle representations; the generic analysis uses learned ones. 
Dimensions $d_c,d_s,d_{\mathcal M}$ denote causal, style, and style-manifold dimensions; $C=|\mathcal Y|$, $d_\beta=(C-1)d_c$, $n_{\min}:=\min_{1\le e\le E}n_e$, and $a\vee b=\max\{a,b\}$.

\subsection{Oracle Representation Guided Theory}\label{sec:oracle_representation_theory}

For domain $e$, use the style distribution $\nu_e$ defined in Section~\ref{sec:problem_setup}, and let $\bm{\beta}_*^{(e)}$ be the oracle source parameter. For the oracle analysis, labels obey $\mathbb{P}(Y\mid \bm{Z}_c,e)=\mathrm{Softmax}(\bm{\beta}_*^{(e)}\bm{Z}_c)$ with $\bm{\beta}_*^{(e)}\in\mathbb{R}^{C\times d_c}$. For observation $i$ in domain $e$, let $\bm z_{c,i}^{(e)}$ and $\bm z_{s,i}^{(e)}$ denote its oracle causal and style features. Set $\widehat\nu_e=n_e^{-1}\sum_i\delta_{\bm z_{s,i}^{(e)}}$; for $e=1,\ldots,E$, $\widehat{\bm\beta}^{(e)}\in\argmin_{\bm\beta\in\mathcal B}n_e^{-1}\sum_i\ell_{\mathrm{CE}}(\bm\beta\bm z_{c,i}^{(e)},y_i^{(e)})$ is the source empirical-risk minimizer. We state compact assumptions here and give their full versions in the appendices.

\begin{assumption}[Oracle latent structure]\label{ass:linear}\label{ass:identifiability}\label{ass:label_completeness}
Observations admit a centered linear oracle decomposition, $\bm{X}^{(e)}=\bm{A}_c\bm{Z}_c^{(e)}+\bm{A}_s\bm{Z}_s^{(e)}+\bm{\epsilon}^{(e)}$. The causal/noise covariances are domain-invariant, while $\bm{\Lambda}_e=\mathrm{Cov}(\bm{Z}_s^{(e)})$ is diagonal with jointly separating normalized cross-domain variance profiles.
\end{assumption}

\begin{assumption}[Smooth geometry]\label{ass:lipschitz_param}\label{ass:style_compactness}\label{ass:manifold_regularity}
There is an $L_\beta$-Lipschitz map $\Psi^*$ from domain-level style distributions to GLM parameters, so that $\bm{\beta}_*^{(e)}=\Psi^*(\nu_e)$ and nearby style distributions have nearby prediction rules under $\mathcal{W}_1$. Style laws are compactly supported; sources are i.i.d. from a regular $d_{\mathcal{M}}$-dimensional space, with the target sampled independently.
\end{assumption}

\begin{assumption}[GLM regularity]\label{ass:boundedness}\label{ass:hessian}\label{ass:centered_logits}
Causal features satisfy $\|\bm Z_c\|_2\le R_c$ almost surely,
and GLM parameters lie in a compact convex set $\mathcal B\subset\{\bm\beta:\bm 1^\top\bm\beta=\bm 0\}$ with $\|\bm\beta\|_F\le R_\beta$ and relative-interior oracle optima.
For $\mathcal{R}_e(\bm{\beta})=\mathbb{E}[\ell_{\mathrm{CE}}(\bm{\beta}\bm Z_c,Y)\mid e]$,
these population cross-entropy risks are uniformly $\mu$-strongly convex, for some $\mu>0$, on
$\mathcal B$. Single-sample Hessians are uniformly bounded
and $L_H$-Lipschitz in $\bm{\beta}$ over $\mathcal B$; and
centered-logit risks are locally $\mu_\ell$-strongly convex
on bounded comparison logits and globally $L_\ell$-smooth.
\end{assumption}

These compact assumptions summarize style-covariance profile separability, smooth rule variation over style distributions, finite-sample source estimation, and empirical Wasserstein concentration.
\begin{Mth}\label{thm:identifiability}
Under Assumption~\ref{ass:linear}, let $\bm{W}_s$ be a population style-extraction matrix, $\widehat{\bm Z}_s^{(e)}=\bm{W}_s\bm X^{(e)}$, and $\bm{M}:=\bm{W}_s\bm{A}_s$ be invertible. Suppose the extractor diagonalizes the domain-varying style covariance, i.e., $\bm{M}\bm{\Lambda}_e\bm{M}^\top$ is diagonal for every source domain $e=1,\ldots,E$. Then $\bm{M}=\bm{P}\bm{D}$ for a permutation matrix $\bm P$ and nonsingular diagonal matrix $\bm D$, and $\widehat{\bm{Z}}_s^{(e)}=\bm{B}\bm{Z}_c^{(e)}+\bm{P}\bm{D}\bm{Z}_s^{(e)}+\widetilde{\bm{\epsilon}}^{(e)}$, where $\bm{B}=\bm{W}_s\bm{A}_c$ and $\widetilde{\bm{\epsilon}}^{(e)}=\bm{W}_s\bm{\epsilon}^{(e)}$. Thus $\bm{W}_s\bm{A}_s$ is identified up to permutation/scaling, while the extracted representation may contain causal leakage. Moreover, under supplemental mean-separation/completeness conditions, $\bm{B}=0$.
\end{Mth}

\begin{Mth}[Target Parameter Error]\label{thm:param_bound}
Under Assumptions~\ref{ass:lipschitz_param}--\ref{ass:boundedness}, and $n_{\min}\gtrsim\mu^{-2}\{d_\beta\log(1+R_\beta L_H/\mu)+\log(d_\beta E)\}$ with $d_\beta=(C-1)d_c$, the reweighted estimator $\widehat{\bm{\beta}}^{(0)}=\sum_{e\in\widehat{\mathcal{N}}_K(0)}w_e\widehat{\bm{\beta}}^{(e)}$ satisfies
\begin{equation*}
\mathbb{E}\|\widehat{\bm{\beta}}^{(0)}-\bm{\beta}_*^{(0)}\|_F
\le \mathcal{O}\!\left(\frac{R_c}{\mu}\sqrt{\frac{d_\beta+\log E}{n_{\min}}}+\gamma(n_0,d_s)+\Gamma(n_{\min},E,d_s)+(K/E)^{1/d_{\mathcal M}}\right).
\end{equation*}
Here $\gamma(n,d)=n^{-1/(d\vee2)}\ell_d(n)$ and $\Gamma(n,E,d)=((1+\log E)/n)^{1/(d\vee2)}\widetilde{\ell}_d(n)$, where $\ell_2(n)=\widetilde{\ell}_2(n)=\log(1+n)$ and $\ell_d(n)=\widetilde{\ell}_d(n)=1$ for $d\ne2$ \citep{fournier2015rate}.
\end{Mth}

\begin{Rem}[Source coverage trade-off]\label{rem:trade_off}
	Theorem~\ref{thm:param_bound}'s terms have three direct roles: $\frac{R_c}{\mu}\sqrt{(d_\beta+\log E)/n_{\min}}$ comes from fitting fixed source classifiers, $\gamma(n_0,d_s)+\Gamma(n_{\min},E,d_s)$ comes from estimating target/source fingerprints, and $(K/E)^{1/d_{\mathcal M}}$ measures the distance from the target style law to the selected source neighborhood. More sources reduce the interpolation term when $n_{\min}$ is large enough to keep the worst source fingerprint reliable; otherwise, empirical distribution error dominates. For fixed dimensions, this bound vanishes when $n_0\to\infty$, $E\to\infty$, $K/E\to0$, and $(1+\log E)\widetilde{\ell}_{d_s}(n_{\min})^{d_s\vee2}=o(n_{\min})$.
\end{Rem}

\begin{Mth}[Target Excess Risk]\label{thm:excess_risk}
Under the assumptions of Theorem~\ref{thm:param_bound},
\begin{equation*}
\begin{aligned}
&\mathbb{E}\{\mathcal{R}_0(\widehat{\bm{\beta}}^{(0)})-\mathcal{R}_0(\bm{\beta}_*^{(0)})\}\\
&\quad\le\mathcal{O}\!\left(R_c^2\left[\frac{R_c}{\mu}\sqrt{\frac{d_\beta+\log E}{n_{\min}}}+\gamma(n_0,d_s)+\Gamma(n_{\min},E,d_s)+(K/E)^{1/d_{\mathcal M}}\right]^2\right).
\end{aligned}
\end{equation*}
\end{Mth}
Theorem~\ref{thm:excess_risk} turns the parameter bound into excess risk: bounded causal features scale logit error by $R_c$, and local smoothness gives a quadratic rate in the sum of the source-estimation, empirical-distribution, and interpolation terms. 

To compare \method{} with PooledJoint and a fixed prediction rule motivated by IRM's invariant-prediction principle \citep{arjovsky2019irm}, we consider a special case in which a scalar coordinate $c\sim\pi_c$ represents latent domain position. Define
\[
f^*(\bm{Z}_c,c)=(\bm{\beta}_{\mathrm{base}}+c\bm{\beta}_{\mathrm{shift}})\bm{Z}_c,\qquad \bm{\Pi}_C=\bm{I}_C-C^{-1}\bm{1}\bm{1}^\top,
\]
and for any predictor $f$, $\mathcal{R}_c(f)
=\mathbb{E}[\ell_{\mathrm{CE}}(f(\bm Z_c,\bm Z_s),Y_c)]$,
where $Y_c\sim\mathrm{Softmax}(f^*(\bm Z_c,c))$. The bounded additive pooled class is
\[
\mathcal{H}_{\mathrm{pool}}(R_W)=\{\,f=\bm{W}_c\bm{Z}_c+\bm{W}_s\bm{Z}_s:\ \|\bm{W}_c\|_F,\|\bm{W}_s\|_F\le R_W\,\}.
\]
For target coordinate $c_0$, write $f_{\mathrm{inv}}=f^*(\cdot,c_{\mathrm{inv}})$ for fixed prediction, where $c_{\mathrm{inv}}$ does not depend on the target domain
and satisfies $c_{\mathrm{inv}}\ne c_0$. Write $f_{\mathrm L}=f^*(\cdot,\widehat c_{\mathrm L})$ for oracle \method{}, where $\widehat c_{\mathrm L}=\sum_e w_e c_e$ and $c_e$ is the coordinate of source $e$.

\begin{assumption}[Comparison shift model]\label{ass:multiplicative_shift}\label{ass:static_anchor}
Let $c$ be a bounded scalar domain coordinate with $\operatorname{Var}_{\pi_c}(c)>0$. The causal representation is centered with $\mathbb{E}\bm{Z}_c\bm{Z}_c^\top=\bm{\Sigma}_c\succ0$ and is independent of $(c,\bm{Z}_s)$. For the target coordinate, $c_{\mathrm{inv}}\ne c_0$.
\end{assumption}

\begin{Cor}[Pooled-joint misspecification]\label{cor:pooled_misspec}
Under Assumptions~\ref{ass:boundedness} and~\ref{ass:multiplicative_shift},
\begin{equation*}
\inf_{f\in\mathcal{H}_{\mathrm{pool}}(R_W)}\mathbb{E}_{c\sim\pi_c}\{\mathcal{R}_c(f)-\mathcal{R}_c(f^*)\}
\ge\frac{\mu_\ell}{2}\operatorname{Var}_{c\sim\pi_c}(c)\|\bm{\Pi}_C\bm{\beta}_{\mathrm{shift}}\bm{\Sigma}_c^{1/2}\|_F^2.
\end{equation*}
\end{Cor}
Under nontrivial rule variation, every additive pooled predictor has positive domain-averaged excess risk: a style-dependent offset cannot modulate the causal decision matrix.

\begin{Cor}[Comparison with fixed prediction]\label{cor:ladder_irm_comparison}
Under Assumptions~\ref{ass:boundedness} and~\ref{ass:multiplicative_shift}, if $\|\bm{\Pi}_C\bm{\beta}_{\mathrm{shift}}\bm{\Sigma}_c^{1/2}\|_F>0$, then for target coordinate $c_0$,
\begin{equation*}
\frac{\mathcal{R}_{c_0}(f_{\mathrm L})-\mathcal{R}_{c_0}(f^*)}
{\mathcal{R}_{c_0}(f_{\mathrm{inv}})-\mathcal{R}_{c_0}(f^*)}
\le
\frac{L_\ell}{\mu_\ell}
\frac{(c_0-\widehat c_{\mathrm L})^2}{(c_0-c_{\mathrm{inv}})^2}.
\end{equation*}
Consequently, within this comparison model, \method{} has
lower target excess risk than the fixed predictor
whenever
$|c_0-\widehat c_{\mathrm L}|
<\sqrt{\mu_\ell/L_\ell}|c_0-c_{\mathrm{inv}}|$.
\end{Cor}
Corollary~\ref{cor:ladder_irm_comparison} gives a sufficient condition for oracle source reweighting to improve on the fixed rule \citep{peters2016invariant,rosenfeld2021risks}. Under the consistency regime in Remark~\ref{rem:trade_off}, Theorem~\ref{thm:excess_risk} gives vanishing expected excess risk for the estimated predictor, whereas this fixed rule retains positive excess risk when $c_{\mathrm{inv}}\ne c_0$ and the centered rule shift is nonzero.

\subsection{Generic Representation Guided Theory}\label{sec:generic_representation_theory}
We next quantify representation error in the parameter and excess risk bounds.

After a fixed coordinate alignment, for $e=0,\ldots,E$ and $i=1,\ldots,n_e$, write the learned representations as
$\widetilde{\bm{z}}_{c,i}^{(e)}=\bm{z}_{c,i}^{(e)}+\bm{\xi}_{c,i}^{(e)}$ and
$\widetilde{\bm{z}}_{s,i}^{(e)}=\bm{z}_{s,i}^{(e)}+\bm{\xi}_{s,i}^{(e)}$,
where $\bm\xi_{c,i}^{(e)}$ and $\bm\xi_{s,i}^{(e)}$ are representation errors.

\begin{assumption}[Second-moment representation error]\label{ass:generic_representation_error}
For a representation-error scale $\epsilon_{\mathrm{rep}}\ge0$, uniformly over domains and observations, $\mathbb E\|\bm\xi_{c,i}^{(e)}\|_2^2\vee\mathbb E\|\bm\xi_{s,i}^{(e)}\|_2^2\le\epsilon_{\mathrm{rep}}^2$.
For some $\mu_{\mathrm{rep}}>0$, each perturbed source empirical risk is $\mu_{\mathrm{rep}}$-strongly convex on $\mathcal B$.
For a fresh target observation, conditional on the fitted objects, oracle representations, and representation errors, the label law remains $\mathrm{Softmax}(\bm\beta_*^{(0)}\bm Z_c^{(0)})$.
\end{assumption}

Write $\widetilde{\mathcal R}_0$ for target cross-entropy risk on learned causal features, evaluated on a fresh observation, and $\mathcal R_0$ for the corresponding risk on oracle features. Set $L_{\nabla z}=\sqrt2+\tfrac12R_\beta R_c$.

\begin{Cor}[Generic Representation Bound]\label{cor:generic_representation}
Under the assumptions of Theorem~\ref{thm:param_bound} and Assumption~\ref{ass:generic_representation_error}, let $E$ be fixed. For $e=1,\ldots,E$, let $\widetilde{\bm\beta}^{(e)}$ minimize the perturbed empirical cross-entropy risk on $\mathcal B$, and let $\widetilde\nu_e=n_e^{-1}\sum_{i=1}^{n_e}\delta_{\widetilde{\bm z}_{s,i}^{(e)}}$ for $e=0,\ldots,E$. Compute $\widetilde{\mathcal N}_K(0)$ and weights $\widetilde w_e$ from these measures by Eq.~\eqref{eq:dynamic_weight}, and set
$\widetilde{\bm{\beta}}^{(0)}=\sum_{e\in\widetilde{\mathcal{N}}_K(0)}\widetilde w_e\widetilde{\bm{\beta}}^{(e)}$.
Then the parameter error satisfies
\[
\begin{aligned}
\mathbb{E}\|\widetilde{\bm{\beta}}^{(0)}-\bm{\beta}_*^{(0)}\|_F
&\le \mathcal{O}\!\Biggl(
 \frac{R_c}{\mu}\sqrt{\frac{d_\beta+\log E}{n_{\min}}}
 +\gamma(n_0,d_s)+\Gamma(n_{\min},E,d_s)\\
&\qquad\quad +(K/E)^{1/d_{\mathcal M}}
 +\left(\frac{L_{\nabla z}}{\mu_{\mathrm{rep}}}+L_\beta\right)
 \epsilon_{\mathrm{rep}}\Biggr),
\end{aligned}
\]
and the target excess risk satisfies
\[
\begin{aligned}
\mathbb{E}\{\widetilde{\mathcal{R}}_0(\widetilde{\bm{\beta}}^{(0)})
 -\mathcal{R}_0(\bm{\beta}_*^{(0)})\}
&\le \mathcal{O}\!\Biggl(
 R_c^2\biggl[\frac{R_c}{\mu}
 \sqrt{\frac{d_\beta+\log E}{n_{\min}}}+\gamma(n_0,d_s)\\
&\qquad\qquad +\Gamma(n_{\min},E,d_s)
 +(K/E)^{1/d_{\mathcal M}}\biggr]^2\\
&\qquad\quad +\biggl[
 R_c^2\left(\frac{L_{\nabla z}}{\mu_{\mathrm{rep}}}+L_\beta\right)^2
 +R_\beta^2\biggr]\epsilon_{\mathrm{rep}}^2\Biggr).
\end{aligned}
\]
\end{Cor}
For fixed $E$, representation error enters linearly in parameter error and quadratically in risk; as $\epsilon_{\mathrm{rep}}\to0$, the oracle bounds are recovered, with the source-coverage term retained. The growing-source consistency regime is given in Remark~\ref{rem:trade_off}.

\section{Experiments}\label{sec:experiments}

We evaluate \method{} in simulations, FMoW-WILDS, and iWildCam-WILDS. Simulation results with centered parameter errors and ablations appear in Appendix~\ref{app:simulation}; the main text focuses on real-data evidence.
We compare with source-trained baselines ERM, IRM~\citep{arjovsky2019irm}, EQRM~\citep{eastwood2022qrm}, DSNR~\citep{zheng2026dsnr}, RDM~\citep{nguyen2024rdm}, and ARM-CML~\citep{zhang2021arm}; LADDER routing ablations (uniform weighting, $K=1$, sample gating, and, on iWildCam, group gating); and, in the simulations, PooledJoint and an oracle target classifier. Baseline implementations, target-summary protocols, and metrics are detailed in Appendix~\ref{app:experimental_details}.

\begingroup
\subsection{Real Data 1: FMoW-WILDS}\label{sec:fmow_wilds}
FMoW-WILDS \citep{koh2021wilds,christie2018functional}  is a $62$-way satellite land-use classification benchmark built from Functional Map of the World images, where distribution shifts arise across time and geographic regions. We index source domains by year and test whether style-based reweighting improves overall OOD accuracy and regional robustness. Metrics are test accuracy, worst-/average-region accuracy (Worst/Avg.), and ID-to-test drop.

\begin{table}[!htb]
	\centering
	\caption{FMoW-WILDS: mean $\pm$ SD over five seeds (accuracies in \%; Drop in percentage points).}
	\label{tab:fmow_results}
	{%
		\begin{tabular}{lcccc}
			\toprule
			Method & Test $\uparrow$ & Worst $\uparrow$ & Avg. $\uparrow$ & Drop $\downarrow$ \\
			\midrule
			ERM & $46.68 \pm 0.96$ & $30.68 \pm 1.92$ & $45.49 \pm 1.30$ & $9.14 \pm 0.77$ \\
			IRM & $46.68 \pm 0.54$ & $29.06 \pm 1.79$ & $45.39 \pm 0.74$ & $8.88 \pm 0.73$ \\
			EQRM & $46.85 \pm 0.55$ & $28.89 \pm 1.61$ & $45.15 \pm 0.54$ & $9.09 \pm 0.71$ \\
			DSNR & $45.52 \pm 0.87$ & $27.03 \pm 1.89$ & $43.93 \pm 0.87$ & $9.78 \pm 0.61$ \\
			RDM & $47.15 \pm 1.19$ & $29.56 \pm 1.68$ & $45.22 \pm 1.36$ & $9.27 \pm 0.70$ \\
ARM-CML & $47.89 \pm 0.35$ & $29.53 \pm 1.49$ & $46.32 \pm 0.80$ & $\mathbf{8.56 \pm 0.20}$ \\
			\midrule
			\method{} (uniform) & $49.65 \pm 0.32$ & $30.46 \pm 1.24$ & $48.55 \pm 0.45$ & $8.69 \pm 0.22$ \\
			\method{} ($K=1$) & $49.12 \pm 0.53$ & $30.71 \pm 0.97$ & $48.08 \pm 0.65$ & $8.70 \pm 0.71$ \\
			\method{} (sample gate) & $49.98 \pm 0.31$ & $30.67 \pm 1.18$ & $48.67 \pm 0.58$ & $9.01 \pm 0.25$ \\
			\textbf{LADDER} & $\textbf{50.19} \pm \textbf{0.34}$ & $\textbf{31.46} \pm \textbf{0.74}$ & $\textbf{49.05} \pm \textbf{0.54}$ & $8.87 \pm 0.53$ \\
			\bottomrule
	\end{tabular}}
\end{table}
All methods use ImageNet-pretrained ResNet-50
\citep{he2016deep,deng2009imagenet} and five seeds. ERM, IRM, EQRM, and RDM predict each example independently of other target examples. DSNR estimates one domain index per evaluation split from unlabeled covariates. ARM-CML conditions predictions on a learned mean context from the same split. The LADDER ablations share its frozen encoders and source classifiers: \method{} (uniform) assigns equal weights, \method{} ($K=1$) uses the nearest source, and \method{} (sample gate) uses the frozen style-domain head for per-sample source-membership weights. LADDER and its $K=1$ variant use split-level Sinkhorn style summaries. Following the WILDS evaluator, worst-region accuracy excludes the Other bucket. We compute average-region accuracy over the same nonempty regions.

\begin{figure}[!htb]
\centering
\includegraphics[width=0.86\linewidth]{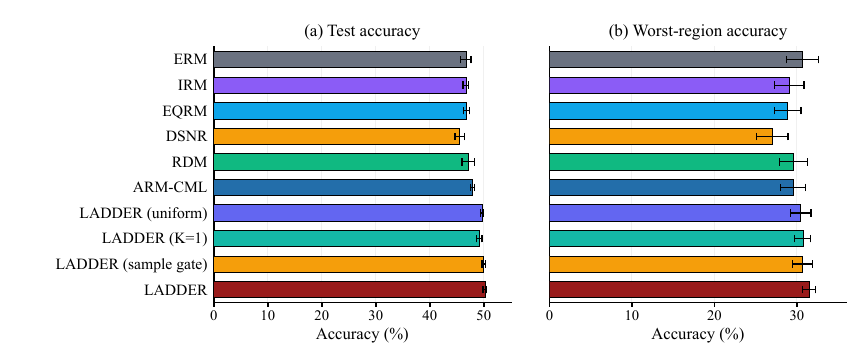}
\caption{FMoW-WILDS. (a) Test accuracy and (b) worst-region accuracy: mean $\pm$ SD over five seeds.}
\label{fig:fmow_main}
\end{figure}

\FloatBarrier
Table~\ref{tab:fmow_results} and Figure~\ref{fig:fmow_main}(a,b) show that the frozen representation and source-classifier stack improves FMoW over DG baselines. LADDER exceeds ARM-CML in test accuracy ($50.19$ versus $47.89$), whereas ARM-CML has a smaller ID-to-test drop ($8.56$ versus $8.87$). \method{} (uniform), \method{} (sample gate), and \method{} are numerically close. The small gains over sample gating have paired-seed intervals that include zero (Appendix Table~\ref{tab:supp_paired_routing}); \method{} ($K=1$) is weaker. Appendix Figure~\ref{fig:supp_fmow_tsne_full} visualizes style representations and Sinkhorn source reweighting. Appendix Figure~\ref{fig:supp_fmow_source_geometry} reports positive associations of style distance with classifier transfer loss and class-prototype distance. Full split-wise results, trade-off/source-weight plots, and t-SNE views appear in Appendix Figures~\ref{fig:supp_fmow_accuracy}, \ref{fig:supp_fmow_tradeoff_weights}, and \ref{fig:supp_fmow_tsne_full}.

\subsection{Real Data 2: iWildCam-WILDS}\label{sec:iwildcam_wilds}
iWildCam-WILDS \citep{koh2021wilds,beery2020iwildcam} stresses a different failure mode: the target split comprises heterogeneous camera-trap locations. It is a $182$-way wildlife species classification benchmark collected from camera traps, with severe long-tailed class and location frequencies. This dataset fits source reweighting because locations may differ in background, illumination, species composition, and sensor style.

\begin{table}[!htb]
	\centering
	\caption{iWildCam-WILDS: mean $\pm$ SD over five seeds (all metrics in \%).}
	\label{tab:iwildcam_results}
	\begin{tabular}{lccc}
		\toprule
		Method & Test $\uparrow$ & Macro F1 $\uparrow$ & Filt. Avg. $\uparrow$ \\
		\midrule
		ERM & $42.26 \pm 5.48$ & $11.23 \pm 1.35$ & $33.29 \pm 7.61$ \\
		IRM & $34.59 \pm 10.11$ & $6.20 \pm 1.08$ & $23.42 \pm 9.71$ \\
		EQRM & $40.75 \pm 6.62$ & $10.45 \pm 1.37$ & $32.22 \pm 8.07$ \\
		DSNR & $47.36 \pm 7.96$ & $12.34 \pm 1.12$ & $39.44 \pm 2.07$ \\
		RDM & $39.36 \pm 7.67$ & $8.97 \pm 1.78$ & $31.09 \pm 5.37$ \\
ARM-CML & $54.08 \pm 1.03$ & $\mathbf{14.74 \pm 0.58}$ & $49.52 \pm 0.84$ \\
		\midrule
		\method{} (uniform) & $44.42 \pm 1.03$ & $4.53 \pm 0.25$ & $34.13 \pm 0.78$ \\
		\method{} ($K=1$) & $57.96 \pm 2.73$ & $11.62 \pm 0.90$ & $48.41 \pm 0.44$ \\
		\method{} (sample gate) & $58.00 \pm 3.05$ & $14.40 \pm 0.70$ & $48.16 \pm 1.23$ \\
		\method{} (group gate) & $59.22 \pm 3.10$ & $12.00 \pm 0.86$ & $48.29 \pm 1.98$ \\
		\textbf{\method{}} & $\textbf{61.71} \pm \textbf{1.76}$ & $11.61 \pm 0.65$ & $\textbf{50.01} \pm \textbf{0.84}$ \\
		\bottomrule
	\end{tabular}
\end{table}

\begin{figure}[!htb]
	\centering
	\includegraphics[width=0.86\linewidth]{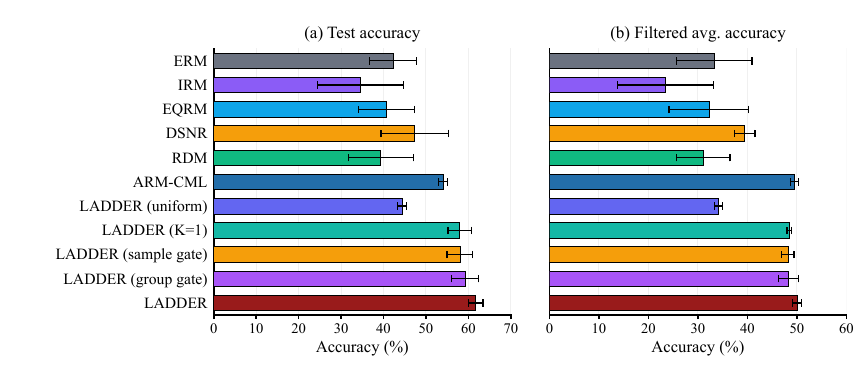}
	\caption{iWildCam-WILDS: (a) test and (b) filtered average-group accuracy (mean $\pm$ SD, five seeds). LADDER uses one target summary per camera location; its group gate ablation shares weights within each location. Whole-split summary comparison: Appendix~\ref{app:iwildcam_results}.}
	\label{fig:iwildcam_main}
\end{figure}
\FloatBarrier

All methods use ImageNet-pretrained ResNet-50
\citep{he2016deep,deng2009imagenet}, $64$ source locations,
class-balanced cross-entropy with parameter $0.999$
\citep{cui2019classbalanced}, and five seeds.
We report test accuracy, macro F1, and filtered average-group accuracy
over the $35$ locations with at least $50$ samples;
raw and filtered worst-location metrics remain nearly degenerate.
ERM, IRM, EQRM, and RDM predict independently per example;
DSNR estimates one domain index from unlabeled covariates for the full evaluation split.
LADDER uses camera-location IDs to partition unlabeled target covariates into domains,
constructing one Sinkhorn style summary per domain, also used by its $K=1$ ablation;
encoders and classifiers remain fixed. ARM-CML uses the same location groups to average its learned contexts.
The \method{} (group gate) ablation averages per-sample source-membership weights within each location;
\method{} (uniform) uses equal weights; \method{} (sample gate) computes weights per example.
Appendix~\ref{app:iwildcam_results} details summary construction and the single whole-split summary comparison.

Table~\ref{tab:iwildcam_results} and Figure~\ref{fig:iwildcam_main}(a,b) show that \method{} obtains the highest test and filtered average-group accuracy, numerically exceeding \method{} (group gate) by about $2.5$ and $1.7$ points, respectively; ARM-CML attains the highest macro F1 ($14.74$), while LADDER has higher test accuracy ($61.71$ versus $54.08$) and slightly higher filtered average-group accuracy ($50.01$ versus $49.52$). Together with the sample gate ($14.40$ macro F1), this indicates a trade-off between overall/location-averaged accuracy and class-balanced performance. The protocol-matched comparison with \method{} (group gate) supports distribution-level routing beyond location sharing. Appendix Figure~\ref{fig:supp_iwildcam_tradeoff_ablation} reports the overall versus filtered location-averaged accuracy trade-off and full routing comparison, including macro F1 and the whole-split summary.

\endgroup
\section{Discussion}\label{sec:discussion}

LADDER uses the target style distribution to combine nearby source prediction rules. Its effectiveness depends on whether style proximity reflects rule similarity and whether the sources adequately cover the target.
Our empirical decomposition uses $\bm Z_c$ for prediction and $\bm Z_s^{(e)}$ for domain-level routing, without requiring statistical independence between representations. Future work could extend LADDER with source weighting based on uncertainty, extrapolation beyond convex source combinations, and representations from foundation models for richer visual and multimodal DG.

\FloatBarrier
\subsection*{AI use statement}
Generative AI tools assisted with language editing, code scaffolding, and the review and revision
of mathematical arguments and manuscript text. The authors reviewed the AI-assisted material and
take responsibility for the final manuscript.

\subsection*{Reproducibility statement}
To support reproducibility, Appendix~\ref{app:experimental_details} documents the datasets, evaluation protocols, baseline configurations, and supplementary experimental results. Appendix~\ref{app:implementation_details} specifies implementation details and hyperparameters, including preprocessing, training schedules, and inference settings. The experimental appendix also reports the simulation setup, repeated-run settings, and computational requirements.

\clearpage
\appendix
Appendix~\ref{app:auxiliary_statements} gives full assumptions and auxiliary results; Appendix~\ref{app:proofs} contains the proofs; Appendix~\ref{app:experimental_details} provides experimental details and additional results.

\section{Theoretical Assumptions and Auxiliary Statements}\label{app:auxiliary_statements}

This section records the formal assumptions and auxiliary statements used in the theoretical analysis. The first subsection expands the compact assumptions from the main paper; the next subsection collects the proposition and lemmas used in the proofs.
\subsection{Formal Assumptions and Conditions}\label{app:full_assumptions}
For the theoretical analysis, conditional on the domain distributions, observations within each domain are sampled independently and identically from that domain, and the samples from distinct source domains and the target summary sample are mutually independent.

The main paper states compact versions of the assumptions to keep the theoretical presentation readable. We restate the full formal assumptions and conditions used by the proofs below. Throughout, $n_{\min}:=\min_{1\le e\le E}n_e$. Positive constants $C_\star,c_\star$ and their indexed versions may vary between bounds and depend on fixed regularity parameters (and on fixed $E$ in the generic representation bound); $C$ always denotes the number of classes.

\begin{assumption}[Oracle latent structure]\label{app:ass:linear_full}
For each domain $e$, the centered observation $\bm X^{(e)}\in\mathbb R^p$ follows
\[
\bm{X}^{(e)}=\bm{A}_c\bm{Z}_c^{(e)}+\bm{A}_s\bm{Z}_s^{(e)}+\bm{\epsilon}^{(e)},
\]
where $\bm{A}=[\bm{A}_c,\bm{A}_s]$ has full column rank. The causal and style components satisfy
\[
\mathrm{Cov}(\bm{Z}_c^{(e)},\bm{Z}_s^{(e)})=\bm{0},\quad
\mathrm{Cov}(\bm{Z}_c^{(e)})=\bm{\Sigma}_c,\quad
\mathrm{Cov}(\bm{Z}_s^{(e)})=\bm{\Lambda}_e.
\]
Here $\bm{\Sigma}_c$ is domain-invariant and $\bm{\Lambda}_e=\mathrm{diag}(\lambda_{e1},\ldots,\lambda_{ed_s})$ is diagonal and domain-dependent. The noise satisfies $\mathbb{E}\bm{\epsilon}^{(e)}=\bm{0}$, is independent of the latents and labels, and has domain-invariant covariance $\mathrm{Cov}(\bm{\epsilon}^{(e)})=\bm{\Sigma}_{\epsilon}$.
Assume $\bm{\Lambda}_1\succ0$. For each style coordinate $m$, define its normalized cross-domain variance profile
\[
\bm{\rho}_m=(\lambda_{2m}/\lambda_{1m},\ldots,\lambda_{Em}/\lambda_{1m}).
\]
The profiles are jointly separating: $\bm{\rho}_m\ne\bm{\rho}_{m'}$ for all $m\ne m'$.
\end{assumption}
\begin{Rem}[Repeated variance profiles]
If several style coordinates share the same normalized variance profile, they are not separately identifiable from second-order domain variation. The corresponding joint eigenspace is identifiable, but coordinates within that eigenspace are recoverable only up to an orthogonal transformation after whitening. Pairwise profile separation is the special case in which every block is one-dimensional.
\end{Rem}

\begin{Con}[Sufficient condition for leakage removal]\label{app:cond:mean_separation}
For leakage removal, the candidate extractor $\bm{W}_s$ additionally satisfies label-invariant conditional means for the extracted style representation, while the true style representation has label-invariant conditional means within each source domain:
$\mathbb{E}[\bm{Z}_s^{(e)}\mid Y^{(e)}=y]=\bm{m}_s^{(e)}$ and
$\mathbb{E}[\widehat{\bm{Z}}_s^{(e)}\mid Y^{(e)}=y]=\widehat{\bm{m}}_s^{(e)}$ for all $y\in\mathrm{supp}(Y^{(e)})$.
Let $\bm{\mu}_{e,y}=\mathbb{E}[\bm{Z}_c^{(e)}\mid Y^{(e)}=y]$ and choose one reference label $y_e^0\in\mathrm{supp}(Y^{(e)})$ per source domain. The causal conditional means span the causal space:
$\mathrm{span}\{\bm{\mu}_{e,y}-\bm{\mu}_{e,y_e^0}:e=1,\ldots,E,\ y\in\mathrm{supp}(Y^{(e)})\}=\mathbb{R}^{d_c}$.
\end{Con}

\begin{assumption}[Smooth domain geometry]\label{app:ass:smooth_geometry_full}
There is a map $\Psi^*$ with $\bm{\beta}_*^{(e)}=\Psi^*(\nu_e)$ and
$\|\Psi^*(\nu)-\Psi^*(\nu')\|_F\le L_\beta\mathcal{W}_1(\nu,\nu')$.
All style distributions are supported on a compact $\mathcal{Z}_s\subset\mathbb{R}^{d_s}$ of diameter at most $D_s$. Moreover, $\Pi_{\mathcal{M}}$ is $d_{\mathcal{M}}$-dimensional Ahlfors-regular on $\mathcal{M}$, with $d_{\mathcal{M}}\in\mathbb{N}$. The source style laws $\nu_1,\ldots,\nu_E$ are independently sampled from $\Pi_{\mathcal{M}}$; $\nu_0$ is either an independent draw or fixed independently of the source sample, with the Ahlfors-regular bounds holding uniformly around $\nu_0$. For the fixed target law $\nu_0$, the distance distribution function
\[
F_{\nu_0}(r)=\Pi_{\mathcal{M}}\bigl(B_{\mathcal{W}_1}(\nu_0,r)\bigr)
\]
is continuous in $r$ on $[0,\operatorname{diam}(\mathcal{M})]$.
\end{assumption}

\begin{assumption}[GLM regularity]\label{app:ass:glm_regularity_full}
Let
\[
\mathcal{S}_\beta=\{\bm{\beta}\in\mathbb{R}^{C\times d_c}:\bm{1}^\top\bm{\beta}=\bm{0}\},
\]
and let $\mathcal{B}\subset\mathcal{S}_\beta$ be a compact convex parameter set with $\|\bm{\beta}\|_F\le R_\beta$. For all samples, $\|\bm{Z}_c\|_2\le R_c$. For $\ell_{\mathrm{CE}}(\bm u,y)=-\log(\mathrm{Softmax}(\bm u)_y)$, define
\[
\mathcal{R}_e(\bm{\beta})=\mathbb{E}[\ell_{\mathrm{CE}}(\bm{\beta}\bm Z_c,Y)\mid e],\qquad
\widehat{\mathcal{R}}_e(\bm{\beta})=\frac{1}{n_e}\sum_{i=1}^{n_e}\ell_{\mathrm{CE}}(\bm{\beta}\bm z_{c,i}^{(e)},y_i^{(e)}).
\]
For every domain considered in the oracle analysis, $\bm{\beta}_*^{(e)}\in\operatorname{relint}(\mathcal{B})$, $\nabla_{\mathcal{S}_\beta}\mathcal{R}_e(\bm{\beta}_*^{(e)})=\bm{0}$, and $\mathcal{R}_e$ is uniformly $\mu$-strongly convex on $\mathcal{B}$ with respect to the Frobenius norm restricted to $\mathcal{S}_\beta$. For every source domain and sample, the single-sample Hessian is uniformly bounded and $L_H$-Lipschitz in $\bm{\beta}$ over $\mathcal{B}$ as an operator on $\mathcal{S}_\beta$. With $\bm{\Pi}_C=\bm{I}_C-C^{-1}\bm{1}\bm{1}^\top$, let
\[
\mathcal{S}_C=\{\bm{u}\in\mathbb{R}^C:\bm{1}^\top\bm{u}=0\},\qquad
\mathcal{U}_R=\{\bm{u}\in\mathcal{S}_C:\|\bm{u}\|_2\le R_{\rm logit}\}.
\]
The oracle and comparison centered logits used in the lower-bound results lie in $\mathcal{U}_R$ almost surely. On $\mathcal{U}_R$, the conditional cross-entropy risk is $\mu_\ell$-strongly convex; on $\mathcal{S}_C$, it is globally $L_\ell$-smooth. Bounded centered logits imply that the Softmax probabilities are uniformly bounded away from zero, yielding a uniform positive-curvature lower bound on $\mathcal{U}_R$ along the centered-logit subspace $\mathcal{S}_C$.
\end{assumption}

\paragraph{Comparison notation.}
We compare \method{}, PooledJoint, and the fixed prediction rule using a scalar coordinate $c$ for latent domain position and define
\[
f^*(\bm Z_c,c)
=(\bm\beta_{\mathrm{base}}+c\bm\beta_{\mathrm{shift}})\bm Z_c.
\]
For any predictor $f$, write
\[
\mathcal R_c(f)
=\mathbb E[\ell_{\mathrm{CE}}(f(\bm Z_c,\bm Z_s),Y_c)],
\qquad
Y_c\sim\operatorname{Softmax}(f^*(\bm Z_c,c)),
\]
and define the bounded additive pooled class
\[
\mathcal H_{\mathrm{pool}}(R_W)
=\{f=\bm W_c\bm Z_c+\bm W_s\bm Z_s:
\|\bm W_c\|_F,\|\bm W_s\|_F\le R_W\}.
\]
For target coordinate $c_0$, define
\[
f_{\mathrm L}=f^*(\cdot,\widehat c_{\mathrm L}),\qquad
\widehat c_{\mathrm L}=\sum_e w_e c_e,\qquad
f_{\mathrm{inv}}=f^*(\cdot,c_{\mathrm{inv}}),
\]
where $c_e$ is source $e$'s coordinate and $c_{\mathrm{inv}}$ does not depend on the target domain.

\begin{assumption}[Comparison shift model]\label{app:ass:comparison_shift_full}
The coordinate satisfies $c\sim\pi_c$, is bounded, and has $\operatorname{Var}_{\pi_c}(c)>0$. Moreover, $\mathbb E\bm Z_c=\bm0$, $\mathbb E\bm Z_c\bm Z_c^\top=\bm\Sigma_c\succ0$, and $\bm Z_c\perp(c,\bm Z_s)$. For the target coordinate, $c_{\mathrm{inv}}\ne c_0$. The radius $R_W$ is large enough to contain one optimizer of the unrestricted additive projection used in Lemma~\ref{lem:pooled_projection}. For the comparison lower bounds, $\bm\Pi_C f^*(\bm Z_c,c)$ and the centered logits of the fixed predictor lie in $\mathcal U_R$ almost surely; moreover, $\bm\Pi_C f(\bm Z_c,\bm Z_s)\in\mathcal U_R$ almost surely for every $f\in\mathcal H_{\mathrm{pool}}(R_W)$.
\end{assumption}
\begin{assumption}[Second-moment representation error]\label{app:ass:generic_representation_error_full}
For $e=0,\ldots,E$ and $i=1,\ldots,n_e$, let
$\bm z_{c,i}^{(e)}\in\mathbb R^{d_c}$ and
$\bm z_{s,i}^{(e)}\in\mathbb R^{d_s}$ be the oracle causal and style
features of observation $i$ in domain $e$.
After a fixed coordinate alignment, the learned features satisfy
\[
\widetilde{\bm z}_{c,i}^{(e)}=\bm z_{c,i}^{(e)}+\bm\xi_{c,i}^{(e)},
\qquad
\widetilde{\bm z}_{s,i}^{(e)}=\bm z_{s,i}^{(e)}+\bm\xi_{s,i}^{(e)}.
\]
The representation errors have a common scale $\epsilon_{\mathrm{rep}}\ge0$:
\[
\mathbb E\|\bm\xi_{c,i}^{(e)}\|_2^2
\vee\mathbb E\|\bm\xi_{s,i}^{(e)}\|_2^2
\le\epsilon_{\mathrm{rep}}^2
\quad\text{for all }e,i.
\]
For some $\mu_{\mathrm{rep}}>0$, each source empirical risk
\[
\widetilde{\mathcal R}^{\mathrm{emp}}_e(\bm\beta)
=\frac1{n_e}\sum_{i=1}^{n_e}
\ell_{\mathrm{CE}}(\bm\beta\widetilde{\bm z}_{c,i}^{(e)},y_i^{(e)}),
\qquad e=1,\ldots,E,
\]
is $\mu_{\mathrm{rep}}$-strongly convex on the same compact convex set $\mathcal B$.
To define target population risk, condition on the fitted encoders, source
classifiers, and the sample used to construct the target summary.
Draw a fresh target observation independently of the fitting and summary samples.
Let $(\bm Z_c^{(0)},\bm Z_s^{(0)},Y^{(0)})$ be its oracle quantities and
$(\bm\xi_c^{(0)},\bm\xi_s^{(0)})$ its representation errors.
Under this conditioning,
\[
\Pr\!\left(Y^{(0)}=y\mid
\bm Z_c^{(0)},\bm Z_s^{(0)},\bm\xi_c^{(0)},\bm\xi_s^{(0)}\right)
=\left[\mathrm{Softmax}(\bm\beta_*^{(0)}\bm Z_c^{(0)})\right]_y.
\]
The fresh-observation errors satisfy
\[
\mathbb E\|\bm\xi_c^{(0)}\|_2^2
\vee\mathbb E\|\bm\xi_s^{(0)}\|_2^2
\le\epsilon_{\mathrm{rep}}^2,
\]
where expectation is over the fitted objects and the fresh observation.
\end{assumption}
\subsection{Auxiliary Propositions and Lemmas}\label{app:auxiliary_lemmas}
The proposition and lemmas below are the technical tools invoked by the proof section.

\begin{Pro}[Unique Structure of the Global Linear Transform]\label{pro:diagonal}
	Let $\bm{W}_s \in \mathbb{R}^{d_s \times p}$ be a candidate population style-extraction matrix, and suppose the induced style transform $\bm{M}=\bm{W}_s\bm{A}_s \in \mathbb{R}^{d_s \times d_s}$ is invertible. Under Assumption~\ref{app:ass:linear_full}, if $\bm{M}\bm{\Lambda}_e\bm{M}^\top$ is diagonal for every source domain $e=1,\ldots,E$, then $\bm{M}=\bm{P}\bm{D}$, where $\bm{P}$ is a permutation matrix and $\bm{D}$ is a non-singular diagonal scaling matrix.
\end{Pro}

\begin{Lem}[Uniform Concentration of Empirical Hessian]\label{lem:hessian_concentration}
	Let $d_\beta=(C-1)d_c$. Under Assumption~\ref{app:ass:glm_regularity_full}, the empirical Hessian converges uniformly to the population Hessian over $\mathcal{B}$ as an operator on $\mathcal{S}_\beta$. In particular, there exists a constant $C_\star>0$, depending only on the uniform single-sample Hessian bound and hence on $R_c$, such that for any $\delta>0$, if
	\[
	n_e \ge \frac{C_\star}{\mu^2}\left[d_\beta\log\left(1+\frac{R_\beta L_H}{\mu}\right)+\log\frac{d_\beta}{\delta}\right],
	\]
	then
	\[
	\sup_{\bm{\beta}\in\mathcal{B}}
	\|\nabla^2\widehat{\mathcal{R}}_e(\bm{\beta})-\nabla^2\mathcal{R}_e(\bm{\beta})\|_2
	\le \mu/2
	\]
	with probability at least $1-\delta$.
\end{Lem}

\begin{Lem}[Finite-Sample Source Estimation]\label{lem:source_estimation}
	Under Assumption~\ref{app:ass:glm_regularity_full}, for any source domain $e$ and sufficiently large $n_e$,
	\[
	\mathbb{E}\|\widehat{\bm{\beta}}^{(e)}-\bm{\beta}_*^{(e)}\|_F
	\le
	\mathcal{O}\!\left(\frac{R_c}{\mu\sqrt{n_e}}\right).
	\]
\end{Lem}

\begin{Lem}[Uniform Source-Parameter Bound]\label{lem:estimator_tail}
	Let $d_\beta=(C-1)d_c$. Under Assumption~\ref{app:ass:glm_regularity_full}, for every $\delta\in(0,1)$, if
	\[
	n_{\min}\ge \frac{C_\star}{\mu^2}\left[d_\beta\log\left(1+\frac{R_\beta L_H}{\mu}\right)+\log\frac{d_\beta E}{\delta}\right],
	\]
	then, with probability at least $1-\delta$,
	\[
	\max_{1\le e\le E}\|\widehat{\bm{\beta}}^{(e)}-\bm{\beta}_*^{(e)}\|_F
	\le
	C_\star\frac{R_c}{\mu}\sqrt{\frac{d_\beta+\log(E/\delta)}{n_{\min}}}.
	\]
Consequently, under the sample-size requirement of Theorem~\ref{thm:param_bound} with a sufficiently large constant,
	\[
	\mathbb{E}\max_{1\le e\le E}\|\widehat{\bm{\beta}}^{(e)}-\bm{\beta}_*^{(e)}\|_F
	\le
	\mathcal{O}\!\left(\frac{R_c}{\mu}\sqrt{\frac{d_\beta+\log E}{n_{\min}}}\right),
	\]
	and the corresponding squared maximum is bounded by
	$\mathcal{O}(R_c^2(d_\beta+\log E)/(\mu^2 n_{\min}))$.
\end{Lem}

For Wasserstein rates, write
\[
\gamma(n,d)=n^{-1/(d\vee2)}\ell_d(n),\qquad
\Gamma(n,E,d)=((1+\log E)/n)^{1/(d\vee2)}\widetilde{\ell}_d(n),
\]
where $\ell_2(n)=\widetilde{\ell}_2(n)=\log(1+n)$ and $\ell_d(n)=\widetilde{\ell}_d(n)=1$ for $d\ne2$.

\begin{Lem}[Empirical Wasserstein Concentration]\label{lem:wasserstein_concentration}
	Under Assumption~\ref{app:ass:smooth_geometry_full}, for any domain style distribution $\nu$ supported on $\mathcal{Z}_s$ and empirical measure $\widehat{\nu}$ from $n$ samples,
	\[
	\mathbb{E}\mathcal{W}_1(\nu,\widehat{\nu})
	\le C_\star\gamma(n,d_s).
	\]
	Moreover, with $\Delta_W=\gamma(n_{\min},d_s)+\Gamma(n_{\min},E,d_s)$,
	\[
	\mathbb{E}\max_{1\le e\le E}\mathcal{W}_1(\nu_e,\widehat{\nu}_e)
	\le \mathcal{O}(\Delta_W),\qquad
	\mathbb{E}\max_{1\le e\le E}\mathcal{W}_1^2(\nu_e,\widehat{\nu}_e)
	\le \mathcal{O}(\Delta_W^2).
	\]
\end{Lem}

\begin{Lem}[Weighted Wasserstein Approximation]\label{lem:wasserstein_approx}
	With these empirical Wasserstein rates, under Assumption~\ref{app:ass:smooth_geometry_full} and Lemma~\ref{lem:wasserstein_concentration},
	\[
	\mathbb{E}\!\left[
	L_\beta\sum_{e\in\widehat{\mathcal{N}}_K(0)}
	w_e\{\mathcal{W}_1(\nu_e,\widehat{\nu}_e)+\mathcal{W}_1(\widehat{\nu}_0,\nu_0)\}
	\right]
	\le
	\mathcal{O}\{\gamma(n_0,d_s)+\Gamma(n_{\min},E,d_s)\}.
	\]
\end{Lem}

\begin{Lem}[Manifold Interpolation via Empirical KNN]\label{lem:knn_interpolation}
	Under Assumption~\ref{app:ass:smooth_geometry_full},
	\[
	\mathbb{E}\!\left[
	L_\beta\sum_{e\in\widehat{\mathcal{N}}_K(0)}
	w_e\mathcal{W}_1(\widehat{\nu}_e,\widehat{\nu}_0)
	\right]
	\le
	\mathcal{O}\!\left((K/E)^{1/d_{\mathcal{M}}}\right)
	+
	\mathcal{O}\{\gamma(n_0,d_s)+\Gamma(n_{\min},E,d_s)\}.
	\]
\end{Lem}

\begin{Lem}[Centered Logit Risk Bounds]\label{lem:centered_logit_excess}
	Under Assumption~\ref{app:ass:glm_regularity_full}, write $\bm Z=(\bm Z_c,\bm Z_s)$ and define
	\[
	\mathcal{D}_{\ell}(f,f^*)
	=
	\mathbb{E}\left\|
	\bm{\Pi}_C\{f(\bm{Z})-f^*(\bm{Z}_c,c)\}
	\right\|_2^2 .
	\]
	Let $\Delta_\ell(f,f^*)$ denote the corresponding expected cross-entropy excess risk. If the centered logits of $f$ and $f^*$ lie in $\mathcal{U}_R$ almost surely, then
	\[
	\Delta_{\ell}(f,f^*)
	\ge
	\frac{\mu_{\ell}}{2}\mathcal{D}_{\ell}(f,f^*).
	\]
	For arbitrary finite centered logits,
	\[
	\Delta_{\ell}(f,f^*)
	\le
	\frac{L_{\ell}}{2}\mathcal{D}_{\ell}(f,f^*).
	\]
\end{Lem}

\begin{Lem}[Best Additive Approximation to a Multiplicative Shift]\label{lem:pooled_projection}
	Under Assumptions~\ref{app:ass:comparison_shift_full} and \ref{app:ass:glm_regularity_full}, the optimal expected centered-logit squared error of the additive pooled class relative to $f^*(\bm{Z}_c,c)$ satisfies
	\[
	\inf_{f\in\mathcal H_{\mathrm{pool}}(R_W)}\mathbb E\|\bm\Pi_C[f(\bm Z_c,\bm Z_s)-f^*(\bm Z_c,c)]\|_2^2
	=
	\mathrm{Var}(c)\,
	\|\bm{\Pi}_C\bm{\beta}_{\mathrm{shift}}\bm{\Sigma}_c^{1/2}\|_F^2 .
	\]
\end{Lem}

\section{Proofs of Theoretical Statements}\label{app:proofs}
	
	This section proves the statements from the main paper and the auxiliary statements stated above. We first establish representation identifiability, then source-parameter and Wasserstein navigation bounds, target parameter and risk bounds, baseline comparisons, and finally the generic-representation perturbation bound. The arguments use standard tools from GLMs, M-estimation, concentration, optimal transport, and nearest-neighbor theory \citep{mccullagh1989generalized,vanderVaart1998asymptotic,wainwright2019high,tropp2012userfriendly,boucheron2013concentration,vershynin2018highdim,fournier2015rate,devroye1996probabilistic,biau2015nearest}.
	
	\subsection{Proof of Proposition~\ref{pro:diagonal}}
	
	\begin{proof}
Let $\bm{D}_e=\bm{M}\bm{\Lambda}_e\bm{M}^\top$. By assumption, each $\bm{D}_e$ is diagonal. Since $\bm{\Lambda}_1\succ0$ and $\bm{M}$ is invertible, $\bm{D}_1\succ0$. Define
\begin{equation*}
\bm{Q}=\bm{D}_1^{-1/2}\bm{M}\bm{\Lambda}_1^{1/2}.
\end{equation*}
The whitening identity gives
\[
\begin{aligned}
\bm Q\bm Q^\top
&=\bm D_1^{-1/2}\bm M\bm\Lambda_1^{1/2}
\bm\Lambda_1^{1/2}\bm M^\top\bm D_1^{-1/2}\\
&=\bm D_1^{-1/2}(\bm M\bm\Lambda_1\bm M^\top)\bm D_1^{-1/2}\\
&=\bm D_1^{-1/2}\bm D_1\bm D_1^{-1/2}=\bm I.
\end{aligned}
\]
Since $\bm Q$ is square, $\bm Q^{-1}=\bm Q^\top$ and $\bm Q^\top\bm Q=\bm I$. For $e=2,\ldots,E$, let
\begin{equation*}
\bm{R}_e=\bm{\Lambda}_1^{-1/2}\bm{\Lambda}_e\bm{\Lambda}_1^{-1/2}.
\end{equation*}
Each $\bm{R}_e$ is diagonal, and
\begin{equation*}
\bm{Q}\bm{R}_e\bm{Q}^\top=\bm{D}_1^{-1/2}\bm{D}_e\bm{D}_1^{-1/2}
\end{equation*}
is diagonal. Equivalently, the rows of $\bm{Q}$ form a common orthonormal eigenbasis of the family $\{\bm{R}_e\}_{e=2}^E$.

Write $\bm D_e=\operatorname{diag}(d_{e1},\ldots,d_{ed_s})$.
Right-multiplication of the preceding diagonalization identity by $\bm Q$ gives
\[
\bm Q\bm R_e
=\bm D_1^{-1/2}\bm D_e\bm D_1^{-1/2}\bm Q,
\qquad e=2,\ldots,E.
\]
Its $(j,m)$ entry is
\[
Q_{jm}\frac{\lambda_{em}}{\lambda_{1m}}
=\frac{d_{ej}}{d_{1j}}Q_{jm},\qquad
\left(\frac{\lambda_{em}}{\lambda_{1m}}-\frac{d_{ej}}{d_{1j}}\right)Q_{jm}=0.
\]
The joint profile is $\bm\rho_m=(\lambda_{2m}/\lambda_{1m},\ldots,\lambda_{Em}/\lambda_{1m})$.
If $Q_{jm}\ne0$ and $Q_{jn}\ne0$, then
\[
\frac{\lambda_{em}}{\lambda_{1m}}
=\frac{d_{ej}}{d_{1j}}
=\frac{\lambda_{en}}{\lambda_{1n}}
\quad(e=2,\ldots,E),
\qquad \bm\rho_m=\bm\rho_n.
\]
The joint profiles are pairwise distinct by Assumption~\ref{app:ass:linear_full},
so each row has at most one nonzero entry. Orthogonality gives
\[
\sum_m Q_{jm}^2=1,\qquad
\sum_m Q_{jm}Q_{km}=0\quad(j\ne k).
\]
Each row therefore has one entry equal to $+1$ or $-1$, and different rows
use different columns. Thus $\bm Q=\bm P\bm S$ for a permutation matrix
$\bm P$ and a diagonal sign matrix $\bm S$. Substitution gives
\begin{equation*}
\bm{M}=\bm{D}_1^{1/2}\bm{P}\bm{S}\bm{\Lambda}_1^{-1/2}
=\bm{P}\widetilde{\bm{D}},
\end{equation*}
where $\widetilde{\bm{D}}=(\bm{P}^\top\bm{D}_1^{1/2}\bm{P})\bm{S}\bm{\Lambda}_1^{-1/2}$ is nonsingular and diagonal. Renaming $\widetilde{\bm{D}}$ as $\bm{D}$ gives the result.
\end{proof}
	
	\subsection{Proof of Theorem~\ref{thm:identifiability}}
\label{app:proof:identifiability}
\begin{proof}
From Assumption~\ref{app:ass:linear_full}, the covariance of $\bm{X}^{(e)}$ decomposes as
\begin{equation*}
\bm{\Sigma}_{\bm{X}}^{(e)}=\bm{A}_c\bm{\Sigma}_c\bm{A}_c^\top+\bm{A}_s\bm{\Lambda}_e\bm{A}_s^\top+\bm{\Sigma}_{\epsilon}.
\end{equation*}
The causal covariance and noise covariance are domain-invariant, while the domain-varying style covariance is $\bm{A}_s\bm{\Lambda}_e\bm{A}_s^\top$. By the oracle diagonalization condition in Theorem~\ref{thm:identifiability}, $\bm{M}\bm{\Lambda}_e\bm{M}^\top$ is diagonal for every source domain, where $\bm{M}=\bm{W}_s\bm{A}_s$. Proposition~\ref{pro:diagonal} gives $\bm{M}=\bm{P}\bm{D}$ for a permutation $\bm{P}$ and nonsingular diagonal $\bm{D}$.

Expanding the extracted representation gives
\begin{equation*}
\widehat{\bm{Z}}_s^{(e)}=\bm{W}_s\bm{X}^{(e)}
=\bm{W}_s\bm{A}_c\bm{Z}_c^{(e)}+\bm{W}_s\bm{A}_s\bm{Z}_s^{(e)}+\bm{W}_s\bm{\epsilon}^{(e)}
=\bm{B}\bm{Z}_c^{(e)}+\bm{P}\bm{D}\bm{Z}_s^{(e)}+\widetilde{\bm{\epsilon}}^{(e)},
\end{equation*}
where $\bm{B}=\bm{W}_s\bm{A}_c$ and $\widetilde{\bm{\epsilon}}^{(e)}=\bm{W}_s\bm{\epsilon}^{(e)}$. This identifies the style mixing component up to permutation and scaling while allowing the extracted representation to contain a causal leakage term $\bm{B}\bm{Z}_c^{(e)}$.

For the leakage-removal statement, assume Condition~\ref{app:cond:mean_separation}. Taking conditional expectations and using the zero-mean independent noise gives, for any $y\in\mathrm{supp}(Y^{(e)})$,
\begin{equation*}
\widehat{\bm{m}}_s^{(e)}=\bm{B}\bm{\mu}_{e,y}+\bm{P}\bm{D}\bm{m}_s^{(e)}.
\end{equation*}
Subtracting the conditional-mean identity at the reference label $y_e^0$ gives
\[
\begin{aligned}
\bm0
&=\widehat{\bm m}_s^{(e)}-\widehat{\bm m}_s^{(e)}\\
&=(\bm B\bm\mu_{e,y}+\bm P\bm D\bm m_s^{(e)})
 -(\bm B\bm\mu_{e,y_e^0}+\bm P\bm D\bm m_s^{(e)})\\
&=\bm B(\bm\mu_{e,y}-\bm\mu_{e,y_e^0}).
\end{aligned}
\]
By the spanning condition, every $\bm v\in\mathbb R^{d_c}$ has coefficients
$a_{e,y}$ such that
\[
\bm v=\sum_{e=1}^E\sum_{y\in\mathrm{supp}(Y^{(e)})}
a_{e,y}(\bm\mu_{e,y}-\bm\mu_{e,y_e^0}).
\]
Applying $\bm B$ gives
\[
\bm B\bm v
=\sum_{e=1}^E\sum_{y\in\mathrm{supp}(Y^{(e)})}
a_{e,y}\bm B(\bm\mu_{e,y}-\bm\mu_{e,y_e^0})
=\bm0\quad\text{for every }\bm v\in\mathbb R^{d_c}.
\]
Hence $\bm B=\bm0$ and
\[
\widehat{\bm Z}_s^{(e)}=\bm P\bm D\bm Z_s^{(e)}
+\widetilde{\bm\epsilon}^{(e)},
\]
which is recovery up to permutation, coordinate-wise scaling, and projected noise.
\end{proof}
\begin{Rem}[Effect of permutation and scaling on style geometry]
Let $T=\bm{P}\bm{D}$. In the leakage-free and noise-free case, or when considering only the deterministic coordinate transform of the oracle style laws, for any two style distributions $\nu$ and $\nu'$,
\begin{equation*}
\sigma_{\min}(\bm{D})\mathcal{W}_1(\nu,\nu')
\le
\mathcal{W}_1(T_{\#}\nu,T_{\#}\nu')
\le
\|\bm{D}\|_{\mathrm{op}}\mathcal{W}_1(\nu,\nu').
\end{equation*}
Hence the oracle and transformed-oracle style geometries are bi-Lipschitz equivalent. For transformed laws $\bar\nu=T_{\#}\nu$, the map $\bar\Psi^*(\bar\nu)=\Psi^*(T_{\#}^{-1}\bar\nu)$ remains Lipschitz with constant $L_\beta/\sigma_{\min}(\bm{D})$. This statement concerns the deterministic transform $T$ alone. In the presence of projected representation noise, exact bi-Lipschitz equivalence need not hold; such residual error is instead handled by the perturbation interface in Assumption~\ref{app:ass:generic_representation_error_full}.
\end{Rem}
	
	\subsection{Proof of Lemma~\ref{lem:hessian_concentration}}
	
	\begin{proof}
		We aim to bound the maximum spectral norm of the empirical Hessian $\nabla^2 \widehat{\mathcal{R}}_e(\bm{\beta})$ deviating from the population Hessian $\nabla^2 \mathcal{R}_e(\bm{\beta})$ over the entire compact parameter space $\mathcal{B}$. The argument combines a finite covering of $\mathcal{B}$, pointwise matrix concentration on the cover, and Lipschitz extension to the full parameter space \citep{tropp2012userfriendly,vershynin2018highdim}.
		
		\textbf{Covering argument.}
		The parameter space $\mathcal{B} \subset \mathbb{R}^{C \times d_c}$ is a bounded compact convex set, and for any $\bm{\beta} \in \mathcal{B}$, $\|\bm{\beta}\|_F \le R_\beta$. The effective dimension of this centered-gauge space is $d_\beta=(C-1)d_c$. Standard covering number bounds in high-dimensional geometry \citep{vershynin2018highdim} imply that, for any precision $\epsilon > 0$, there exists a finite subset $\mathcal{B}_\epsilon \subset \mathcal{B}$ (an $\epsilon$-net) such that for any $\bm{\beta} \in \mathcal{B}$, there exists at least one $\bm{\beta}' \in \mathcal{B}_\epsilon$ satisfying $\|\bm{\beta} - \bm{\beta}'\|_F \le \epsilon$. The cardinality of this $\epsilon$-net satisfies:
		\begin{equation*}
			|\mathcal{B}_\epsilon| \le \left( 1 + \frac{2 R_\beta}{\epsilon} \right)^{d_\beta}
		\end{equation*}
		
		\textbf{Pointwise concentration and union bound.}
		Fix an arbitrary point $\bm{\beta}' \in \mathcal{B}_\epsilon$. The empirical Hessian is the average of $n_e$ independent random matrices. For the multi-class cross-entropy loss, its single-sample Hessian has the Kronecker-product form: $\nabla^2 \ell_{\mathrm{CE}} = (\bm{Z}_c \bm{Z}_c^\top) \otimes (\mathrm{diag}(\bm{p}) - \bm{p}\bm{p}^\top)$, where $\bm p=\mathrm{Softmax}(\bm\beta\bm Z_c)$ is the predicted probability vector.
Here $\operatorname{vec}$ stacks the columns of a matrix.
		
		By the spectral norm property of the Kronecker product, $\|\nabla^2 \ell_{\mathrm{CE}}\|_2 = \|\mathrm{diag}(\bm{p}) - \bm{p}\bm{p}^\top\|_2 \cdot \|\bm{Z}_c \bm{Z}_c^\top\|_2$. The simplex covariance has maximum eigenvalue at most $\frac{1}{2}$ (achieved in binary classification with uniform probabilities), and $\|\bm{Z}_c\|_2 \le R_c$; hence the single-sample Hessian has spectral norm at most $M = \frac{1}{2}R_c^2$, independent of the number of classes $C$. 
		
        Work in an orthonormal coordinate system of $\mathcal S_\beta$, of
dimension $d_\beta$. Let $\bm H_i$ be the single-sample Hessian restricted
to this space at $\bm\beta'$, and set $\bm X_i=\bm H_i-\mathbb E\bm H_i$.
The positive-semidefinite Hessian bound implies
\[
\begin{aligned}
\bm0\preceq\bm H_i\preceq M\bm I_{d_\beta},\quad
\bm0\preceq\mathbb E\bm H_i\preceq M\bm I_{d_\beta}
&\ \Longrightarrow\ -M\bm I_{d_\beta}\preceq\bm X_i\preceq M\bm I_{d_\beta},\\
\mathbb E\bm X_i&=\bm0,\qquad
\bm X_i^2\preceq M^2\bm I_{d_\beta},\\
\left\|\sum_{i=1}^{n_e}M^2\bm I_{d_\beta}\right\|_2&=n_eM^2.
\end{aligned}
\]
The deterministic bound $M\bm I_{d_\beta}$ commutes with $\bm X_i$.
The corresponding matrix Hoeffding bound
\citep[Remark~7.4]{tropp2012userfriendly}, applied to both $\bm X_i$ and
$-\bm X_i$, gives
\[
\begin{aligned}
\Pr\!\left(\left\|\frac1{n_e}\sum_i\bm X_i\right\|_2\ge t_H\right)
&\le\Pr\!\left(\lambda_{\max}\!\left(\sum_i\bm X_i\right)\ge n_et_H\right)\\
&\quad+\Pr\!\left(\lambda_{\max}\!\left(-\sum_i\bm X_i\right)\ge n_et_H\right)\\
&\le2d_\beta\exp\!\left(-\frac{(n_et_H)^2}{2n_eM^2}\right).
\end{aligned}
\]
Identifying the centered sum with the empirical Hessian deviation yields
		\begin{equation*}
			\mathbb{P} \left( \left\| \nabla^2 \widehat{\mathcal{R}}_e(\bm{\beta}') - \nabla^2 \mathcal{R}_e(\bm{\beta}') \right\|_2 \ge t_H \right) \le 2d_\beta \exp \left( - \frac{n_e t_H^2}{2 M^2} \right)
		\end{equation*}
		Applying the union bound scales this probability across the entire $\epsilon$-net $\mathcal{B}_\epsilon$:
		\begin{equation*}
			\mathbb{P} \left( \sup_{\bm{\beta}' \in \mathcal{B}_\epsilon} \left\| \nabla^2 \widehat{\mathcal{R}}_e(\bm{\beta}') - \nabla^2 \mathcal{R}_e(\bm{\beta}') \right\|_2 \ge t_H \right) \le 2d_\beta \left( 1 + \frac{2 R_\beta}{\epsilon} \right)^{d_\beta} \exp \left( - \frac{n_e t_H^2}{2 M^2} \right)
		\end{equation*}
		
		\textbf{Extension to the full parameter space.}
		For any point $\bm{\beta} \in \mathcal{B}$ in the full space, since $\mathcal{B}_\epsilon$ is an $\epsilon$-net, there must exist a $\bm{\beta}' \in \mathcal{B}_\epsilon$ such that $\|\bm{\beta} - \bm{\beta}'\|_F \le \epsilon$. Using the triangle inequality to decompose the deviation:
		\begin{equation*}
			\begin{aligned}
				\left\| \nabla^2 \widehat{\mathcal{R}}_e(\bm{\beta}) - \nabla^2 \mathcal{R}_e(\bm{\beta}) \right\|_2 \le & \left\| \nabla^2 \widehat{\mathcal{R}}_e(\bm{\beta}) - \nabla^2 \widehat{\mathcal{R}}_e(\bm{\beta}') \right\|_2 \\
				& + \left\| \nabla^2 \widehat{\mathcal{R}}_e(\bm{\beta}') - \nabla^2 \mathcal{R}_e(\bm{\beta}') \right\|_2 + \left\| \nabla^2 \mathcal{R}_e(\bm{\beta}') - \nabla^2 \mathcal{R}_e(\bm{\beta}) \right\|_2 .
			\end{aligned}
		\end{equation*}
		By Assumption~\ref{app:ass:glm_regularity_full}, the Hessian matrix of the multi-class cross-entropy is $L_H$-Lipschitz continuous. Thus, both the empirical and population deviations are controlled:
		\begin{equation*}
			\left\| \nabla^2 \widehat{\mathcal{R}}_e(\bm{\beta}) - \nabla^2 \widehat{\mathcal{R}}_e(\bm{\beta}') \right\|_2 \le L_H \|\bm{\beta} - \bm{\beta}'\|_F \le L_H \epsilon
		\end{equation*}
		\begin{equation*}
			\left\| \nabla^2 \mathcal{R}_e(\bm{\beta}) - \nabla^2 \mathcal{R}_e(\bm{\beta}') \right\|_2 \le L_H \|\bm{\beta} - \bm{\beta}'\|_F \le L_H \epsilon
		\end{equation*}
		The triangle inequality and the preceding bounds give:
		\begin{equation*}
			\sup_{\bm{\beta} \in \mathcal{B}} \left\| \nabla^2 \widehat{\mathcal{R}}_e(\bm{\beta}) - \nabla^2 \mathcal{R}_e(\bm{\beta}) \right\|_2 \le \sup_{\bm{\beta}' \in \mathcal{B}_\epsilon} \left\| \nabla^2 \widehat{\mathcal{R}}_e(\bm{\beta}') - \nabla^2 \mathcal{R}_e(\bm{\beta}') \right\|_2 + 2 L_H \epsilon
		\end{equation*}
		
		\textbf{Choice of parameters.}
		To bound the error over the entire space by $\frac{\mu}{2}$, we set $2 L_H \epsilon = \frac{\mu}{4}$, yielding the grid precision $\epsilon = \frac{\mu}{8 L_H}$. Accordingly, the deviation threshold at the grid points is required to be $t_H = \frac{\mu}{4}$. Substituting this into the pointwise concentration bound and requiring the failure probability to be at most $\delta$ gives:
		\begin{equation*}
			2d_\beta \left( 1 + \frac{16 R_\beta L_H}{\mu} \right)^{d_\beta} \exp \left( - \frac{n_e (\mu/4)^2}{2 M^2} \right) \le \delta
		\end{equation*}
		Taking the natural logarithm on both sides and rearranging to solve for the lower bound of the sample size $n_e$:
		\begin{equation*}
			n_e \ge \frac{32 M^2}{\mu^2} \left( d_\beta \log\left( 1 + \frac{16 R_\beta L_H}{\mu} \right) + \log(2d_\beta) + \log\left(\frac{1}{\delta}\right) \right)
		\end{equation*}
		Equivalently, up to a universal constant depending only on the uniform Hessian bound,
		\begin{equation*}
			n_e \ge \frac{C_\star}{\mu^2}\left[d_\beta\log\left(1+\frac{R_\beta L_H}{\mu}\right)+\log\frac{d_\beta}{\delta}\right].
		\end{equation*}
		Under this sample size condition, the displayed uniform Hessian bound holds with probability at least $1-\delta$.
	\end{proof}
	
	\subsection{Proof of Lemma~\ref{lem:source_estimation}}
	
	\begin{proof}
		\textbf{Constrained quadratic control.}
		Set
		\[
		\Delta_e=\widehat{\bm{\beta}}^{(e)}-\bm{\beta}_*^{(e)}.
		\]
		We set the probability decay level to $\delta=n_e^{-k}$ for some $k\ge2$. Under the stated coverage-complexity condition, Lemma~\ref{lem:hessian_concentration} holds for sufficiently large $n_e$. Let $\mathcal E$ denote the uniform Hessian event in Lemma~\ref{lem:hessian_concentration}. On $\mathcal E$, the empirical risk is $\mu/2$-strongly convex on $\mathcal{B}$. Therefore,
		\[
		\widehat{\mathcal{R}}_e(\widehat{\bm{\beta}}^{(e)})
		\ge
		\widehat{\mathcal{R}}_e(\bm{\beta}_*^{(e)})
		+
		\left\langle
		\nabla\widehat{\mathcal{R}}_e(\bm{\beta}_*^{(e)}),
		\Delta_e
		\right\rangle
		+
		\frac{\mu}{4}\|\Delta_e\|_F^2 .
		\]
		Because $\widehat{\bm{\beta}}^{(e)}$ minimizes $\widehat{\mathcal{R}}_e$ over the convex set $\mathcal{B}$ and $\bm{\beta}_*^{(e)}\in\mathcal{B}$,
		\[
		\widehat{\mathcal{R}}_e(\widehat{\bm{\beta}}^{(e)})
		\le
		\widehat{\mathcal{R}}_e(\bm{\beta}_*^{(e)}).
		\]
		On $\mathcal{E}$,
		\[
		\frac{\mu}{4}\|\Delta_e\|_F^2
		\le
		-\left\langle
		\nabla\widehat{\mathcal{R}}_e(\bm{\beta}_*^{(e)}),
		\Delta_e
		\right\rangle
		\le
		\|\nabla\widehat{\mathcal{R}}_e(\bm{\beta}_*^{(e)})\|_F\|\Delta_e\|_F,
		\]
		and hence
		\[
		\|\widehat{\bm{\beta}}^{(e)}-\bm{\beta}_*^{(e)}\|_F
		\le
		\frac{4}{\mu}\|\nabla\widehat{\mathcal{R}}_e(\bm{\beta}_*^{(e)})\|_F .
		\]
		This argument uses only constrained optimality and does not require the empirical minimizer to be an interior point.
		
		\textbf{Expected empirical gradient size.}
		Let
		\[
			\bm{G}_i
			=
			\nabla_{\bm{\beta}}
			\ell_{\mathrm{CE}}(\bm{\beta}_*^{(e)}\bm{Z}_{c,i}^{(e)},Y_i^{(e)}),
			\qquad
			\bar{\bm{G}}_e
			=
			\frac{1}{n_e}\sum_{i=1}^{n_e}\bm{G}_i
			=
			\nabla \widehat{\mathcal{R}}_e(\bm{\beta}_*^{(e)}).
		\]
		Assumption~\ref{app:ass:glm_regularity_full} gives $\mathbb{E}\bm{G}_i=\nabla\mathcal{R}_e(\bm{\beta}_*^{(e)})=\bm{0}$.
        Independence and the zero-mean gradients give, for $i\ne j$,
\[
\mathbb E\langle\bm G_i,\bm G_j\rangle_F
=\left\langle\mathbb E\bm G_i,\mathbb E\bm G_j\right\rangle_F=0.
\]
Expanding the squared empirical gradient yields
\[
\begin{aligned}
\mathbb E\|\bar{\bm G}_e\|_F^2
&=\frac1{n_e^2}\mathbb E\left\langle\sum_i\bm G_i,\sum_j\bm G_j\right\rangle_F\\
&=\frac1{n_e^2}\sum_{i=1}^{n_e}\mathbb E\|\bm G_i\|_F^2
+\frac1{n_e^2}\sum_{i\ne j}\mathbb E\langle\bm G_i,\bm G_j\rangle_F\\
&=\frac1{n_e^2}\sum_{i=1}^{n_e}\mathbb E\|\bm G_i\|_F^2\\
&=\frac1{n_e}\mathbb E\left\|
\nabla_{\bm\beta}\ell_{\mathrm{CE}}(\bm\beta_*^{(e)}\bm Z_c,Y)
\right\|_F^2.
\end{aligned}
\]
		For multi-class cross-entropy, with $\bm p_{\bm\beta}(\bm Z_c)=\mathrm{Softmax}(\bm\beta\bm Z_c)$ and one-hot label vector $\bm e_Y$,
		\[
			\nabla_{\bm{\beta}}\ell_{\mathrm{CE}}(\bm{\beta}\bm{Z}_c,Y)
			=
			(\bm{p}_{\bm{\beta}}(\bm{Z}_c)-\bm{e}_Y)\bm{Z}_c^\top,
		\]
        The probability-vector identity and the outer-product norm give
\[
\begin{aligned}
\|\bm p_{\bm\beta}(\bm Z_c)-\bm e_Y\|_2^2
&=\|\bm p_{\bm\beta}(\bm Z_c)\|_2^2+1
-2[\bm p_{\bm\beta}(\bm Z_c)]_Y\le2,\\
\|\nabla_{\bm\beta}\ell_{\mathrm{CE}}(\bm\beta\bm Z_c,Y)\|_F
&=\|\bm p_{\bm\beta}(\bm Z_c)-\bm e_Y\|_2\|\bm Z_c\|_2
\le\sqrt2R_c.
\end{aligned}
\]
The preceding second-moment identity and Cauchy--Schwarz yield
\[
\begin{aligned}
\mathbb E\|\nabla\widehat{\mathcal R}_e(\bm\beta_*^{(e)})\|_F^2
&\le\frac{2R_c^2}{n_e},\\
\mathbb E\|\nabla\widehat{\mathcal R}_e(\bm\beta_*^{(e)})\|_F
&\le\left(\mathbb E\|\nabla\widehat{\mathcal R}_e(\bm\beta_*^{(e)})\|_F^2\right)^{1/2}
\le\frac{\sqrt2R_c}{\sqrt{n_e}}.
\end{aligned}
\]

		\textbf{Expectation over the high-probability event.}
		Let $X_e=\|\widehat{\bm{\beta}}^{(e)}-\bm{\beta}_*^{(e)}\|_F$. On $\mathcal{E}$,
		\[
		\mathbb{E}[X_e\mathbf{1}_{\mathcal{E}}]
		\le
		\frac{4}{\mu}\mathbb{E}\|\nabla\widehat{\mathcal{R}}_e(\bm{\beta}_*^{(e)})\|_F
		\le
		\frac{4\sqrt{2}R_c}{\mu\sqrt{n_e}}.
		\]
		On $\mathcal{E}^c$, compactness gives $X_e\le2R_\beta$ and $\mathbb{P}(\mathcal{E}^c)\le n_e^{-k}$, so
		\[
		\mathbb{E}[X_e\mathbf{1}_{\mathcal{E}^c}]
		\le
		2R_\beta n_e^{-k}.
		\]
		The polynomial tail is lower order for $k\ge2$, which gives
		\[
		\mathbb{E}\|\widehat{\bm{\beta}}^{(e)}-\bm{\beta}_*^{(e)}\|_F
		\le
		\mathcal{O}\!\left(\frac{R_c}{\mu\sqrt{n_e}}\right).
		\]
	\end{proof}
	
	\subsection{Proof of Lemma~\ref{lem:estimator_tail}}
	\begin{proof}
		Let $\mathcal E_e$ be the uniform Hessian event of Lemma~\ref{lem:hessian_concentration} for source $e$. On this event, the empirical risk in source domain $e$ is $\mu/2$-strongly convex on $\mathcal{B}$. The constrained-optimality argument in Lemma~\ref{lem:source_estimation} gives
		\begin{equation*}
		\|\widehat{\bm{\beta}}^{(e)}-\bm{\beta}_*^{(e)}\|_F
		\le
		\frac{4}{\mu}\|\nabla\widehat{\mathcal{R}}_e(\bm{\beta}_*^{(e)})\|_F .
		\end{equation*}
		Choose the Hessian failure level $\delta/(2E)$ in Lemma~\ref{lem:hessian_concentration} and union bound over $e=1,\ldots,E$; the stated lower bound on $n_{\min}$ ensures $\mathbb{P}(\cap_e\mathcal{E}_e)\ge1-\delta/2$.
		
		It remains to control the empirical gradients. Let $\bm g_i^{(e)}\in\mathbb R^{d_\beta}$ be the coordinates of $\nabla_{\bm\beta}\ell_{\mathrm{CE}}(\bm\beta_*^{(e)}\bm Z_{c,i}^{(e)},Y_i^{(e)})$ in an orthonormal basis of $\mathcal S_\beta$.
		Assumption~\ref{app:ass:glm_regularity_full} gives $\mathbb{E}\bm{g}_i^{(e)}=\bm{0}$, and the bounded-feature assumption gives $\|\bm{g}_i^{(e)}\|_2\le \sqrt{2}R_c$ for multi-class cross-entropy. Write $\bar{\bm g}^{(e)}=n_e^{-1}\sum_i\bm g_i^{(e)}$.
Let $\mathcal V$ be a $1/2$-net of the unit sphere in $\mathbb R^{d_\beta}$,
with $|\mathcal V|\le5^{d_\beta}$. For any $\bm v\ne\bm0$, choose
$\bm u\in\mathcal V$ within $1/2$ of $\bm v/\|\bm v\|_2$. Then
\[
\begin{aligned}
\|\bm v\|_2
&=\left(\frac{\bm v}{\|\bm v\|_2}\right)^\top\bm v\\
&\le|\bm u^\top\bm v|
+\left\|\frac{\bm v}{\|\bm v\|_2}-\bm u\right\|_2\|\bm v\|_2\\
&\le\max_{\bm u\in\mathcal V}|\bm u^\top\bm v|+\tfrac12\|\bm v\|_2,
\end{aligned}
\qquad
\|\bm v\|_2\le2\max_{\bm u\in\mathcal V}|\bm u^\top\bm v|.
\]
For each $\bm u\in\mathcal V$, the scalar summands have mean zero and
$|\bm u^\top\bm g_i^{(e)}|\le\sqrt2R_c$. Scalar Hoeffding and a union bound give
\[
\begin{aligned}
\Pr(|\bm u^\top\bar{\bm g}^{(e)}|>t)
&\le2\exp\!\left(-\frac{n_et^2}{4R_c^2}\right),\\
\Pr(\|\bar{\bm g}^{(e)}\|_2>s)
&\le\sum_{\bm u\in\mathcal V}
\Pr(|\bm u^\top\bar{\bm g}^{(e)}|>s/2)\\
&\le2\cdot5^{d_\beta}\exp\!\left(-\frac{n_es^2}{16R_c^2}\right).
\end{aligned}
\]
At $s=4R_c\sqrt{\{d_\beta\log5+\log(4E/\delta)\}/n_e}$,
\[
2\cdot5^{d_\beta}\exp\{-d_\beta\log5-\log(4E/\delta)\}
=\frac{\delta}{2E}.
\]
Absorbing numerical constants gives
\[
\Pr\!\left(
\left\|\frac1{n_e}\sum_i\bm g_i^{(e)}\right\|_2
>C_\star R_c\sqrt{\frac{d_\beta+\log(2E/\delta)}{n_e}}
\right)\le\frac{\delta}{2E}.
\]
		A second union bound shows that, with probability at least $1-\delta$, both the Hessian event and the gradient event hold for all sources. Since $n_e\ge n_{\min}$, on this event
		\[
		Y:=\max_{1\le e\le E}\|\widehat{\bm{\beta}}^{(e)}-\bm{\beta}_*^{(e)}\|_F
		\le
		C_\star\frac{R_c}{\mu}\sqrt{\frac{d_\beta+\log(E/\delta)}{n_{\min}}},
		\]
		which proves the high-probability statement after absorbing constants.
		
        For the expectation bounds, set $\delta=e^{-t}$. The high-probability sample-size condition then has the form
        \[
        n_{\min}
        \ge
        \frac{C_\star}{\mu^2}
        \left[
        d_\beta\log\left(1+\frac{R_\beta L_H}{\mu}\right)
        +
        \log(d_\beta E)
        +
        t
        \right].
        \]
        Equivalently, after adjusting constants, this condition is satisfied whenever
        \[
        0\le t\le t_{\max},
        \qquad
        t_{\max}
        =
        c_\star\mu^2n_{\min}
        -
        d_\beta\log\left(1+\frac{R_\beta L_H}{\mu}\right)
        -
        \log(c_{\star,1}d_\beta E).
        \]
        Set
        \[
        q(t)=C_\star\frac{R_c}{\mu}
        \sqrt{\frac{d_\beta+\log E+t}{n_{\min}}},
        \qquad
        q'(t)=\frac{C_\star R_c}
        {2\mu\sqrt{n_{\min}}\sqrt{d_\beta+\log E+t}}.
        \]
        The tail bound and compactness give
        \[
        \Pr(Y>q(t))\le e^{-t}\quad(0\le t\le t_{\max}),
        \qquad 0\le Y\le2R_\beta.
        \]
        Splitting the first-moment integral at $q(0)$ and $q(t_{\max})$ yields
        \[
        \begin{aligned}
        \mathbb E Y
        &=\int_0^{2R_\beta}\Pr(Y>u)\,du\\
        &\le q(0)+\int_{q(0)}^{q(t_{\max})}\Pr(Y>u)\,du
        +2R_\beta e^{-t_{\max}}\\
        &=q(0)+\int_0^{t_{\max}}
        \Pr(Y>q(t))q'(t)\,dt+2R_\beta e^{-t_{\max}}\\
        &\le q(0)+
        \frac{C_\star R_c}{2\mu\sqrt{n_{\min}}}
        \int_0^{t_{\max}}\frac{e^{-t}}{\sqrt{d_\beta+\log E+t}}\,dt
        +2R_\beta e^{-t_{\max}}\\
        &\le q(0)+
        \frac{C_\star R_c}{2\mu\sqrt{n_{\min}(d_\beta+\log E)}}
        \int_0^\infty e^{-t}\,dt
        +2R_\beta e^{-t_{\max}}\\
        &\le\frac{3C_\star R_c}{2\mu}
        \sqrt{\frac{d_\beta+\log E}{n_{\min}}}
        +2R_\beta e^{-t_{\max}}.
        \end{aligned}
        \]
        The last inequality uses $d_\beta+\log E\ge1$.
        For the second moment,
        \[
        \begin{aligned}
        \mathbb E Y^2
        &=\int_0^{2R_\beta}2u\,\Pr(Y>u)\,du\\
        &\le q(0)^2+
        \int_{q(0)}^{q(t_{\max})}2u\,\Pr(Y>u)\,du
        +4R_\beta^2e^{-t_{\max}}\\
        &=q(0)^2+
        \int_0^{t_{\max}}2q(t)q'(t)\Pr(Y>q(t))\,dt
        +4R_\beta^2e^{-t_{\max}}\\
        &\le q(0)^2+
        \frac{C_\star^2R_c^2}{\mu^2n_{\min}}
        \int_0^{t_{\max}}e^{-t}\,dt
        +4R_\beta^2e^{-t_{\max}}\\
        &\le\frac{C_\star^2R_c^2(d_\beta+\log E+1)}
        {\mu^2n_{\min}}
        +4R_\beta^2e^{-t_{\max}}\\
        &\le\frac{2C_\star^2R_c^2(d_\beta+\log E)}
        {\mu^2n_{\min}}
        +4R_\beta^2e^{-t_{\max}}.
        \end{aligned}
        \]
        The exponential remainder is
        \[
        \begin{aligned}
        e^{-t_{\max}}
        &=
        c_{\star,1}d_\beta E
        \left(1+\frac{R_\beta L_H}{\mu}\right)^{d_\beta}
        e^{-c_\star\mu^2n_{\min}}.
        \end{aligned}
        \]
        A sufficiently large constant in the stated sample-size condition ensures
        \[
        \begin{aligned}
        d_\beta\log\left(1+\frac{R_\beta L_H}{\mu}\right)
        +\log(c_{\star,1}d_\beta E)
        &\le\tfrac{c_\star}{2}\mu^2n_{\min},\\
        t_{\max}&\ge\tfrac{c_\star}{2}\mu^2n_{\min},\\
        e^{-t_{\max}}&\le e^{-c_\star\mu^2n_{\min}/2}.
        \end{aligned}
        \]
        For $x>0$, both $x^{1/2}e^{-c_\star x/2}$ and
        $xe^{-c_\star x/2}$ are bounded. With $x=\mu^2n_{\min}$,
        \[
        e^{-t_{\max}}\le\frac{C_\star}{\mu\sqrt{n_{\min}}},
        \qquad
        e^{-t_{\max}}\le\frac{C_\star}{\mu^2n_{\min}}.
        \]
        Absorbing fixed regularity constants gives
        \[
        \mathbb E Y\le\mathcal O\!\left(
        \frac{R_c}{\mu}\sqrt{\frac{d_\beta+\log E}{n_{\min}}}\right),
        \qquad
        \mathbb E Y^2\le\mathcal O\!\left(
        \frac{R_c^2(d_\beta+\log E)}{\mu^2n_{\min}}\right).
        \]
        These bounds also cover growing $E$ subject to the stated sample-size condition.
	\end{proof}
	
	\subsection{Proof of Lemma~\ref{lem:wasserstein_concentration}}
	
	\begin{proof}
		Let
		\[
			X_e
			=
			\mathcal{W}_1(\nu_e,\widehat{\nu}_e).
		\]
		For compactly supported measures in $\mathbb{R}^{d_s}$, empirical optimal-transport theory gives constants $C_\star,c_\star>0$, independent of $n_e$, such that
		\begin{equation}
			\mathbb{E}X_e
			\le
			C_\star\gamma(n_e,d_s),
			\label{eq:w1_mean_single}
		\end{equation}
		where $\gamma(\cdot,\cdot)$ is the rate shorthand defined above
		\citep{fournier2015rate,villani2009optimal}. The same result gives the following centered tails:
		\begin{equation}
			\mathbb{P}\!\left(
			X_e>C_\star\gamma(n_e,d_s)+t
			\right)
			\le
			\begin{cases}
				C_\star\exp(-c_\star n_e t^2), & d_s<2,\\
				C_\star\exp\!\left(-c_\star n_e t^2/(\log(1+n_e))^2\right), & d_s=2,\\
				C_\star\exp(-c_\star n_e t^{d_s}), & d_s>2.
			\end{cases}
			\label{eq:w1_tail_single}
		\end{equation}
		For $n_e\ge n_{\min}$, $\gamma(n_e,d_s)\le C_\star\gamma(n_{\min},d_s)$. Enlarging constants to cover small sample sizes, for any $t>0$,
		\begin{align}
			&\mathbb{P}\left(
			\max_{1\le e\le E}X_e
			>
			C_\star\gamma(n_{\min},d_s)+t
			\right) \notag\\
			&\quad\le
			\sum_{e=1}^E
			\mathbb{P}\left(
			X_e
			>
			C_\star\gamma(n_e,d_s)+t
			\right).
			\label{eq:w1_union_bound}
		\end{align}
Substituting \eqref{eq:w1_tail_single} into \eqref{eq:w1_union_bound}, choose a sufficiently large constant in the threshold
		\begin{equation}
			T_E(n_{\min},d_s)
			=
			\begin{cases}
                C_\star\sqrt{\frac{1+\log E}{n_{\min}}}, & d_s<2,\\[3pt]
                C_\star\log(1+n_{\min})\sqrt{\frac{1+\log E}{n_{\min}}}, & d_s=2,\\[3pt]
                C_\star\left(\frac{1+\log E}{n_{\min}}\right)^{1/d_s}, & d_s>2.
			\end{cases}
			\label{eq:w1_threshold}
		\end{equation}
This threshold is $\Gamma(n_{\min},E,d_s)$ up to constants. For $d_s>2$, choose its constant so that $c_\star n_{\min}T_E^{d_s}\ge2(1+\log E)$; the corresponding transformed thresholds below satisfy the same lower bound in the other two dimension regimes.
		
		We now integrate the tail explicitly. Set
		\[
			M_E
			=
			\max_{1\le e\le E}X_e,
			\qquad
			a_E
			=
			C_\star\gamma(n_{\min},d_s)+T_E(n_{\min},d_s).
		\]
		By the layer-cake identity,
		\begin{align*}
			\mathbb{E}M_E
			&=
			\int_0^\infty
			\mathbb{P}(M_E>s)\,ds\\
			&\le
			a_E
			+
			\int_{a_E}^{\infty}
			\mathbb{P}(M_E>s)\,ds.
		\end{align*}
		For example, in the case $d_s>2$, write
		$s=C_\star\gamma(n_{\min},d_s)+r$ and use \eqref{eq:w1_union_bound}:
		\begin{align*}
			\int_{a_E}^{\infty}
			\mathbb{P}(M_E>s)\,ds
			&\le
			C_\star E
			\int_{T_E}^{\infty}
			\exp(-c_\star n_{\min} r^{d_s})\,dr\\
            &=
            \frac{C_\star E(c_\star n_{\min})^{-1/d_s}}{d_s}
            \int_{c_\star n_{\min}T_E^{d_s}}^\infty
            e^{-u}u^{1/d_s-1}\,du\\
			&\le
			C_\star
			\left(
			\frac{1+\log E}{n_{\min}}
			\right)^{1/d_s}.
		\end{align*}
For $d_s>2$, set $u_0=c_\star n_{\min}T_E^{d_s}$.
For $j\in\{1,2\}$, $j/d_s-1\le0$ and $u_0\ge2(1+\log E)>1$, so
\[
E\int_{u_0}^\infty e^{-u}u^{j/d_s-1}\,du
\le E\int_{u_0}^\infty e^{-u}\,du
=Ee^{-u_0}
\le e^{-2}E^{-1}\le e^{-2}.
\]
For $d_s\le2$, use $\ell_{d_s}(n_{\min})$ from Theorem~\ref{thm:param_bound}
and the substitution
\[
u=\frac{c_\star n_{\min}r^2}{\ell_{d_s}(n_{\min})^2},
\qquad
r=\frac{\ell_{d_s}(n_{\min})}{\sqrt{c_\star n_{\min}}}u^{1/2},
\qquad
dr=\frac{\ell_{d_s}(n_{\min})}{2\sqrt{c_\star n_{\min}}}u^{-1/2}\,du.
\]
The transformed threshold is
$u_0=c_\star n_{\min}T_E^2/\ell_{d_s}(n_{\min})^2\ge2(1+\log E)$.
For $j\in\{1,2\}$,
\[
\begin{aligned}
&E\int_{T_E}^\infty
r^{j-1}\exp\!\left(
-\frac{c_\star n_{\min}r^2}{\ell_{d_s}(n_{\min})^2}\right)\,dr\\
&\quad=\frac E2
\left(\frac{\ell_{d_s}(n_{\min})}{\sqrt{c_\star n_{\min}}}\right)^j
\int_{u_0}^\infty e^{-u}u^{j/2-1}\,du\\
&\quad\le\frac E2
\left(\frac{\ell_{d_s}(n_{\min})}{\sqrt{c_\star n_{\min}}}\right)^j
e^{-u_0}
\le C_\star\gamma(n_{\min},d_s)^j
\le C_\star\Gamma(n_{\min},E,d_s)^j.
\end{aligned}
\]
The cases $d_s<2$ and $d_s=2$ correspond to
$\ell_{d_s}(n_{\min})=1$ and $\log(1+n_{\min})$, respectively.
		\begin{equation}
			\mathbb{E}M_E
			\le
			\mathcal{O}\!\left(
			\gamma(n_{\min},d_s)
			+
			\Gamma(n_{\min},E,d_s)
			\right).
			\label{eq:w1_max_first_moment}
		\end{equation}
		
        For the second moment, the layer-cake identity gives
        \begin{align*}
        \mathbb E M_E^2
        &=\int_0^\infty 2s\,\Pr(M_E>s)\,ds\\
        &\le a_E^2+
        \int_{T_E}^{\infty}
        2\{C_\star\gamma(n_{\min},d_s)+r\}
        \Pr\{M_E>C_\star\gamma(n_{\min},d_s)+r\}\,dr.
        \end{align*}
        For $d_s>2$, the tail contribution is bounded by
        \[
        \begin{aligned}
        &2C_\star E\int_{T_E}^\infty
        \{C_\star\gamma(n_{\min},d_s)+r\}
        e^{-c_\star n_{\min}r^{d_s}}\,dr\\
        &\quad=
        2C_\star^2\gamma(n_{\min},d_s)E
        \int_{T_E}^\infty e^{-c_\star n_{\min}r^{d_s}}\,dr
        +2C_\star E\int_{T_E}^\infty r e^{-c_\star n_{\min}r^{d_s}}\,dr.
        \end{aligned}
        \]
        The first integral has already been bounded. For the second,
        \[
        \begin{aligned}
        E\int_{T_E}^\infty r e^{-c_\star n_{\min}r^{d_s}}\,dr
        &=
        \frac{E(c_\star n_{\min})^{-2/d_s}}{d_s}
        \int_{u_0}^\infty e^{-u}u^{2/d_s-1}\,du\\
        &\le C_\star n_{\min}^{-2/d_s}
        \le C_\star\Gamma(n_{\min},E,d_s)^2.
        \end{aligned}
        \]
        For $d_s\le2$, the preceding substitutions with $j=1,2$
        give the same bounds. In all three dimension regimes,
        \[
        \begin{aligned}
        \mathbb E M_E^2
        &\le a_E^2+
        C_\star\gamma(n_{\min},d_s)\Gamma(n_{\min},E,d_s)
        +C_\star\Gamma(n_{\min},E,d_s)^2\\
        &\le
        2C_\star^2\gamma(n_{\min},d_s)^2+2T_E^2
        +\frac{C_\star}{2}
        \{\gamma(n_{\min},d_s)^2+\Gamma(n_{\min},E,d_s)^2\}
        +C_\star\Gamma(n_{\min},E,d_s)^2\\
        &\le C_\star\{
        \gamma(n_{\min},d_s)^2+\Gamma(n_{\min},E,d_s)^2\}.
        \end{aligned}
        \]
		\begin{equation}
			\mathbb{E}M_E^2
			\le
			\mathcal{O}\!\left(
			\gamma(n_{\min},d_s)^2
			+
			\Gamma(n_{\min},E,d_s)^2
			\right).
			\label{eq:w1_max_second_moment}
		\end{equation}
		Equations \eqref{eq:w1_mean_single}, \eqref{eq:w1_max_first_moment}, and
		\eqref{eq:w1_max_second_moment} prove the lemma.
	\end{proof}
	
	\subsection{Proof of Lemma~\ref{lem:wasserstein_approx}}
	
	\begin{proof}
	We decouple the quantity in the lemma into target-domain empirical error and source-domain empirical errors:
    \begin{equation*}
    \begin{aligned}
    &L_\beta\sum_{e \in \widehat{\mathcal{N}}_K(0)}
    w_e\{\mathcal{W}_1(\nu_e,\widehat{\nu}_e)+\mathcal{W}_1(\widehat{\nu}_0,\nu_0)\}\\
    &=L_\beta\sum_{e \in \widehat{\mathcal{N}}_K(0)}w_e\,\mathcal{W}_1(\widehat{\nu}_0,\nu_0) + L_\beta\sum_{e \in \widehat{\mathcal{N}}_K(0)}w_e\,\mathcal{W}_1(\nu_e,\widehat{\nu}_e).
    \end{aligned}
    \end{equation*}
	
Since $\mathcal W_1(\widehat\nu_0,\nu_0)$ does not depend on the source index and $\sum_e w_e=1$,
	\begin{equation*}
	\sum_{e \in \widehat{\mathcal{N}}_K(0)} w_e \mathcal{W}_1(\widehat{\nu}_0, \nu_0) = \mathcal{W}_1(\widehat{\nu}_0, \nu_0) \sum_{e \in \widehat{\mathcal{N}}_K(0)} w_e
	\end{equation*}
	\begin{equation*}
	\mathbb{E} \left[ \sum_{e \in \widehat{\mathcal{N}}_K(0)} w_e \mathcal{W}_1(\widehat{\nu}_0, \nu_0) \right]
	= \mathbb{E} \left[ \mathcal{W}_1(\widehat{\nu}_0, \nu_0) \right]
	\le \mathcal{O}\left( \gamma(n_0, d_s) \right),
	\end{equation*}
	where the last inequality follows from Lemma~\ref{lem:wasserstein_concentration}.
	
	For the source part, non-negativity and normalization of the weights imply the pointwise bound
	\begin{equation*}
	\sum_{e \in \widehat{\mathcal{N}}_K(0)} w_e \mathcal{W}_1(\nu_e, \widehat{\nu}_e) \le \max_{e \in \widehat{\mathcal{N}}_K(0)} \mathcal{W}_1(\nu_e, \widehat{\nu}_e) \le \max_{1 \le e \le E} \mathcal{W}_1(\nu_e, \widehat{\nu}_e)
	\end{equation*}
	and Lemma~\ref{lem:wasserstein_concentration} gives
	\begin{equation*}
	\mathbb{E} \left[ \sum_{e \in \widehat{\mathcal{N}}_K(0)} w_e \mathcal{W}_1(\nu_e, \widehat{\nu}_e) \right]
	\le
	\mathcal{O}\left(\gamma(n_{\min},d_s)+\Gamma(n_{\min},E,d_s)\right).
	\end{equation*}
Using $\gamma(n_{\min},d_s)\le\Gamma(n_{\min},E,d_s)$ and adding the two parts gives
	\begin{equation*}
	\mathbb{E}\!\left[
	L_\beta\sum_{e\in\widehat{\mathcal{N}}_K(0)}
	w_e\{\mathcal{W}_1(\nu_e,\widehat{\nu}_e)+\mathcal{W}_1(\widehat{\nu}_0,\nu_0)\}
	\right]
	\le \mathcal{O}\Big( \gamma(n_0, d_s) + \Gamma(n_{\min}, E, d_s) \Big).
	\end{equation*}
	\end{proof}
	
	\subsection{Proof of Lemma~\ref{lem:knn_interpolation}}
	
	\begin{proof}
	\textbf{Weighted radius reduction.}
	Since $w_e \ge 0$ and $\sum w_e = 1$, the weighted convex combination is bounded by the maximum distance within the selected set. Let $\widehat{R}_K$ be the distance of the $K$-th nearest neighbor found by the algorithm in the empirical space (i.e., the empirical KNN radius):
	\begin{equation*}
	L_\beta \sum_{e \in \widehat{\mathcal{N}}_K(0)} w_e \mathcal{W}_1(\widehat{\nu}_e, \widehat{\nu}_0)
	\le L_\beta \max_{e \in \widehat{\mathcal{N}}_K(0)} \mathcal{W}_1(\widehat{\nu}_e, \widehat{\nu}_0) := L_\beta \widehat{R}_K .
	\end{equation*}
	
	\textbf{From empirical neighbors to true neighbors.}
	We define the set $\mathcal{N}^*_K(0)$ as the $K$ true nearest neighbors of the target domain $\nu_0$ on the true measure manifold $\mathcal{M}$. By construction, $\widehat{\mathcal{N}}_K(0)$ consists of the $K$ domains with the smallest empirical distances among all $E$ source domains. Thus, the maximum empirical distance among these $K$ domains, $\widehat{R}_K$, is the $K$-th order statistic of empirical distances across the global set.
	Conversely, under the empirical metric, the true nearest-neighbor set $\mathcal{N}^*_K(0)$ is another subset of size $K$ among the $E$ domains. The maximum over any subset of size $K$ is at least the maximum over the $K$ empirically closest domains. Hence, the following inequality holds deterministically:
	\begin{equation*}
	\widehat{R}_K = \max_{e \in \widehat{\mathcal{N}}_K(0)} \mathcal{W}_1(\widehat{\nu}_e, \widehat{\nu}_0) \le \max_{j \in \mathcal{N}^*_K(0)} \mathcal{W}_1(\widehat{\nu}_j, \widehat{\nu}_0)
	\end{equation*}
	Applying the Wasserstein distance triangle inequality to the true neighbor $j \in \mathcal{N}^*_K(0)$ and leveraging the sub-additivity of the maximum operator:
	\begin{equation*}
	\mathcal{W}_1(\widehat{\nu}_j, \widehat{\nu}_0) \le \mathcal{W}_1(\widehat{\nu}_j, \nu_j) + \mathcal{W}_1(\nu_j, \nu_0) + \mathcal{W}_1(\nu_0, \widehat{\nu}_0)
	\end{equation*}
	\begin{equation*}
	\widehat{R}_K \le R_K^* + \max_{j \in \mathcal{N}^*_K(0)} \mathcal{W}_1(\widehat{\nu}_j, \nu_j) + \mathcal{W}_1(\nu_0, \widehat{\nu}_0),
	\end{equation*}
	where $R_K^*:=\max_{j \in \mathcal{N}^*_K(0)}\mathcal{W}_1(\nu_j,\nu_0)$. Taking the mathematical expectation on both sides, the latter two terms are bounded by $\mathcal{O}\left( \gamma(n_0, d_s) + \Gamma(n_{\min}, E, d_s) \right)$ as established in Lemma~\ref{lem:wasserstein_approx}. The remaining task is to evaluate $\mathbb{E}[R_K^*]$.
	
    \textbf{Order statistic of the true KNN radius.}
Condition on the target law $\nu_0$. For $R_e=\mathcal W_1(\nu_e,\nu_0)$,
\[
F_{\nu_0}(r)=\Pr(R_e\le r\mid\nu_0)
=\Pi_{\mathcal M}(B_{\mathcal W_1}(\nu_0,r)).
\]
All distances are at most $D_s$. Uniform Ahlfors regularity provides
$a>0$ and $0<r_0\le D_s$ with $F_{\nu_0}(r)\ge ar^{d_{\mathcal M}}$
for $0\le r\le r_0$. For larger radii,
\[
F_{\nu_0}(r)\ge F_{\nu_0}(r_0)
\ge ar_0^{d_{\mathcal M}}
\ge a(r_0/D_s)^{d_{\mathcal M}}r^{d_{\mathcal M}},
\quad r_0<r\le D_s.
\]
With $a_0=a(r_0/D_s)^{d_{\mathcal M}}$, this gives a global bound
\[
F_{\nu_0}(r)\ge a_0r^{d_{\mathcal M}}
\quad(0\le r\le D_s),\qquad
R_K^*\le a_0^{-1/d_{\mathcal M}}
F_{\nu_0}(R_K^*)^{1/d_{\mathcal M}}.
\]
The assumed continuity of $F_{\nu_0}$ implies that
$U_e=F_{\nu_0}(R_e)$ are conditionally i.i.d. Uniform$(0,1)$.
Monotonicity of $F_{\nu_0}$ gives
$U_{(K)}=F_{\nu_0}(R_K^*)$. For $0<u<1$,
\[
\begin{aligned}
\Pr(U_{(K)}\le u\mid\nu_0)
&=\Pr\!\left(\sum_{e=1}^E\mathbf1_{\{U_e\le u\}}\ge K\,\middle|\,\nu_0\right)\\
&=\sum_{j=K}^E\binom Ej u^j(1-u)^{E-j},\\
f_{U_{(K)}\mid\nu_0}(u)
&=\frac{d}{du}\sum_{j=K}^E\binom Ej u^j(1-u)^{E-j}\\
&=\frac{E!}{(K-1)!(E-K)!}u^{K-1}(1-u)^{E-K}.
\end{aligned}
\]
Its first moment is
\[
\begin{aligned}
\mathbb E[U_{(K)}\mid\nu_0]
&=\frac{E!}{(K-1)!(E-K)!}\int_0^1u^K(1-u)^{E-K}\,du\\
&=\frac{E!}{(K-1)!(E-K)!}\frac{K!(E-K)!}{(E+1)!}
=\frac K{E+1}.
\end{aligned}
\]
Since $d_{\mathcal M}\in\mathbb N$, $u^{1/d_{\mathcal M}}$ is concave.
The radius bound and Jensen's inequality yield
\[
\begin{aligned}
\mathbb E[R_K^*\mid\nu_0]
&\le a_0^{-1/d_{\mathcal M}}
\mathbb E[U_{(K)}^{1/d_{\mathcal M}}\mid\nu_0]\\
&\le a_0^{-1/d_{\mathcal M}}
\{\mathbb E[U_{(K)}\mid\nu_0]\}^{1/d_{\mathcal M}}\\
&=a_0^{-1/d_{\mathcal M}}\left(\frac K{E+1}\right)^{1/d_{\mathcal M}}
\le C_\star(K/E)^{1/d_{\mathcal M}}.
\end{aligned}
\]
The constants are uniform in $\nu_0$. For an independently sampled target,
\[
\mathbb E R_K^*
=\mathbb E_{\nu_0}\mathbb E[R_K^*\mid\nu_0]
\le C_\star(K/E)^{1/d_{\mathcal M}}.
\]
The same bound holds for a fixed target satisfying the stated regularity condition.

	The empirical-measure and nearest-neighbor radius bounds give:
	\begin{equation*}
	\mathbb{E}\!\left[
	L_\beta \sum_{e \in \widehat{\mathcal{N}}_K(0)} w_e \mathcal{W}_1(\widehat{\nu}_e, \widehat{\nu}_0)
	\right]
	\le L_\beta \mathcal{O}\left( \left(\frac{K}{E}\right)^{\frac{1}{d_{\mathcal{M}}}} \right) + L_\beta \mathcal{O}\left( \gamma(n_0, d_s) + \Gamma(n_{\min}, E, d_s) \right).
	\end{equation*}
		External Lipschitz constants are absorbed, completing the proof.
	\end{proof}
	
	\subsection{Proof of Theorem~\ref{thm:param_bound}}
	
	\begin{proof}
Write the truncated reweighted estimator as
		\begin{equation*}
			\widehat{\bm{\beta}}^{(0)}
			=
			\sum_{e \in \widehat{\mathcal{N}}_K(0)} w_e \widehat{\bm{\beta}}^{(e)},
			\qquad
			w_e\ge 0,\quad
			\sum_{e \in \widehat{\mathcal{N}}_K(0)}w_e=1.
		\end{equation*}
		The normalization of the weights allows us to write
		\begin{equation*}
			\bm{\beta}_*^{(0)}
			=
			\sum_{e \in \widehat{\mathcal{N}}_K(0)}w_e\bm{\beta}_*^{(0)}.
		\end{equation*}
		Adding and subtracting $\bm{\beta}_*^{(e)}$ inside each summand gives the exact identity
		\begin{align*}
			\widehat{\bm{\beta}}^{(0)}-\bm{\beta}_*^{(0)}
			&=
			\sum_{e \in \widehat{\mathcal{N}}_K(0)}w_e
			\left(\widehat{\bm{\beta}}^{(e)}-\bm{\beta}_*^{(0)}\right)\\
			&=
			\sum_{e \in \widehat{\mathcal{N}}_K(0)}w_e
			\left[
			\left(\widehat{\bm{\beta}}^{(e)}-\bm{\beta}_*^{(e)}\right)
			+
			\left(\bm{\beta}_*^{(e)}-\bm{\beta}_*^{(0)}\right)
			\right].
		\end{align*}
		Taking Frobenius norms, applying the triangle inequality, and using $w_e\ge 0$ yields
        \begin{align*}
            \|\widehat{\bm{\beta}}^{(0)}-\bm{\beta}_*^{(0)}\|_F
            \le&
            \sum_{e \in \widehat{\mathcal{N}}_K(0)}w_e\|\widehat{\bm{\beta}}^{(e)}-\bm{\beta}_*^{(e)}\|_F + \sum_{e \in \widehat{\mathcal{N}}_K(0)}w_e\|\bm{\beta}_*^{(e)}-\bm{\beta}_*^{(0)}\|_F.
        \end{align*}
		Assumption~\ref{app:ass:smooth_geometry_full} gives $\bm{\beta}_*^{(e)}=\Psi^*(\nu_e)$ and $\bm{\beta}_*^{(0)}=\Psi^*(\nu_0)$, so
		\begin{equation*}
			\|\bm{\beta}_*^{(e)}-\bm{\beta}_*^{(0)}\|_F
			=
			\|\Psi^*(\nu_e)-\Psi^*(\nu_0)\|_F
			\le
			L_\beta\mathcal{W}_1(\nu_e,\nu_0).
		\end{equation*}
		For each selected source domain, the Wasserstein triangle inequality gives
		\begin{align*}
			\mathcal{W}_1(\nu_e,\nu_0)
			&\le
			\mathcal{W}_1(\nu_e,\widehat{\nu}_e)
			+
			\mathcal{W}_1(\widehat{\nu}_e,\widehat{\nu}_0)
			+
			\mathcal{W}_1(\widehat{\nu}_0,\nu_0).
		\end{align*}
		The resulting three-term decomposition is:
		\begin{align*}
			\|\widehat{\bm{\beta}}^{(0)} - \bm{\beta}_*^{(0)}\|_F
			\le&
			\underbrace{\sum_{e \in \widehat{\mathcal{N}}_K(0)} w_e \|\widehat{\bm{\beta}}^{(e)} - \bm{\beta}_*^{(e)}\|_F}_{\text{Term A}} \\
			&+
			\underbrace{L_\beta \sum_{e \in \widehat{\mathcal{N}}_K(0)} w_e \Big( \mathcal{W}_1(\nu_e, \widehat{\nu}_e) + \mathcal{W}_1(\widehat{\nu}_0, \nu_0) \Big)}_{\text{Term B}} \\
			&+
			\underbrace{L_\beta \sum_{e \in \widehat{\mathcal{N}}_K(0)} w_e \mathcal{W}_1(\widehat{\nu}_e, \widehat{\nu}_0)}_{\text{Term C}}.
		\end{align*}
		Taking the unconditional mathematical expectation $\mathbb{E}[\cdot]$ on both sides, we bound the three terms separately.
		
		For \textbf{Term A}, the possible dependence between the empirically computed neighbor set $\widehat{\mathcal{N}}_K(0)$ and the parameter estimators is handled by the deterministic maximum bound
		\begin{equation*}
			\sum_{e \in \widehat{\mathcal{N}}_K(0)} w_e \|\widehat{\bm{\beta}}^{(e)} - \bm{\beta}_*^{(e)}\|_F \le \max_{e \in \widehat{\mathcal{N}}_K(0)} \|\widehat{\bm{\beta}}^{(e)} - \bm{\beta}_*^{(e)}\|_F \le \max_{1 \le e \le E} \|\widehat{\bm{\beta}}^{(e)} - \bm{\beta}_*^{(e)}\|_F
		\end{equation*}
		Lemma~\ref{lem:estimator_tail} then directly implies
		\begin{equation*}
			\mathbb{E}[\text{Term A}]
			\le
			\mathcal{O}\left(\frac{R_c}{\mu}\sqrt{\frac{d_\beta+\log E}{n_{\min}}}\right).
		\end{equation*}

		For \textbf{Term B}, Lemma~\ref{lem:wasserstein_approx} gives the explicit empirical-measure calculation
		\begin{align*}
			\mathbb{E}[\text{Term B}]
			&=
			L_\beta\,
			\mathbb{E}
			\left[
			\sum_{e\in\widehat{\mathcal{N}}_K(0)}
			w_e
			\left\{
			\mathcal{W}_1(\nu_e,\widehat{\nu}_e)
			+
			\mathcal{W}_1(\widehat{\nu}_0,\nu_0)
			\right\}
			\right]\\
			&\le
			L_\beta
			\left[
			\mathbb{E}
			\max_{1\le e\le E}
			\mathcal{W}_1(\nu_e,\widehat{\nu}_e)
			+
			\mathbb{E}
			\mathcal{W}_1(\widehat{\nu}_0,\nu_0)
			\right]\\
			&\le
			\mathcal{O}\left(
			\Gamma(n_{\min},E,d_s)
			+
			\gamma(n_0,d_s)
			\right).
		\end{align*}
		The first inequality uses $w_e\ge0$ and $\sum_e w_e=1$; the second uses Lemma~\ref{lem:wasserstein_concentration}.
		
		For \textbf{Term C}, Lemma~\ref{lem:knn_interpolation} gives
		\begin{align*}
			\mathbb{E}[\text{Term C}]
			&=
			L_\beta\,
			\mathbb{E}
			\left[
			\sum_{e\in\widehat{\mathcal{N}}_K(0)}
			w_e
			\mathcal{W}_1(\widehat{\nu}_e,\widehat{\nu}_0)
			\right]\\
			&\le
			L_\beta\,
			\mathbb{E}\widehat R_K\\
			&\le
			L_\beta
			\mathbb{E}R_K^*
			+
			L_\beta
			\mathbb{E}
			\max_{1\le e\le E}
			\mathcal{W}_1(\widehat{\nu}_e,\nu_e)
			+
			L_\beta
			\mathbb{E}
			\mathcal{W}_1(\widehat{\nu}_0,\nu_0)\\
			&\le
			\mathcal{O}\left(
			\left(\frac{K}{E}\right)^{1/d_{\mathcal{M}}}
			+
			\Gamma(n_{\min},E,d_s)
			+
			\gamma(n_0,d_s)
			\right).
		\end{align*}
		Here $\widehat R_K$ is the empirical top-$K$ radius and $R_K^*$ is the oracle top-$K$ radius on the population space of style distributions; the last line uses the order-statistic calculation in Lemma~\ref{lem:knn_interpolation}.
		
		Summing the three expected bounds gives
        \begin{align*}
            \mathbb{E}
            \left[
            \left\|
            \widehat{\bm{\beta}}^{(0)}
            -
            \bm{\beta}_*^{(0)}
            \right\|_F
            \right]
            \le\;&
            \mathcal{O}\left(
            \frac{R_c}{\mu}\sqrt{\frac{d_\beta+\log E}{n_{\min}}}
            \right)
            +
            \mathcal{O}\left(
            \gamma(n_0,d_s)
            +
            \Gamma(n_{\min},E,d_s)
            \right)\\
            &+
            \mathcal{O}\left(
            \left(\frac{K}{E}\right)^{1/d_{\mathcal{M}}}
            \right).
        \end{align*}

	\end{proof}
	
	\subsection{Proof of Theorem~\ref{thm:excess_risk}}
	
	\begin{proof}
		Here $\mathcal R_0$ is the population cross-entropy risk of the fixed target domain, evaluated on a fresh observation independent of the samples used to construct the predictor and its target summary. The fitted predictor is held fixed inside this risk; the outer expectation averages the random quantities used to construct it.
		\textbf{Smoothness reduction.}
        By Assumption~\ref{app:ass:glm_regularity_full}, $\bm{\beta}_*^{(0)}$ is a relative-interior population optimum and $\nabla_{\mathcal{S}_\beta}\mathcal{R}_0(\bm{\beta}_*^{(0)})=\bm{0}$. Set
        \[
        \Delta_0:=\widehat{\bm{\beta}}^{(0)}-\bm{\beta}_*^{(0)}.
        \]
        Since both $\widehat{\bm{\beta}}^{(0)}$ and $\bm{\beta}_*^{(0)}$ belong to $\mathcal{B}\subset\mathcal{S}_\beta$, we have $\Delta_0\in\mathcal{S}_\beta$. Hence the first-order term vanishes in the centered-gauge tangent space:
        \[
        \left\langle
        \nabla\mathcal{R}_0(\bm{\beta}_*^{(0)}),\Delta_0
        \right\rangle
        =
        \left\langle
        \nabla_{\mathcal{S}_\beta}\mathcal{R}_0(\bm{\beta}_*^{(0)}),\Delta_0
        \right\rangle
        =0.
        \]
        A second-order Taylor expansion of the target-domain expected risk at $\bm{\beta}_*^{(0)}$ gives
        \begin{equation*}
            \mathcal{R}_0(\widehat{\bm{\beta}}^{(0)}) - \mathcal{R}_0(\bm{\beta}_*^{(0)}) = \frac{1}{2}\operatorname{vec}(\Delta_0)^\top \nabla^2 \mathcal{R}_0(\tilde{\bm{\beta}}) \operatorname{vec}(\Delta_0)
        \end{equation*}
		where $\tilde{\bm{\beta}}$ lies between $\widehat{\bm{\beta}}^{(0)}$ and $\bm{\beta}_*^{(0)}$. Based on standard properties of matrix operator quadratic forms, the second-order expansion term is bounded by the product of the population Hessian matrix's spectral norm and the squared $L_2$ norm of the vector:
		\begin{equation*}
			\frac{1}{2} \text{vec}(\widehat{\bm{\beta}}^{(0)} - \bm{\beta}_*^{(0)})^\top \nabla^2 \mathcal{R}_0(\tilde{\bm{\beta}}) \text{vec}(\widehat{\bm{\beta}}^{(0)} - \bm{\beta}_*^{(0)}) \le \frac{1}{2} \|\nabla^2 \mathcal{R}_0(\tilde{\bm{\beta}})\|_2 \cdot \|\text{vec}(\widehat{\bm{\beta}}^{(0)} - \bm{\beta}_*^{(0)})\|_2^2
		\end{equation*}
		We note that the single-sample Hessian matrix structure for the multi-class cross-entropy loss is:
		\begin{equation*}
			\nabla^2_{\bm{\beta}} \ell_{\mathrm{CE}}(\bm{\beta}\bm{Z}_c, Y) = (\bm{Z}_c \bm{Z}_c^\top) \otimes (\mathrm{diag}(\bm{p}) - \bm{p}\bm{p}^\top)
		\end{equation*}
		where $\otimes$ denotes the Kronecker product. Because the maximum eigenvalue of the covariance matrix on the simplex $\mathrm{diag}(\bm{p}) - \bm{p} \bm{p}^\top$ is bounded by $\frac{1}{2}$, and $\|\bm{Z}_c\|_2 \le R_c$, the spectral norm satisfies:
		\begin{equation*}
			\|\nabla^2_{\bm{\beta}} \ell_{\mathrm{CE}}(\bm{\beta}\bm{Z}_c, Y)\|_2 = \|\mathrm{diag}(\bm{p}) - \bm{p} \bm{p}^\top\|_2 \cdot \|\bm{Z}_c \bm{Z}_c^\top\|_2 \le \frac{1}{2} \|\bm{Z}_c\|_2^2 \le \frac{1}{2} R_c^2
		\end{equation*}
		By the sub-additivity and monotonicity of the expectation operator, the spectral norm of the population Hessian matrix is bounded over the compact convex set:
		\begin{equation*}
			\|\nabla^2 \mathcal{R}_0(\tilde{\bm{\beta}})\|_2 \le \mathbb{E} \left[ \|\nabla^2_{\bm{\beta}} \ell_{\mathrm{CE}}(\tilde{\bm{\beta}}\bm{Z}_c, Y)\|_2 \right] \le \frac{1}{2} R_c^2
		\end{equation*}
		Also, the $L_2$ norm of the vectorized matrix exactly equals the Frobenius norm of the original matrix:
		\begin{equation*}
			\|\text{vec}(\widehat{\bm{\beta}}^{(0)} - \bm{\beta}_*^{(0)})\|_2^2 = \left\| \widehat{\bm{\beta}}^{(0)} - \bm{\beta}_*^{(0)} \right\|_F^2
		\end{equation*}
		The expected risk is $R_c^2/2$-smooth over the compact domain:
		\begin{equation*}
			\mathcal{R}_0(\widehat{\bm{\beta}}^{(0)}) - \mathcal{R}_0(\bm{\beta}_*^{(0)}) \le \frac{R_c^2}{4} \left\| \widehat{\bm{\beta}}^{(0)} - \bm{\beta}_*^{(0)} \right\|_F^2
		\end{equation*}
		Taking the external expectation gives the dominant inequality:
		\begin{equation*}
			\mathbb{E}\left[ \mathcal{R}_0(\widehat{\bm{\beta}}^{(0)}) - \mathcal{R}_0(\bm{\beta}_*^{(0)}) \right] \le \frac{R_c^2}{4} \mathbb{E}\left[ \left\| \widehat{\bm{\beta}}^{(0)} - \bm{\beta}_*^{(0)} \right\|_F^2 \right]
		\end{equation*}
		
		\textbf{Second moment of the decomposed parameter error.}
		We recall the triangle inequality of the error decoupling from Theorem~\ref{thm:param_bound}:
		\begin{equation*}
			\left\| \widehat{\bm{\beta}}^{(0)} - \bm{\beta}_*^{(0)} \right\|_F \le \text{Term A} + \text{Term B} + \text{Term C}
		\end{equation*}
		Using the algebraic inequality $(a+b+c)^2 \le 3(a^2 + b^2 + c^2)$, we scale the square:
		\begin{equation*}
			\left\| \widehat{\bm{\beta}}^{(0)} - \bm{\beta}_*^{(0)} \right\|_F^2 \le 3 \left( (\text{Term A})^2 + (\text{Term B})^2 + (\text{Term C})^2 \right)
		\end{equation*}
		
		For \textbf{$(\text{Term A})^2$}, Lemma~\ref{lem:estimator_tail} gives the squared maximal estimator error:
		\begin{equation*}
			\mathbb{E}\left[ (\text{Term A})^2 \right]
			\le \mathbb{E}\left[ \max_{1 \le e \le E} \|\widehat{\bm{\beta}}^{(e)} - \bm{\beta}_*^{(e)}\|_F^2 \right]
			\le \mathcal{O}\left( \frac{R_c^2(d_\beta+\log E)}{\mu^2 n_{\min}} \right)
		\end{equation*}

		For \textbf{$(\text{Term B})^2$}, use the same deterministic decomposition as in Lemma~\ref{lem:wasserstein_approx}:
		\begin{equation*}
			\text{Term B}
			\le
			L_\beta\left(
			\mathcal{W}_1(\widehat{\nu}_0,\nu_0)
			+\max_{1\le e\le E}\mathcal{W}_1(\nu_e,\widehat{\nu}_e)
			\right).
		\end{equation*}
        Using $(a+b)^2\le2(a^2+b^2)$ and Lemma~\ref{lem:wasserstein_concentration},
\[
\begin{aligned}
\mathbb E[(\text{Term B})^2]
&\le2L_\beta^2\left[
\mathbb E\mathcal W_1(\widehat\nu_0,\nu_0)^2
+\mathbb E\max_{1\le e\le E}\mathcal W_1(\nu_e,\widehat\nu_e)^2\right]\\
&\le C_\star L_\beta^2\left[
\gamma(n_0,d_s)^2+\gamma(n_{\min},d_s)^2
+\Gamma(n_{\min},E,d_s)^2\right]\\
&\le\mathcal O\!\left(\gamma(n_0,d_s)^2+\Gamma(n_{\min},E,d_s)^2\right).
\end{aligned}
\]
The last step uses $\gamma(n_{\min},d_s)\le\Gamma(n_{\min},E,d_s)$.

		For \textbf{$(\text{Term C})^2$}, the proof of Lemma~\ref{lem:knn_interpolation} gives
		\begin{equation*}
			\text{Term C}
			\le
			L_\beta\left(
			R_K^*
			+\max_{1\le e\le E}\mathcal{W}_1(\widehat{\nu}_e,\nu_e)
			+\mathcal{W}_1(\widehat{\nu}_0,\nu_0)
			\right).
		\end{equation*}
        For the radius term, use $U_{(K)}=F_{\nu_0}(R_K^*)$
and $R_K^*\le a_0^{-1/d_{\mathcal M}}U_{(K)}^{1/d_{\mathcal M}}$
from the proof of Lemma~\ref{lem:knn_interpolation}.
For $d_{\mathcal M}=1$, the displayed order-statistic density gives
\[
\begin{aligned}
\mathbb E[U_{(K)}^2\mid\nu_0]
&=\frac{E!}{(K-1)!(E-K)!}\int_0^1u^{K+1}(1-u)^{E-K}\,du\\
&=\frac{E!}{(K-1)!(E-K)!}\frac{(K+1)!(E-K)!}{(E+2)!}\\
&=\frac{K(K+1)}{(E+1)(E+2)}\le2(K/E)^2,\\
\mathbb E[(R_K^*)^2\mid\nu_0]
&\le a_0^{-2}\mathbb E[U_{(K)}^2\mid\nu_0]
\le2a_0^{-2}(K/E)^2.
\end{aligned}
\]
For $d_{\mathcal M}\ge2$, $u^{2/d_{\mathcal M}}$ is concave, so
\[
\begin{aligned}
\mathbb E[(R_K^*)^2\mid\nu_0]
&\le a_0^{-2/d_{\mathcal M}}\mathbb E[U_{(K)}^{2/d_{\mathcal M}}\mid\nu_0]\\
&\le a_0^{-2/d_{\mathcal M}}
\{\mathbb E[U_{(K)}\mid\nu_0]\}^{2/d_{\mathcal M}}\\
&=a_0^{-2/d_{\mathcal M}}\left(\frac K{E+1}\right)^{2/d_{\mathcal M}}
\le C_\star(K/E)^{2/d_{\mathcal M}}.
\end{aligned}
\]
The uniform bounds cover a fixed target and, by iterated expectation,
an independent random target:
\[
\mathbb E(R_K^*)^2\le C_\star(K/E)^{2/d_{\mathcal M}}.
\]
Applying $(a+b+c)^2\le3(a^2+b^2+c^2)$ to Term C gives
\[
\begin{aligned}
\mathbb E[(\text{Term C})^2]
&\le3L_\beta^2\left[
\mathbb E(R_K^*)^2
+\mathbb E\max_{1\le e\le E}\mathcal W_1(\widehat\nu_e,\nu_e)^2
+\mathbb E\mathcal W_1(\widehat\nu_0,\nu_0)^2\right]\\
&\le C_\star L_\beta^2\left[
(K/E)^{2/d_{\mathcal M}}+\gamma(n_{\min},d_s)^2
+\Gamma(n_{\min},E,d_s)^2+\gamma(n_0,d_s)^2\right]\\
&\le\mathcal O\!\left(
(K/E)^{2/d_{\mathcal M}}+\Gamma(n_{\min},E,d_s)^2+\gamma(n_0,d_s)^2\right).
\end{aligned}
\]

        The three bounds yield the parameter second moment
\[
\begin{aligned}
\mathbb E\|\widehat{\bm\beta}^{(0)}-\bm\beta_*^{(0)}\|_F^2
&\le3\left\{\mathbb E[(\text{Term A})^2]
+\mathbb E[(\text{Term B})^2]+\mathbb E[(\text{Term C})^2]\right\}\\
&\le C_\star\left[
\frac{R_c^2(d_\beta+\log E)}{\mu^2n_{\min}}
+\gamma(n_0,d_s)^2+\Gamma(n_{\min},E,d_s)^2
+(K/E)^{2/d_{\mathcal M}}\right].
\end{aligned}
\]
Substitution into the smoothness bound gives
\[
\begin{aligned}
&\mathbb E\{\mathcal R_0(\widehat{\bm\beta}^{(0)})-\mathcal R_0(\bm\beta_*^{(0)})\}\\
&\quad\le\frac{R_c^2}{4}
\mathbb E\|\widehat{\bm\beta}^{(0)}-\bm\beta_*^{(0)}\|_F^2\\
&\quad\le\frac{C_\star R_c^2}{4}\left[
\frac{R_c^2(d_\beta+\log E)}{\mu^2n_{\min}}
+\gamma(n_0,d_s)^2+\Gamma(n_{\min},E,d_s)^2
+(K/E)^{2/d_{\mathcal M}}\right]\\
&\quad\le\frac{C_\star R_c^2}{4}\left[
\frac{R_c}{\mu}\sqrt{\frac{d_\beta+\log E}{n_{\min}}}
+\gamma(n_0,d_s)+\Gamma(n_{\min},E,d_s)
+(K/E)^{1/d_{\mathcal M}}\right]^2\\
&\quad\le\mathcal O\!\left(R_c^2\left[
\frac{R_c}{\mu}\sqrt{\frac{d_\beta+\log E}{n_{\min}}}
+\gamma(n_0,d_s)+\Gamma(n_{\min},E,d_s)
+(K/E)^{1/d_{\mathcal M}}\right]^2\right).
\end{aligned}
\]
The penultimate step uses $\sum_j a_j^2\le(\sum_j a_j)^2$ for nonnegative $a_j$.
	\end{proof}

	\subsection{Proof of Lemma~\ref{lem:centered_logit_excess}}

	\begin{proof}
		For a fixed pair $(\bm{Z}_c,c)$, define the conditional cross-entropy risk as a function of the logit vector $\bm{u}\in\mathbb{R}^C$:
		\begin{equation*}
			\rho_{\bm{Z}_c,c}(\bm{u})
			=
			\mathbb{E}\!\left[
			\ell_{\mathrm{CE}}(\bm{u},Y)
			\mid
			\bm{Z}_c,c
			\right].
		\end{equation*}
		Because the softmax probabilities are invariant under common logit translations,
		\begin{equation*}
			\mathrm{Softmax}(\bm{u}+a\bm{1})
			=
			\mathrm{Softmax}(\bm{u})
			\qquad\text{for every }a\in\mathbb{R},
		\end{equation*}
		the conditional risk depends only on the centered component of the logits. Hence
		\begin{equation*}
			\rho_{\bm{Z}_c,c}(\bm{u})
			=
			\rho_{\bm{Z}_c,c}(\bm{\Pi}_C\bm{u}).
		\end{equation*}
		Let $\bm{u}=\bm{\Pi}_C f(\bm{Z})$ and $\bm{u}^*=\bm{\Pi}_C f^*(\bm{Z}_c,c)$. For the lower bound, assume $\bm{u},\bm{u}^*\in\mathcal{U}_R$ almost surely. By Assumption~\ref{app:ass:glm_regularity_full}, $\rho_{\bm{Z}_c,c}$ is $\mu_{\ell}$-strongly convex on $\mathcal{U}_R$. Therefore,
		\begin{equation*}
			\rho_{\bm{Z}_c,c}(\bm{u})
			\ge
			\rho_{\bm{Z}_c,c}(\bm{u}^*)
			+
			\left\langle
			\nabla_{\mathcal{S}_C}\rho_{\bm{Z}_c,c}(\bm{u}^*),
			\bm{u}-\bm{u}^*
			\right\rangle
			+
			\frac{\mu_{\ell}}{2}\|\bm{u}-\bm{u}^*\|_2^2,
		\end{equation*}
		where $\nabla_{\mathcal S_C}$ is the gradient restricted to $\mathcal S_C$.
Set $\bm p_* = \mathrm{Softmax}(f^*(\bm Z_c,c))$. The conditional label law gives
\[
\begin{aligned}
\rho_{\bm Z_c,c}(\bm u)
&=-\sum_{j=1}^C p_{*,j}
\log\!\left(\frac{e^{u_j}}{\sum_k e^{u_k}}\right)\\
&=\log\!\left(\sum_k e^{u_k}\right)-\bm p_*^\top\bm u,\\
\nabla\rho_{\bm Z_c,c}(\bm u)
&=\mathrm{Softmax}(\bm u)-\bm p_*,\\
\mathrm{Softmax}(\bm u^*)
&=\mathrm{Softmax}(\bm\Pi_C f^*(\bm Z_c,c))
=\mathrm{Softmax}(f^*(\bm Z_c,c))=\bm p_*,\\
\nabla_{\mathcal S_C}\rho_{\bm Z_c,c}(\bm u^*)
&=\bm\Pi_C(\mathrm{Softmax}(\bm u^*)-\bm p_*)=\bm0.
\end{aligned}
\]
The first-order term is therefore
		\begin{equation*}
			\left\langle
			\nabla_{\mathcal{S}_C}\rho_{\bm{Z}_c,c}(\bm{u}^*),
			\bm{u}-\bm{u}^*
			\right\rangle
			=
			0.
		\end{equation*}
		Substituting $\bm{u}=\bm{\Pi}_C f(\bm{Z})$ and $\bm{u}^*=\bm{\Pi}_C f^*(\bm{Z}_c,c)$ gives the pointwise inequality
		\begin{equation*}
			\rho_{\bm{Z}_c,c}(f(\bm{Z}))
			-
			\rho_{\bm{Z}_c,c}(f^*(\bm{Z}_c,c))
			\ge
			\frac{\mu_{\ell}}{2}
			\left\|
			\bm{\Pi}_C\!\left(f(\bm{Z})-f^*(\bm{Z}_c,c)\right)
			\right\|_2^2.
		\end{equation*}
		This proves the lower bound after taking expectation over $(\bm{Z}_c,\bm{Z}_s,c)$.

		For the upper bound, no bounded-logit condition is needed. Assumption~\ref{app:ass:glm_regularity_full} gives global $L_{\ell}$-smoothness on $\mathcal{S}_C$, so for arbitrary finite centered logits the standard smoothness inequality at $\bm{u}^*$ yields
		\begin{equation*}
			\rho_{\bm{Z}_c,c}(\bm{u})
			\le
			\rho_{\bm{Z}_c,c}(\bm{u}^*)
			+
			\left\langle
			\nabla_{\mathcal{S}_C}\rho_{\bm{Z}_c,c}(\bm{u}^*),
			\bm{u}-\bm{u}^*
			\right\rangle
			+
			\frac{L_{\ell}}{2}\|\bm{u}-\bm{u}^*\|_2^2.
		\end{equation*}
		The first-order term again vanishes because $\bm{u}^*$ is the conditional Bayes optimum on $\mathcal{S}_C$. Hence
		\begin{equation*}
			\rho_{\bm{Z}_c,c}(f(\bm{Z}))
			-
			\rho_{\bm{Z}_c,c}(f^*(\bm{Z}_c,c))
			\le
			\frac{L_{\ell}}{2}
			\left\|
			\bm{\Pi}_C\!\left(f(\bm{Z})-f^*(\bm{Z}_c,c)\right)
			\right\|_2^2.
		\end{equation*}
		Taking expectation over $(\bm{Z}_c,\bm{Z}_s,c)$ completes the proof.
	\end{proof}

	\subsection{Proof of Lemma~\ref{lem:pooled_projection}}

	\begin{proof}
		For $\bm{W}_c$ and $\bm{W}_s$, define the centered approximation objective
		\begin{equation*}
			J(\bm{W}_c,\bm{W}_s)
			=
			\mathbb{E}
			\left[
			\left\|
			\bm{\Pi}_C
			\left(
			\bm{W}_c\bm{Z}_c+\bm{W}_s\bm{Z}_s
			-
			(\bm{\beta}_{\mathrm{base}}+c\bm{\beta}_{\mathrm{shift}})\bm{Z}_c
			\right)
			\right\|_2^2
			\right].
		\end{equation*}
		Set
		\begin{equation*}
			\bm{U}
			=
			\bm{\Pi}_C(\bm{W}_c-\bm{\beta}_{\mathrm{base}}),
			\qquad
			\bm{V}
			=
			\bm{\Pi}_C\bm{W}_s,
			\qquad
			\bm{B}
			=
			\bm{\Pi}_C\bm{\beta}_{\mathrm{shift}}.
		\end{equation*}
		Using the linearity of $\bm{\Pi}_C$, the objective becomes
		\begin{equation*}
			J(\bm{W}_c,\bm{W}_s)
			=
			\mathbb{E}
			\left[
			\left\|
			(\bm{U}-c\bm{B})\bm{Z}_c+\bm{V}\bm{Z}_s
			\right\|_2^2
			\right].
		\end{equation*}
		Expanding the squared norm gives
		\begin{equation*}
			\begin{aligned}
			J(\bm{W}_c,\bm{W}_s)
			&=
			\mathbb{E}
			\left[
			\left\|
			(\bm{U}-c\bm{B})\bm{Z}_c
			\right\|_2^2
			\right]
			+
			\mathbb{E}
			\left[
			\left\|
			\bm{V}\bm{Z}_s
			\right\|_2^2
			\right] \\
			&\quad
			+
			2\,
			\mathbb{E}
			\left[
			\left((\bm{U}-c\bm{B})\bm{Z}_c\right)^\top
			\bm{V}\bm{Z}_s
			\right].
			\end{aligned}
		\end{equation*}
		The cross term is zero. Indeed, conditioning on $(c,\bm{Z}_s)$ and using Assumption~\ref{app:ass:comparison_shift_full},
		\begin{equation*}
			\begin{aligned}
			&\mathbb{E}
			\left[
			\left((\bm{U}-c\bm{B})\bm{Z}_c\right)^\top
			\bm{V}\bm{Z}_s
			\mid c,\bm{Z}_s
			\right] \\
			&\quad =
			\mathbb{E}
			\left[
			\bm{Z}_c^\top
			(\bm{U}-c\bm{B})^\top
			\bm{V}\bm{Z}_s
			\mid c,\bm{Z}_s
			\right] \\
			&\quad =
			\mathbb{E}[\bm{Z}_c]^\top
			(\bm{U}-c\bm{B})^\top
			\bm{V}\bm{Z}_s
			=
			0.
			\end{aligned}
		\end{equation*}
		The objective decomposes as
		\begin{equation*}
			J(\bm{W}_c,\bm{W}_s)
			=
			J_1(\bm{U})+J_2(\bm{V}),
		\end{equation*}
		where
		\begin{equation*}
			J_1(\bm{U})
			=
			\mathbb{E}_{c,\bm{Z}_c}
			\left[
			\left\|
			(\bm{U}-c\bm{B})\bm{Z}_c
			\right\|_2^2
			\right],
			\qquad
			J_2(\bm{V})
			=
			\mathbb{E}_{\bm{Z}_s}
			\left[
			\|\bm{V}\bm{Z}_s\|_2^2
			\right].
		\end{equation*}
		Since $J_2(\bm{V})\ge0$ and $J_2(\bm{0})=0$, the style-additive part cannot reduce the multiplicative approximation error. Hence it remains to minimize $J_1(\bm{U})$.

		For any fixed matrix $\bm{M}\in\mathbb{R}^{C\times d_c}$, the second moment identity for $\bm{Z}_c$ gives
		\begin{equation*}
			\begin{aligned}
			\mathbb{E}_{\bm{Z}_c}
			\left[
			\|\bm{M}\bm{Z}_c\|_2^2
			\right]
			&=
			\mathbb{E}_{\bm{Z}_c}
			\left[
			\mathrm{Tr}\!\left(
			\bm{M}\bm{Z}_c\bm{Z}_c^\top\bm{M}^\top
			\right)
			\right] \\
			&=
			\mathrm{Tr}\!\left(
			\bm{M}\,
			\mathbb{E}[\bm{Z}_c\bm{Z}_c^\top]\,
			\bm{M}^\top
			\right) \\
			&=
			\mathrm{Tr}\!\left(
			\bm{M}\bm{\Sigma}_c\bm{M}^\top
			\right)
			=
			\|\bm{M}\bm{\Sigma}_c^{1/2}\|_F^2.
			\end{aligned}
		\end{equation*}
		Applying this identity with $\bm{M}=\bm{U}-c\bm{B}$ yields
		\begin{equation*}
			J_1(\bm{U})
			=
			\mathbb{E}_{c}
			\left[
			\left\|
			(\bm{U}-c\bm{B})\bm{\Sigma}_c^{1/2}
			\right\|_F^2
			\right].
		\end{equation*}
		Let $\widetilde{\bm{U}}=\bm{U}\bm{\Sigma}_c^{1/2}$ and $\widetilde{\bm{B}}=\bm{B}\bm{\Sigma}_c^{1/2}$. Then
		\begin{equation*}
			J_1(\bm{U})
			=
			\mathbb{E}_{c}
			\left[
			\|\widetilde{\bm{U}}-c\widetilde{\bm{B}}\|_F^2
			\right].
		\end{equation*}
		Writing $\bar c=\mathbb{E}[c]$ and adding/subtracting $\bar c\,\widetilde{\bm{B}}$ gives the bias-variance decomposition
		\begin{equation*}
			\begin{aligned}
			\mathbb{E}_{c}
			\left[
			\|\widetilde{\bm{U}}-c\widetilde{\bm{B}}\|_F^2
			\right]
			&=
			\mathbb{E}_{c}
			\left[
			\|(\widetilde{\bm{U}}-\bar c\,\widetilde{\bm{B}})
			-
			(c-\bar c)\widetilde{\bm{B}}\|_F^2
			\right] \\
			&=
			\|\widetilde{\bm{U}}-\bar c\,\widetilde{\bm{B}}\|_F^2
			+
			\mathbb{E}\!\left[(c-\bar c)^2\right]
			\|\widetilde{\bm{B}}\|_F^2 \\
			&\quad
			-
			2\,
			\mathbb{E}[c-\bar c]\,
			\left\langle
			\widetilde{\bm{U}}-\bar c\,\widetilde{\bm{B}},
			\widetilde{\bm{B}}
			\right\rangle_F \\
			&=
			\|\widetilde{\bm{U}}-\bar c\,\widetilde{\bm{B}}\|_F^2
			+
			\mathrm{Var}(c)\|\widetilde{\bm{B}}\|_F^2.
			\end{aligned}
		\end{equation*}
		The first term is minimized at $\widetilde{\bm{U}}=\bar c\,\widetilde{\bm{B}}$. Because $\bm{\Sigma}_c\succ\bm{0}$, this choice is attainable by $\bm{U}=\bar c\,\bm{B}$, equivalently by any $\bm{W}_c$ satisfying
		\begin{equation*}
			\bm{\Pi}_C\bm{W}_c
			=
			\bm{\Pi}_C\bm{\beta}_{\mathrm{base}}
			+
			\bar c\,\bm{\Pi}_C\bm{\beta}_{\mathrm{shift}}.
		\end{equation*}
		Combining the minimized value of $J_1$ with the minimized value of $J_2$ gives
		\begin{equation*}
			\min_{\bm{W}_c,\bm{W}_s}J(\bm{W}_c,\bm{W}_s)
			=
			\mathrm{Var}(c)
			\left\|
			\bm{\Pi}_C\bm{\beta}_{\mathrm{shift}}\bm{\Sigma}_c^{1/2}
			\right\|_F^2.
		\end{equation*}
		By Assumption~\ref{app:ass:comparison_shift_full}, $R_W$ is chosen large enough that one minimizer described above belongs to $\mathcal{H}_{\mathrm{pool}}(R_W)$. Therefore, restricting the infimum to $\mathcal{H}_{\mathrm{pool}}(R_W)$ does not change the optimal value.
	\end{proof}

	\subsection{Proof of Corollary~\ref{cor:pooled_misspec}}

	\begin{proof}
		For any $f_{\mathrm{pool}}\in\mathcal{H}_{\mathrm{pool}}(R_W)$, the lower-bound half of Lemma~\ref{lem:centered_logit_excess} implies
		\begin{equation*}
			\mathbb{E}_{c\sim\pi_c}
			\left[
			\mathcal{R}_c(f_{\mathrm{pool}})
			-
			\mathcal{R}_c(f^*)
			\right]
			\ge
			\frac{\mu_{\ell}}{2}\,
			\mathbb{E}
			\left[
			\left\|
			\bm{\Pi}_C
			\left(
			f_{\mathrm{pool}}(\bm{Z}_c,\bm{Z}_s)
			-
			f^*(\bm{Z}_c,c)
			\right)
			\right\|_2^2
			\right].
		\end{equation*}
		Taking the infimum over $f_{\mathrm{pool}}\in\mathcal{H}_{\mathrm{pool}}(R_W)$ on both sides and substituting the exact additive approximation error from Lemma~\ref{lem:pooled_projection} yields
		\begin{equation*}
			\inf_{f_{\mathrm{pool}}\in\mathcal{H}_{\mathrm{pool}}(R_W)}
			\mathbb{E}_{c\sim\pi_c}
			\left[
			\mathcal{R}_c(f_{\mathrm{pool}})
			-
			\mathcal{R}_c(f^*)
			\right]
			\ge
			\frac{\mu_{\ell}}{2}\,
			\mathrm{Var}(c)
			\left\|
			\bm{\Pi}_C\bm{\beta}_{\mathrm{shift}}\bm{\Sigma}_c^{1/2}
			\right\|_F^2.
		\end{equation*}
		The right-hand side is positive whenever $\mathrm{Var}(c)>0$ and $\|\bm{\Pi}_C\bm{\beta}_{\mathrm{shift}}\bm{\Sigma}_c^{1/2}\|_F>0$. The first condition states that the domain population contains distinct domains, and the second states that the domain shift changes the identifiable centered logits. Under these non-degeneracy conditions, the pooled-joint additive class has a positive structural misspecification error.
	\end{proof}

	\subsection{Proof of Corollary~\ref{cor:ladder_irm_comparison}}

	\begin{proof}
		Fix the target domain coordinate $c_0$. Under Assumption~\ref{app:ass:comparison_shift_full}, the target-domain Bayes-optimal parameter is
		\begin{equation*}
			\bm{\beta}_*^{(0)}
			=
			\bm{\beta}_{\mathrm{base}}+c_0\bm{\beta}_{\mathrm{shift}}.
		\end{equation*}
		By Assumption~\ref{app:ass:comparison_shift_full}, write $\bm\beta_{\mathrm L}=\bm\beta_{\mathrm{base}}+\widehat c_{\mathrm L}\bm\beta_{\mathrm{shift}}$ and $\bm\beta_{\mathrm{inv}}=\bm\beta_{\mathrm{base}}+c_{\mathrm{inv}}\bm\beta_{\mathrm{shift}}$. Their identifiable centered differences are
		\begin{equation*}
			\begin{aligned}
			\bm{\Pi}_C(\bm{\beta}_{\mathrm{L}}-\bm{\beta}_*^{(0)})
			&=
			(\widehat c_{\mathrm{L}}-c_0)\bm{\Pi}_C\bm{\beta}_{\mathrm{shift}},\\
			\bm{\Pi}_C(\bm{\beta}_{\mathrm{inv}}-\bm{\beta}_*^{(0)})
			&=
			(c_{\mathrm{inv}}-c_0)\bm{\Pi}_C\bm{\beta}_{\mathrm{shift}}.
			\end{aligned}
		\end{equation*}
		Let
		\begin{equation*}
			\bm{B}
			=
			\bm{\Pi}_C\bm{\beta}_{\mathrm{shift}}\bm{\Sigma}_c^{1/2}.
		\end{equation*}
		Using the second-moment identity from the proof of Lemma~\ref{lem:pooled_projection}, for any scalar $a$,
		\begin{equation*}
			\mathbb{E}_{\bm{Z}_c}
			\left[
			\|a\,\bm{\Pi}_C\bm{\beta}_{\mathrm{shift}}\bm{Z}_c\|_2^2
			\right]
			=
			a^2\|\bm{B}\|_F^2.
		\end{equation*}
		Applying the global upper-smoothness half of Lemma~\ref{lem:centered_logit_excess} to $f_{\mathrm{L}}$ on the fixed target domain $c_0$ gives
		\begin{equation*}
			\mathcal{R}_{c_0}(f_{\mathrm{L}})
			-
			\mathcal{R}_{c_0}(f^*)
			\le
			\frac{L_{\ell}}{2}
			(c_0-\widehat c_{\mathrm{L}})^2
			\|\bm{B}\|_F^2.
		\end{equation*}
		Applying the lower half of the same lemma to $f_{\mathrm{inv}}$ gives
		\begin{equation*}
			\mathcal{R}_{c_0}(f_{\mathrm{inv}})
			-
			\mathcal{R}_{c_0}(f^*)
			\ge
			\frac{\mu_{\ell}}{2}
			(c_0-c_{\mathrm{inv}})^2
			\|\bm{B}\|_F^2.
		\end{equation*}
		By the non-degeneracy assumption in the corollary, $\|\bm{B}\|_F^2>0$. Since $c_{\mathrm{inv}}\neq c_0$ by Assumption~\ref{app:ass:comparison_shift_full}, the denominator is strictly positive. Dividing the two inequalities yields
		\begin{equation*}
			\frac{
			\mathcal{R}_{c_0}(f_{\mathrm{L}})-\mathcal{R}_{c_0}(f^*)
			}{
			\mathcal{R}_{c_0}(f_{\mathrm{inv}})-\mathcal{R}_{c_0}(f^*)
			}
			\le
			\frac{L_{\ell}}{\mu_{\ell}}\,
			\frac{(c_0-\widehat c_{\mathrm{L}})^2}{(c_0-c_{\mathrm{inv}})^2}.
		\end{equation*}
		The strict improvement condition follows immediately: if
		\begin{equation*}
			|c_0-\widehat c_{\mathrm{L}}|
			<
			\sqrt{\frac{\mu_{\ell}}{L_{\ell}}}\,
			|c_0-c_{\mathrm{inv}}|,
		\end{equation*}
		then the right-hand side is strictly smaller than $1$, and hence
		\begin{equation*}
			\mathcal{R}_{c_0}(f_{\mathrm{L}})-\mathcal{R}_{c_0}(f^*)
			<
			\mathcal{R}_{c_0}(f_{\mathrm{inv}})-\mathcal{R}_{c_0}(f^*).
		\end{equation*}
	\end{proof}
	
	\subsection{Proof of Corollary~\ref{cor:generic_representation}}
\begin{proof}
The source estimators and empirical style measures are
\[
\begin{aligned}
\widehat{\bm\beta}^{(e)}
&\in\argmin_{\bm\beta\in\mathcal B}\widehat{\mathcal R}_e(\bm\beta),
&
\widetilde{\bm\beta}^{(e)}
&\in\argmin_{\bm\beta\in\mathcal B}
\widetilde{\mathcal R}^{\mathrm{emp}}_e(\bm\beta),
&&e=1,\ldots,E,\\
\widehat\nu_e
&=\frac1{n_e}\sum_{i=1}^{n_e}\delta_{\bm z_{s,i}^{(e)}},
&
\widetilde\nu_e
&=\frac1{n_e}\sum_{i=1}^{n_e}\delta_{\widetilde{\bm z}_{s,i}^{(e)}},
&&e=0,\ldots,E.
\end{aligned}
\]
The weights satisfy
\[
\widetilde w_e\ge0,\qquad
\sum_{e\in\widetilde{\mathcal N}_K(0)}\widetilde w_e=1,\qquad
\widetilde{\bm\beta}^{(0)}
=\sum_{e\in\widetilde{\mathcal N}_K(0)}
\widetilde w_e\widetilde{\bm\beta}^{(e)}\in\mathcal B.
\]
Constants denoted by $C_E$ may depend on the fixed number of sources $E$.

\textbf{Empirical-measure perturbations.}
Pairing each learned style point with its oracle counterpart gives
\[
\begin{aligned}
\mathcal W_1(\widetilde\nu_e,\widehat\nu_e)
&\le\frac1{n_e}\sum_{i=1}^{n_e}
\|\widetilde{\bm z}_{s,i}^{(e)}-\bm z_{s,i}^{(e)}\|_2\\
&=\frac1{n_e}\sum_{i=1}^{n_e}\|\bm\xi_{s,i}^{(e)}\|_2,
\qquad e=0,\ldots,E.
\end{aligned}
\]
For $q\in\{c,s\}$, Jensen's inequality and
Assumption~\ref{app:ass:generic_representation_error_full} give
\[
\begin{aligned}
\mathbb E\left(\frac1{n_e}\sum_{i=1}^{n_e}
\|\bm\xi_{q,i}^{(e)}\|_2\right)^2
&\le\frac1{n_e}\sum_{i=1}^{n_e}
\mathbb E\|\bm\xi_{q,i}^{(e)}\|_2^2\\
&\le\frac1{n_e}\sum_{i=1}^{n_e}\epsilon_{\mathrm{rep}}^2
=\epsilon_{\mathrm{rep}}^2,\\
\mathbb E\left(\frac1{n_e}\sum_{i=1}^{n_e}
\|\bm\xi_{q,i}^{(e)}\|_2\right)
&\le
\left[\mathbb E\left(\frac1{n_e}\sum_{i=1}^{n_e}
\|\bm\xi_{q,i}^{(e)}\|_2\right)^2\right]^{1/2}
\le\epsilon_{\mathrm{rep}}.
\end{aligned}
\]
For the source maximum,
\[
\begin{aligned}
\mathbb E\!\left\{\left[
\max_{1\le e\le E}\frac1{n_e}\sum_i
\|\bm\xi_{q,i}^{(e)}\|_2\right]^2\right\}
&\le\sum_{e=1}^E
\mathbb E\left(\frac1{n_e}\sum_i
\|\bm\xi_{q,i}^{(e)}\|_2\right)^2
\le E\epsilon_{\mathrm{rep}}^2,\\
\mathbb E\max_{1\le e\le E}\frac1{n_e}\sum_i
\|\bm\xi_{q,i}^{(e)}\|_2
&\le
\left\{\mathbb E\!\left\{\left[
\max_{1\le e\le E}\frac1{n_e}\sum_i
\|\bm\xi_{q,i}^{(e)}\|_2\right]^2\right\}\right\}^{1/2}
\le\sqrt E\,\epsilon_{\mathrm{rep}}.
\end{aligned}
\]
In particular, the target and source style perturbations satisfy
\[
\begin{aligned}
\mathbb E\mathcal W_1(\widetilde\nu_0,\widehat\nu_0)
&\le\epsilon_{\mathrm{rep}},
&
\mathbb E\mathcal W_1(\widetilde\nu_0,\widehat\nu_0)^2
&\le\epsilon_{\mathrm{rep}}^2,\\
\mathbb E\max_{1\le e\le E}
\mathcal W_1(\widetilde\nu_e,\widehat\nu_e)
&\le\sqrt E\,\epsilon_{\mathrm{rep}},
&
\mathbb E\max_{1\le e\le E}
\mathcal W_1(\widetilde\nu_e,\widehat\nu_e)^2
&\le E\epsilon_{\mathrm{rep}}^2.
\end{aligned}
\]

\textbf{Source-classifier perturbations.}
Write $\bm e_y$ for the one-hot label vector and set
\[
\bm p=\mathrm{Softmax}(\bm\beta\bm z),\qquad
\bm p'=\mathrm{Softmax}(\bm\beta(\bm z+\bm\xi)).
\]
For any probability vector $\bm p$, the symmetric Softmax Jacobian obeys
\[
\begin{aligned}
\bm H(\bm p)&=\operatorname{diag}(\bm p)-\bm p\bm p^\top,\\
\|\bm H(\bm p)\|_{\mathrm{op}}
&\le\max_j\sum_k|H_{jk}(\bm p)|
=\max_j 2p_j(1-p_j)\le\tfrac12.
\end{aligned}
\]
The fundamental theorem of calculus gives
\[
\begin{aligned}
\bm p'-\bm p
&=\int_0^1
\bm H\!\left(\mathrm{Softmax}(\bm\beta(\bm z+t\bm\xi))\right)
\bm\beta\bm\xi\,dt,\\
\|\bm p'-\bm p\|_2
&\le\int_0^1\tfrac12\|\bm\beta\bm\xi\|_2\,dt
\le\tfrac12R_\beta\|\bm\xi\|_2,\\
\|\bm p'-\bm e_y\|_2^2
&=\|\bm p'\|_2^2+1-2p'_y\le2.
\end{aligned}
\]
The single-sample gradient difference is
\[
\begin{aligned}
&\nabla_{\bm\beta}\ell_{\mathrm{CE}}(\bm\beta(\bm z+\bm\xi),y)
-\nabla_{\bm\beta}\ell_{\mathrm{CE}}(\bm\beta\bm z,y)\\
&\quad=(\bm p'-\bm e_y)(\bm z+\bm\xi)^\top
-(\bm p-\bm e_y)\bm z^\top\\
&\quad=(\bm p'-\bm e_y)\bm\xi^\top
+(\bm p'-\bm p)\bm z^\top.
\end{aligned}
\]
For $\bm\beta\in\mathcal B$ and $\|\bm z\|_2\le R_c$,
\[
\begin{aligned}
&\|\nabla_{\bm\beta}\ell_{\mathrm{CE}}(\bm\beta(\bm z+\bm\xi),y)
-\nabla_{\bm\beta}\ell_{\mathrm{CE}}(\bm\beta\bm z,y)\|_F\\
&\quad\le
\|\bm p'-\bm e_y\|_2\|\bm\xi\|_2
+\|\bm p'-\bm p\|_2\|\bm z\|_2\\
&\quad\le
\sqrt2\|\bm\xi\|_2+\tfrac12R_\beta R_c\|\bm\xi\|_2
=L_{\nabla z}\|\bm\xi\|_2.
\end{aligned}
\]
Averaging over source $e$ yields
\[
\begin{aligned}
&\|\nabla\widetilde{\mathcal R}^{\mathrm{emp}}_e(\bm\beta)
-\nabla\widehat{\mathcal R}_e(\bm\beta)\|_F\\
&\quad\le\frac1{n_e}\sum_{i=1}^{n_e}
\|\nabla_{\bm\beta}\ell_{\mathrm{CE}}
(\bm\beta\widetilde{\bm z}_{c,i}^{(e)},y_i^{(e)})
-\nabla_{\bm\beta}\ell_{\mathrm{CE}}
(\bm\beta\bm z_{c,i}^{(e)},y_i^{(e)})\|_F\\
&\quad\le\frac{L_{\nabla z}}{n_e}
\sum_{i=1}^{n_e}\|\bm\xi_{c,i}^{(e)}\|_2.
\end{aligned}
\]
Set $\Delta_e=\widetilde{\bm\beta}^{(e)}-\widehat{\bm\beta}^{(e)}$.
The constrained first-order conditions are
\[
\left\langle
\nabla\widehat{\mathcal R}_e(\widehat{\bm\beta}^{(e)}),\Delta_e
\right\rangle_F\ge0,\qquad
\left\langle
\nabla\widetilde{\mathcal R}^{\mathrm{emp}}_e
(\widetilde{\bm\beta}^{(e)}),\Delta_e
\right\rangle_F\le0.
\]
Strong monotonicity on $\mathcal B$ and these conditions imply
\[
\begin{aligned}
\mu_{\mathrm{rep}}\|\Delta_e\|_F^2
&\le\left\langle
\nabla\widetilde{\mathcal R}^{\mathrm{emp}}_e
(\widetilde{\bm\beta}^{(e)})
-\nabla\widetilde{\mathcal R}^{\mathrm{emp}}_e
(\widehat{\bm\beta}^{(e)}),\Delta_e
\right\rangle_F\\
&=\left\langle
\nabla\widetilde{\mathcal R}^{\mathrm{emp}}_e
(\widetilde{\bm\beta}^{(e)}),\Delta_e
\right\rangle_F
-\left\langle
\nabla\widehat{\mathcal R}_e(\widehat{\bm\beta}^{(e)}),\Delta_e
\right\rangle_F\\
&\quad+\left\langle
\nabla\widehat{\mathcal R}_e(\widehat{\bm\beta}^{(e)})
-\nabla\widetilde{\mathcal R}^{\mathrm{emp}}_e
(\widehat{\bm\beta}^{(e)}),\Delta_e
\right\rangle_F\\
&\le\left\langle
\nabla\widehat{\mathcal R}_e(\widehat{\bm\beta}^{(e)})
-\nabla\widetilde{\mathcal R}^{\mathrm{emp}}_e
(\widehat{\bm\beta}^{(e)}),\Delta_e
\right\rangle_F\\
&\le
\|\nabla\widehat{\mathcal R}_e(\widehat{\bm\beta}^{(e)})
-\nabla\widetilde{\mathcal R}^{\mathrm{emp}}_e
(\widehat{\bm\beta}^{(e)})\|_F\,\|\Delta_e\|_F\\
&\le\frac{L_{\nabla z}}{n_e}
\sum_i\|\bm\xi_{c,i}^{(e)}\|_2\,\|\Delta_e\|_F.
\end{aligned}
\]
For $\Delta_e\ne0$, division by
$\mu_{\mathrm{rep}}\|\Delta_e\|_F$ gives the following bound;
it also holds for $\Delta_e=0$:
\[
\|\widetilde{\bm\beta}^{(e)}-\widehat{\bm\beta}^{(e)}\|_F
\le\frac{L_{\nabla z}}{\mu_{\mathrm{rep}}}
\frac1{n_e}\sum_i\|\bm\xi_{c,i}^{(e)}\|_2.
\]
The source-maximum moments are
\[
\begin{aligned}
\mathbb E\max_e
\|\widetilde{\bm\beta}^{(e)}-\widehat{\bm\beta}^{(e)}\|_F
&\le\frac{L_{\nabla z}}{\mu_{\mathrm{rep}}}
\mathbb E\max_e\frac1{n_e}\sum_i\|\bm\xi_{c,i}^{(e)}\|_2
\le\frac{\sqrt E L_{\nabla z}}{\mu_{\mathrm{rep}}}
\epsilon_{\mathrm{rep}},\\
\mathbb E\max_e
\|\widetilde{\bm\beta}^{(e)}-\widehat{\bm\beta}^{(e)}\|_F^2
&\le\frac{L_{\nabla z}^2}{\mu_{\mathrm{rep}}^2}
\mathbb E\!\left\{\left[\max_e\frac1{n_e}\sum_i
\|\bm\xi_{c,i}^{(e)}\|_2\right]^2\right\}\\
&\le\frac{E L_{\nabla z}^2}{\mu_{\mathrm{rep}}^2}
\epsilon_{\mathrm{rep}}^2.
\end{aligned}
\]

\textbf{Perturbed source selection.}
Let $\mathcal N_K^*(0)$ contain the $K$ nearest source laws to $\nu_0$ and set
\[
R_K^*=\max_{e\in\mathcal N_K^*(0)}\mathcal W_1(\nu_e,\nu_0).
\]
The empirical $K$-nearest-neighbor property and the triangle inequality give
\[
\begin{aligned}
\max_{e\in\widetilde{\mathcal N}_K(0)}
\mathcal W_1(\widetilde\nu_e,\widetilde\nu_0)
&\le\max_{j\in\mathcal N_K^*(0)}
\mathcal W_1(\widetilde\nu_j,\widetilde\nu_0)\\
&\le\max_{j\in\mathcal N_K^*(0)}
\{\mathcal W_1(\widetilde\nu_j,\nu_j)
+\mathcal W_1(\nu_j,\nu_0)
+\mathcal W_1(\nu_0,\widetilde\nu_0)\}\\
&\le R_K^*+\max_{1\le e\le E}
\mathcal W_1(\widetilde\nu_e,\nu_e)
+\mathcal W_1(\widetilde\nu_0,\nu_0).
\end{aligned}
\]
For each selected source,
\[
\begin{aligned}
\mathcal W_1(\nu_e,\nu_0)
&\le\mathcal W_1(\nu_e,\widetilde\nu_e)
+\mathcal W_1(\widetilde\nu_e,\widetilde\nu_0)
+\mathcal W_1(\widetilde\nu_0,\nu_0)\\
&\le R_K^*+2\max_{1\le j\le E}
\mathcal W_1(\widetilde\nu_j,\nu_j)
+2\mathcal W_1(\widetilde\nu_0,\nu_0).
\end{aligned}
\]
Splitting through the oracle empirical measures yields
\[
\begin{aligned}
\max_{1\le e\le E}\mathcal W_1(\widetilde\nu_e,\nu_e)
&\le\max_{1\le e\le E}\mathcal W_1(\widetilde\nu_e,\widehat\nu_e)
+\max_{1\le e\le E}\mathcal W_1(\widehat\nu_e,\nu_e),\\
\mathcal W_1(\widetilde\nu_0,\nu_0)
&\le\mathcal W_1(\widetilde\nu_0,\widehat\nu_0)
+\mathcal W_1(\widehat\nu_0,\nu_0).
\end{aligned}
\]
For the target parameter,
\[
\begin{aligned}
\widetilde{\bm\beta}^{(0)}-\bm\beta_*^{(0)}
&=\sum_{e\in\widetilde{\mathcal N}_K(0)}\widetilde w_e
(\widetilde{\bm\beta}^{(e)}-\bm\beta_*^{(0)})\\
&=\sum_{e\in\widetilde{\mathcal N}_K(0)}\widetilde w_e
\big[
(\widehat{\bm\beta}^{(e)}-\bm\beta_*^{(e)})
+(\widetilde{\bm\beta}^{(e)}-\widehat{\bm\beta}^{(e)})
+(\bm\beta_*^{(e)}-\bm\beta_*^{(0)})
\big].
\end{aligned}
\]
Convex weights and the Lipschitz rule map give
\[
\begin{aligned}
\|\widetilde{\bm\beta}^{(0)}-\bm\beta_*^{(0)}\|_F
&\le\sum_{e\in\widetilde{\mathcal N}_K(0)}\widetilde w_e
\big[
\|\widehat{\bm\beta}^{(e)}-\bm\beta_*^{(e)}\|_F
+\|\widetilde{\bm\beta}^{(e)}-\widehat{\bm\beta}^{(e)}\|_F\\
&\hspace{52mm}
+L_\beta\mathcal W_1(\nu_e,\nu_0)
\big]\\
&\le\max_e\|\widehat{\bm\beta}^{(e)}-\bm\beta_*^{(e)}\|_F
+\max_e\|\widetilde{\bm\beta}^{(e)}-\widehat{\bm\beta}^{(e)}\|_F\\
&\quad+L_\beta\big[
R_K^*+2\mathcal W_1(\widehat\nu_0,\nu_0)
+2\max_e\mathcal W_1(\widehat\nu_e,\nu_e)\\
&\hspace{30mm}
+2\mathcal W_1(\widetilde\nu_0,\widehat\nu_0)
+2\max_e\mathcal W_1(\widetilde\nu_e,\widehat\nu_e)
\big].
\end{aligned}
\tag{*}
\]
Here and below, source maxima range over $e=1,\ldots,E$.

\textbf{Parameter moments.}
Lemmas~\ref{lem:estimator_tail} and~\ref{lem:wasserstein_concentration} imply
\[
\begin{aligned}
\mathbb E\max_e
\|\widehat{\bm\beta}^{(e)}-\bm\beta_*^{(e)}\|_F
&\le C_\star\frac{R_c}{\mu}
\sqrt{\frac{d_\beta+\log E}{n_{\min}}},\\
\mathbb E\max_e
\|\widehat{\bm\beta}^{(e)}-\bm\beta_*^{(e)}\|_F^2
&\le C_\star\frac{R_c^2(d_\beta+\log E)}{\mu^2n_{\min}},\\
\mathbb E\mathcal W_1(\widehat\nu_0,\nu_0)
&\le C_\star\gamma(n_0,d_s),\\
\mathbb E\mathcal W_1(\widehat\nu_0,\nu_0)^2
&\le C_\star\gamma(n_0,d_s)^2,\\
\mathbb E\max_e\mathcal W_1(\widehat\nu_e,\nu_e)
&\le C_\star\{\gamma(n_{\min},d_s)+\Gamma(n_{\min},E,d_s)\},\\
\mathbb E\max_e\mathcal W_1(\widehat\nu_e,\nu_e)^2
&\le C_\star\{\gamma(n_{\min},d_s)^2+\Gamma(n_{\min},E,d_s)^2\}.
\end{aligned}
\]
The nearest-neighbor calculation in Lemma~\ref{lem:knn_interpolation}
and the proof of Theorem~\ref{thm:excess_risk} gives
\[
\mathbb E R_K^*\le C_\star(K/E)^{1/d_{\mathcal M}},\qquad
\mathbb E(R_K^*)^2\le C_\star(K/E)^{2/d_{\mathcal M}}.
\]
Taking expectations in $(*)$ and using the displayed moment bounds gives
\[
\begin{aligned}
\mathbb E\|\widetilde{\bm\beta}^{(0)}-\bm\beta_*^{(0)}\|_F
&\le C_\star\frac{R_c}{\mu}
\sqrt{\frac{d_\beta+\log E}{n_{\min}}}
+\frac{\sqrt E L_{\nabla z}}{\mu_{\mathrm{rep}}}\epsilon_{\mathrm{rep}}\\
&\quad+C_\star L_\beta\big[
(K/E)^{1/d_{\mathcal M}}
+2\gamma(n_0,d_s)\\
&\hspace{37mm}
+2\gamma(n_{\min},d_s)+2\Gamma(n_{\min},E,d_s)\big]\\
&\quad+2L_\beta(1+\sqrt E)\epsilon_{\mathrm{rep}}\\
&\le\mathcal O\!\Bigg(
\frac{R_c}{\mu}\sqrt{\frac{d_\beta+\log E}{n_{\min}}}
+\gamma(n_0,d_s)+\Gamma(n_{\min},E,d_s)\\
&\hspace{19mm}
+(K/E)^{1/d_{\mathcal M}}
+\left(\frac{L_{\nabla z}}{\mu_{\mathrm{rep}}}+L_\beta\right)
\epsilon_{\mathrm{rep}}\Bigg).
\end{aligned}
\]
The last inequality uses
$\gamma(n_{\min},d_s)\le\Gamma(n_{\min},E,d_s)$ and fixed $E$.
Applying $(\sum_{j=1}^7 a_j)^2\le7\sum_{j=1}^7 a_j^2$
to the seven nonnegative terms in $(*)$ gives
\[
\begin{aligned}
\mathbb E\|\widetilde{\bm\beta}^{(0)}-\bm\beta_*^{(0)}\|_F^2
&\le7\mathbb E\max_e
\|\widehat{\bm\beta}^{(e)}-\bm\beta_*^{(e)}\|_F^2
+7\mathbb E\max_e
\|\widetilde{\bm\beta}^{(e)}-\widehat{\bm\beta}^{(e)}\|_F^2\\
&\quad+7L_\beta^2\mathbb E(R_K^*)^2\\
&\quad+28L_\beta^2\big[
\mathbb E\mathcal W_1(\widehat\nu_0,\nu_0)^2
+\mathbb E\max_e\mathcal W_1(\widehat\nu_e,\nu_e)^2\big]\\
&\quad+28L_\beta^2\big[
\mathbb E\mathcal W_1(\widetilde\nu_0,\widehat\nu_0)^2
+\mathbb E\max_e\mathcal W_1(\widetilde\nu_e,\widehat\nu_e)^2\big]\\
&\le7C_\star\frac{R_c^2(d_\beta+\log E)}{\mu^2n_{\min}}
+\frac{7E L_{\nabla z}^2}{\mu_{\mathrm{rep}}^2}
\epsilon_{\mathrm{rep}}^2\\
&\quad+7C_\star L_\beta^2(K/E)^{2/d_{\mathcal M}}\\
&\quad+28C_\star L_\beta^2\big[
\gamma(n_0,d_s)^2+\gamma(n_{\min},d_s)^2
+\Gamma(n_{\min},E,d_s)^2\big]\\
&\quad+28L_\beta^2(1+E)\epsilon_{\mathrm{rep}}^2\\
&\le C_E\Bigg(
\left[\frac{R_c}{\mu}\sqrt{\frac{d_\beta+\log E}{n_{\min}}}
+\gamma(n_0,d_s)+\Gamma(n_{\min},E,d_s)
+(K/E)^{1/d_{\mathcal M}}\right]^2\\
&\hspace{21mm}
+\left(\frac{L_{\nabla z}}{\mu_{\mathrm{rep}}}+L_\beta\right)^2
\epsilon_{\mathrm{rep}}^2\Bigg).
\end{aligned}
\tag{**}
\]

\textbf{Target excess risk.}
Let $\mathcal F$ be generated by the fitted encoders, source classifiers,
and target-summary sample. For a fresh target observation, define
\[
\begin{aligned}
\bm u_*&=\bm\beta_*^{(0)}\bm Z_c^{(0)},&
\widetilde{\bm u}
&=\widetilde{\bm\beta}^{(0)}(\bm Z_c^{(0)}+\bm\xi_c^{(0)}),\\
\bm p_*&=\mathrm{Softmax}(\bm u_*),&
\Delta_{\bm u}&=\widetilde{\bm u}-\bm u_*.
\end{aligned}
\]
Assumption~\ref{app:ass:generic_representation_error_full} gives
\[
\mathbb E[\bm e_Y\mid
\mathcal F,\bm Z_c^{(0)},\bm Z_s^{(0)},\bm\xi_c^{(0)},\bm\xi_s^{(0)}]
=\bm p_*.
\]
The conditional cross-entropy as a function of logits is
\[
\psi_{\bm p_*}(\bm u)
=\log\!\left(\sum_{j=1}^C e^{u_j}\right)-\bm p_*^\top\bm u.
\]
Its derivatives satisfy
\[
\begin{aligned}
\nabla\psi_{\bm p_*}(\bm u_*)&=\mathrm{Softmax}(\bm u_*)-\bm p_*=\bm0,\\
\nabla^2\psi_{\bm p_*}(\bm u)
&=\bm H(\mathrm{Softmax}(\bm u)),\qquad
\bm0\preceq\nabla^2\psi_{\bm p_*}(\bm u)\preceq\tfrac12\bm I_C.
\end{aligned}
\]
Taylor's formula with integral remainder gives
\[
\begin{aligned}
0\le\psi_{\bm p_*}(\widetilde{\bm u})-\psi_{\bm p_*}(\bm u_*)
&=\nabla\psi_{\bm p_*}(\bm u_*)^\top\Delta_{\bm u}\\
&\quad+\int_0^1(1-t)\,
\Delta_{\bm u}^\top
\nabla^2\psi_{\bm p_*}(\bm u_*+t\Delta_{\bm u})
\Delta_{\bm u}\,dt\\
&\le\tfrac12\|\Delta_{\bm u}\|_2^2\int_0^1(1-t)\,dt
=\tfrac14\|\Delta_{\bm u}\|_2^2.
\end{aligned}
\]
The logit difference obeys
\[
\begin{aligned}
\Delta_{\bm u}
&=(\widetilde{\bm\beta}^{(0)}-\bm\beta_*^{(0)})\bm Z_c^{(0)}
+\widetilde{\bm\beta}^{(0)}\bm\xi_c^{(0)},\\
\|\Delta_{\bm u}\|_2^2
&\le2\|(\widetilde{\bm\beta}^{(0)}-\bm\beta_*^{(0)})
\bm Z_c^{(0)}\|_2^2
+2\|\widetilde{\bm\beta}^{(0)}\bm\xi_c^{(0)}\|_2^2\\
&\le2R_c^2\|\widetilde{\bm\beta}^{(0)}-\bm\beta_*^{(0)}\|_F^2
+2R_\beta^2\|\bm\xi_c^{(0)}\|_2^2.
\end{aligned}
\]
The risks are evaluated with fitted objects fixed:
\[
\begin{aligned}
\widetilde{\mathcal R}_0(\widetilde{\bm\beta}^{(0)})
&=\mathbb E[\ell_{\mathrm{CE}}(\widetilde{\bm u},Y^{(0)})\mid\mathcal F],\\
\mathcal R_0(\bm\beta_*^{(0)})
&=\mathbb E[\ell_{\mathrm{CE}}(\bm u_*,Y^{(0)})\mid\mathcal F].
\end{aligned}
\]
Iterated expectation and the fresh-observation second-moment bound yield
\[
\begin{aligned}
0\le{}&
\mathbb E\{\widetilde{\mathcal R}_0(\widetilde{\bm\beta}^{(0)})
-\mathcal R_0(\bm\beta_*^{(0)})\}\\
={}&\mathbb E\{
\psi_{\bm p_*}(\widetilde{\bm u})-\psi_{\bm p_*}(\bm u_*)\}\\
\le{}&\tfrac14\mathbb E\|\Delta_{\bm u}\|_2^2\\
\le{}&\frac{R_c^2}{2}
\mathbb E\|\widetilde{\bm\beta}^{(0)}-\bm\beta_*^{(0)}\|_F^2
+\frac{R_\beta^2}{2}\mathbb E\|\bm\xi_c^{(0)}\|_2^2\\
\le{}&\frac{R_c^2}{2}
\mathbb E\|\widetilde{\bm\beta}^{(0)}-\bm\beta_*^{(0)}\|_F^2
+\frac{R_\beta^2}{2}\epsilon_{\mathrm{rep}}^2.
\end{aligned}
\]
Substitution of $(**)$ gives the full stated rate:
\[
\begin{aligned}
&\mathbb E\{\widetilde{\mathcal R}_0(\widetilde{\bm\beta}^{(0)})
-\mathcal R_0(\bm\beta_*^{(0)})\}\\
&\quad\le\frac{C_E R_c^2}{2}
\Bigg(
\left[\frac{R_c}{\mu}\sqrt{\frac{d_\beta+\log E}{n_{\min}}}
+\gamma(n_0,d_s)+\Gamma(n_{\min},E,d_s)
+(K/E)^{1/d_{\mathcal M}}\right]^2\\
&\hspace{49mm}
+\left(\frac{L_{\nabla z}}{\mu_{\mathrm{rep}}}+L_\beta\right)^2
\epsilon_{\mathrm{rep}}^2\Bigg)
+\frac{R_\beta^2}{2}\epsilon_{\mathrm{rep}}^2\\
&\quad=\frac{C_E R_c^2}{2}
\left[\frac{R_c}{\mu}\sqrt{\frac{d_\beta+\log E}{n_{\min}}}
+\gamma(n_0,d_s)+\Gamma(n_{\min},E,d_s)
+(K/E)^{1/d_{\mathcal M}}\right]^2\\
&\hspace{16mm}
+\left[\frac{C_E R_c^2}{2}
\left(\frac{L_{\nabla z}}{\mu_{\mathrm{rep}}}+L_\beta\right)^2
+\frac{R_\beta^2}{2}\right]\epsilon_{\mathrm{rep}}^2\\
&\quad\le\mathcal O\!\Bigg(
R_c^2\left[\frac{R_c}{\mu}\sqrt{\frac{d_\beta+\log E}{n_{\min}}}
+\gamma(n_0,d_s)+\Gamma(n_{\min},E,d_s)
+(K/E)^{1/d_{\mathcal M}}\right]^2\\
&\hspace{30mm}
+\left[
R_c^2\left(\frac{L_{\nabla z}}{\mu_{\mathrm{rep}}}+L_\beta\right)^2
+R_\beta^2\right]\epsilon_{\mathrm{rep}}^2
\Bigg).
\end{aligned}
\]
\end{proof}

\section{Additional Experimental Details and Results}\label{app:experimental_details}

\subsection{Dataset Details}

FMoW-WILDS is a $62$-class satellite land-use classification benchmark. The images come from the Functional Map of the World collection, and the WILDS split induces distribution shifts across acquisition years and geographic regions. In our experiments, years define source domains for source reweighting, and regions define the group metric used for robustness evaluation. The Other bucket is excluded from worst-region accuracy following the WILDS evaluator, and from our average-region statistic. Overall accuracy includes all examples.

For an evaluation split $\mathcal S$, test accuracy is the fraction of correctly classified examples. Let $\mathcal R$ be the retained FMoW regions after excluding ``Other'', and let $\operatorname{Acc}_r$ be the accuracy within region $r$. The region-aware metrics and distribution-shift drop are
\[
\operatorname{Worst}=\min_{r\in\mathcal R}\operatorname{Acc}_r,\qquad
\operatorname{Avg.}=\frac{1}{|\mathcal R|}\sum_{r\in\mathcal R}\operatorname{Acc}_r,\qquad
\operatorname{Drop}=\operatorname{Acc}_{\mathrm{ID\ val}}-\operatorname{Acc}_{\mathrm{OOD\ test}}.
\]
Thus, Worst and Avg. are respectively the minimum and unweighted mean of the retained region accuracies, while ID-to-test drop measures the decrease from ID validation accuracy to OOD test accuracy.

iWildCam-WILDS is a $182$-class wildlife species classification benchmark collected by camera traps. Domains are camera locations, and both class frequencies and location frequencies are highly long-tailed. This makes raw worst-location accuracy extremely brittle: a few target locations contain only a handful of examples, so a single error can dominate the raw worst-domain metric. Moreover, naturally tiny target locations are entangled with other effects, including unseen camera locations, severe class imbalance, and incomplete class support within a location. We therefore use filtered average-group accuracy for the main real-data robustness evaluation and reserve target-sample-size sensitivity for synthetic simulations where the target-summary sample size can be varied without changing the label support or the environment mechanism.

For iWildCam-WILDS, test accuracy is computed over the full evaluation split, macro F1 is the unweighted mean of the F1 scores over classes present in the evaluation labels, and filtered average-group accuracy is
\[
\operatorname{FiltAvg}
=\frac{1}{|\mathcal G_{50}|}\sum_{g\in\mathcal G_{50}}\operatorname{Acc}_g,\qquad
\mathcal G_{50}=\{g:n_g\ge50\},
\]
where $g$ indexes target camera locations, $n_g$ is the number of evaluation examples at location $g$, and $\operatorname{Acc}_g$ is its accuracy.

\paragraph{Information-use boundary.}
In the benchmarks, the unlabeled summary covariates are taken from the evaluation split or camera location whose examples are predicted. The theoretical population risk evaluates a fresh observation from the same target domain after the summary has been formed.

The experiments use source-only training; all learned model parameters and source normalization statistics remain fixed during evaluation. For \method{}, this fixes both encoders and all source classifiers. At inference, \method{} encodes unlabeled target covariates $\{\bm{x}_i^{(0)}\}_{i=1}^{n_0}$ into a domain-level style fingerprint $\widehat{\nu}_0$. Distances to cached source fingerprints determine the source-neighbor set and source-classifier weights. Test labels are used only to compute evaluation metrics. On iWildCam, location metadata additionally partitions unlabeled covariates for the group-level routing variants; predictions use image representations and the resulting source weights. Optimization is confined to source-data training.

\paragraph{Real-data domain construction and target summaries.}
On FMoW-WILDS, year metadata defines $11$ retained source domains after source-domain filtering, and one unlabeled style summary is formed for each evaluation split without using target-domain metadata. Overall accuracy is split-level; region metadata is used only for the region-aware metrics, which exclude ``Other''.

On iWildCam-WILDS, camera-location metadata defines $64$ retained source domains. LADDER forms one unlabeled style summary per target camera location, treating each location as a target domain. Location IDs define the summary groups, and each image is classified using the weights assigned to its location. The alternative \method{} (split summary) forms one summary for the whole validation or test split and uses location metadata only for evaluation. Both variants follow the source-only training and fixed-model inference protocol above. Filtered average-group accuracy averages the target locations with at least $50$ examples.

\subsection{Baselines and Dataset-Specific Results}

\paragraph{Reference baselines.}
ERM, IRM, EQRM, and RDM predict each example independently of the other target examples. DSNR conditions predictions on a domain index estimated from an unlabeled target set. In the simulation, we use a linear DSNR implementation with the empirical target-domain covariate mean as its index. On FMoW and iWildCam, the classification implementation averages the outputs of a source-trained index network over an unlabeled subset of each evaluation split and shares the resulting index across that split. Its trained parameters and normalization statistics remain fixed. The iWildCam location-grouping option used by \method{} does not change DSNR's split-level index.

ARM-CML~\citep{zhang2021arm} learns a context network $c_\eta$ jointly with a predictor $h_\theta$ through source-domain episodes. For an unlabeled target set $\mathcal U_g$, it uses
\[
\bar{\bm c}_g=\frac{1}{|\mathcal U_g|}\sum_{\bm x\in\mathcal U_g}c_\eta(\bm x),
\qquad \widehat p(\bm y\mid\bm x)=\operatorname{softmax}\!\left(h_\theta([\bm x,\bar{\bm c}_g])\right).
\]
Our adaptation follows the authors' context-network implementation, with source-trained parameters and batch-normalization statistics fixed at inference. The primary target sets are the evaluation split on FMoW, each camera location on iWildCam, and each target separately in the simulation.

\paragraph{LADDER routing ablations.}
We use four LADDER routing ablations in the real-data tables. \method{} (uniform) reuses the same frozen representation and source-classifier stack as \method{} but assigns equal weights,
\[
\widehat{\bm{\ell}}_{\mathrm{uni}}(\bm{x})=
\frac{1}{E}\sum_{e=1}^{E}
h_{\widehat{\bm{\psi}}_e}\!\left(f_{\widehat{\bm{\theta}}}(\bm{x})\right).
\]
\method{} ($K=1$) uses the same target summary and distance geometry as \method{} but performs hard $K=1$ routing,
\[
e^*=\arg\min_e d_e,\qquad
\widehat{\bm{\ell}}_{\mathrm{near}}(\bm{x})=
h_{\widehat{\bm{\psi}}_{e^*}}\!\left(f_{\widehat{\bm{\theta}}}(\bm{x})\right).
\]
\method{} (sample gate) is a BSF-style per-sample control that uses the
frozen auxiliary style-domain classifier $\rho_s$ to form
source-membership weights,
\[
\widehat{\bm{\ell}}_{\mathrm{sg}}(\bm{x})
=
\sum_{e=1}^{E}
\pi_e(\bm{x})
h_{\widehat{\bm{\psi}}_e}
\!\left(f_{\widehat{\bm{\theta}}}(\bm{x})\right),
\qquad
\bm{\pi}(\bm{x})
=
\operatorname{softmax}
\!\left(
\rho_s(g_{\widehat{\bm{\phi}}}(\bm{x}))
/T_{\mathrm{sg}}
\right).
\]
We fix $T_{\rm sg}=1$ for all datasets and seeds. On iWildCam, \method{} (group gate) uses the same target grouping metadata as \method{} but averages \method{} (sample gate) weights within each target group,
\[
\bar\pi_{g,e}=\frac{1}{|\mathcal D_g|}\sum_{\bm{x}\in\mathcal D_g}\pi_e(\bm{x}),\qquad
\widehat{\bm{\ell}}_{\mathrm{sg\text{-}group}}(\bm{x})=
\sum_{e=1}^{E}\bar\pi_{g,e}
h_{\widehat{\bm{\psi}}_e}\!\left(f_{\widehat{\bm{\theta}}}(\bm{x})\right),\quad \bm{x}\in\mathcal D_g.
\]
Table~\ref{tab:supp_style_pooled_ablation} evaluates learned summary-conditioned alternatives in the simulation.

\subsubsection{Synthetic Simulation}\label{app:simulation}
\paragraph{Data generation and metrics.}
Each repetition draws two matrices in $\mathbb R^{3\times10}$ with independent entries $(\bm\beta_{\mathrm{base}})_{kj}\sim\mathcal N(0,0.6^2)$ and $(\bm\beta_{\mathrm{shift}})_{kj}\sim\mathcal N(0,0.4^2)$, shared across domains within that repetition. The nine source coordinates are $c_e=-5+1.25(e-1)$ for $e=1,\ldots,9$; the interpolation and extrapolation targets have coordinates $2$ and $6$. At coordinate $c$, the independent causal and style vectors follow
\[
\bm Z_c\sim\mathcal N(\bm0,\bm I_{10}),\qquad
\bm Z_s\sim\mathcal N\!\left(\bm0,\operatorname{diag}(3+0.4c,\,3-0.4c,\,64\bm I_6)\right).
\]
The observation is $\bm X=(\bm Z_c^\top,\bm Z_s^\top)^\top$, and
$Y\mid\bm Z_c,c\sim\operatorname{Categorical}(\operatorname{Softmax}((\bm\beta_{\mathrm{base}}+c\bm\beta_{\mathrm{shift}})\bm Z_c))$.
Each domain contains $1000$ examples. The simulation supplies the latent features directly to \method{}, while the reference baselines receive $\bm X$; PooledJoint receives the concatenated latent vectors. Accuracy is the fraction of correct class predictions. For probabilities $\widehat p_{ik}$ on $n$ target examples, the Brier score is $n^{-1}\sum_{i=1}^{n}\sum_{k=1}^{3}(\widehat p_{ik}-\mathbf1\{y_i=k\})^2$.
The parameter error for the effective causal classifier matrix is
\[
\left\|\bm\Pi_C\bigl(\widehat{\bm\beta}^{(0)}-\bm\beta_*^{(0)}\bigr)\right\|_F,
\qquad \bm\Pi_C=\bm I_C-C^{-1}\bm1\bm1^\top.
\]
This removes the common class shift that leaves Softmax probabilities unchanged. Parameter error is evaluated after fitting; it does not alter training or predictions. Tables report means and population standard deviations across repetitions. The main comparison, KNN sweep, and additional pooled ablation use separate sets of $100$ repetitions.

\begin{table}[!htb]
\centering
\caption{Simulation results over $100$ repetitions (mean $\pm$ standard deviation; accuracy in \%). Param.\ is the centered parameter error; it is not defined for nonlinear ARM-CML.}
\label{tab:supp_simulation_results}
\fitwidetable{%
\begin{tabular}{lccc}
\toprule
\multicolumn{4}{c}{Target A: interpolation, $c=2$} \\
\midrule
Method & Accuracy $\uparrow$ & Brier $\downarrow$ & Parameter error $\downarrow$ \\
\midrule
ERM & $55.56\pm4.85$ & $0.553\pm0.041$ & $3.82\pm0.74$ \\
IRM & $45.91\pm3.35$ & $0.626\pm0.018$ & $4.15\pm0.74$ \\
EQRM & $54.38\pm5.41$ & $0.570\pm0.039$ & $3.95\pm0.77$ \\
DSNR & $55.47\pm5.00$ & $0.554\pm0.041$ & $3.82\pm0.74$ \\
RDM & $55.61\pm4.56$ & $0.555\pm0.038$ & $3.86\pm0.76$ \\
PooledJoint & $55.56\pm4.86$ & $0.553\pm0.041$ & $3.82\pm0.74$ \\
ARM-CML & $67.64 \pm 5.52$ & $0.505 \pm 0.084$ & -- \\
\textbf{LADDER} & $\mathbf{72.88\pm10.28}$ & $\mathbf{0.379\pm0.112}$ & $\mathbf{2.64\pm0.76}$ \\
Oracle & $80.80\pm3.30$ & $0.267\pm0.042$ & $0.00\pm0.00$ \\
\midrule
\multicolumn{4}{c}{Target B: extrapolation, $c=6$} \\
\midrule
Method & Accuracy $\uparrow$ & Brier $\downarrow$ & Parameter error $\downarrow$ \\
\midrule
ERM & $41.86\pm4.92$ & $0.673\pm0.042$ & $10.48\pm1.83$ \\
IRM & $38.68\pm2.39$ & $0.668\pm0.009$ & $10.59\pm1.84$ \\
EQRM & $41.26\pm4.20$ & $0.662\pm0.032$ & $10.52\pm1.84$ \\
DSNR & $41.59\pm5.27$ & $0.676\pm0.044$ & $10.48\pm1.83$ \\
RDM & $42.01\pm3.44$ & $0.664\pm0.023$ & $10.49\pm1.85$ \\
PooledJoint & $41.87\pm4.93$ & $0.673\pm0.042$ & $10.48\pm1.83$ \\
ARM-CML & $82.79 \pm 2.98$ & $0.273 \pm 0.047$ & -- \\
\textbf{LADDER} & $\mathbf{86.23\pm3.70}$ & $\mathbf{0.234\pm0.037}$ & $\mathbf{8.15\pm1.75}$ \\
Oracle & $91.80\pm1.48$ & $0.117\pm0.020$ & $0.00\pm0.00$ \\
\bottomrule
\end{tabular}}
\end{table}

\begin{figure}[!htb]
\centering
\includegraphics[width=0.85\textwidth]{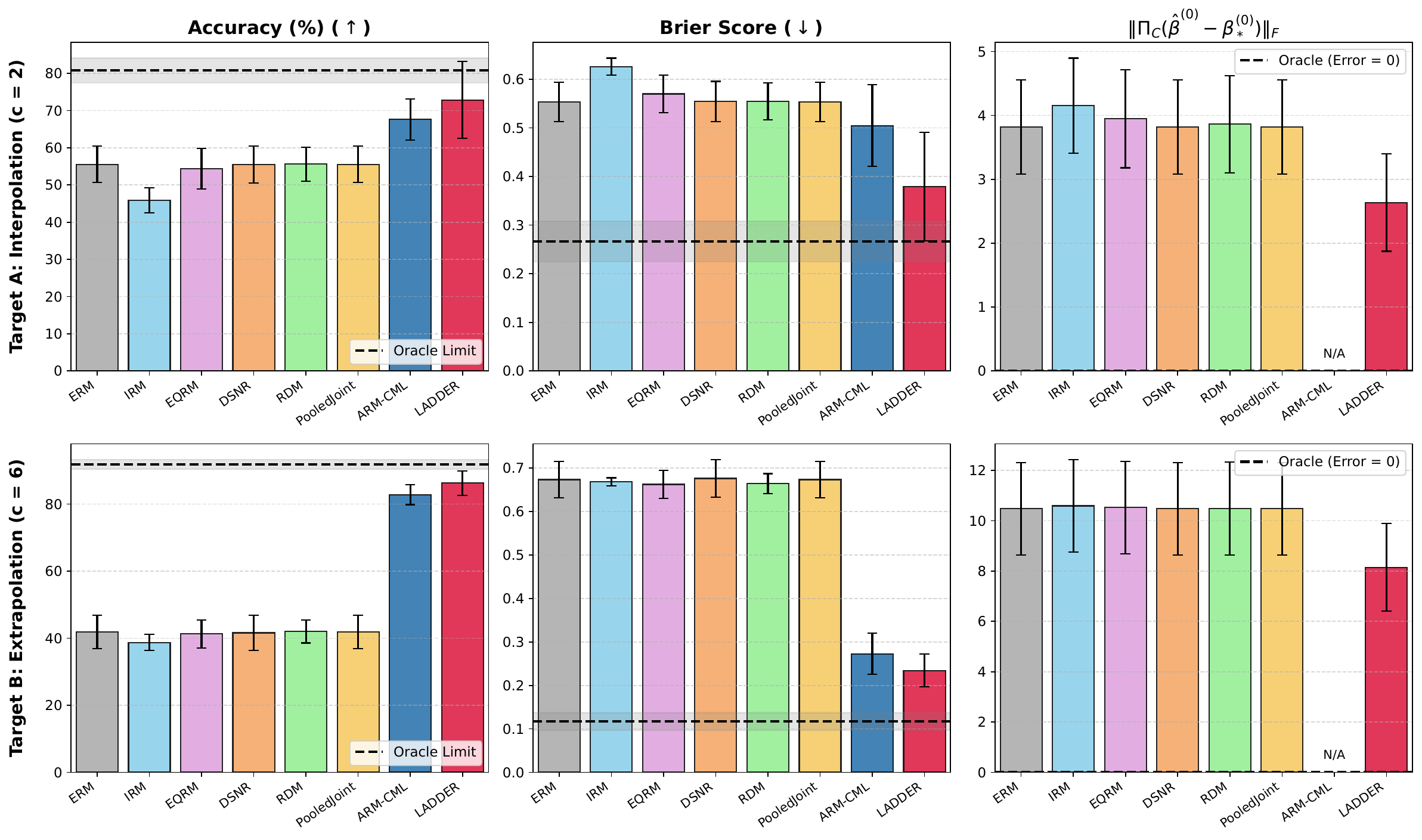}
\caption{Simulation results over $100$ independent repetitions. Bars and error bars show means and standard deviations; parameter errors use class-centered classifier matrices and are not defined for nonlinear ARM-CML.}
\label{fig:supp_simulation_results}
\end{figure}

Figure~\ref{fig:supp_simulation_results} reports the controlled simulation results. The top and bottom rows separate the interpolation and extrapolation targets. This controlled SCM isolates the source-reweighting mechanism rather than the sufficient identifiability conditions in Section~\ref{app:auxiliary_statements}. The style representation is $8$-dimensional, there are $9$ source domains, and each domain contains $1000$ samples; the main \method{} setting uses $K=4$ source neighbors. Two coordinates encode the environment coordinate through opposite covariance shifts, while the remaining coordinates are high-variance nuisance variation. In the main comparison, \method{} reaches mean accuracies of $72.88\%$ and $86.23\%$ on the interpolation and extrapolation targets, respectively. It has the highest mean accuracy and lowest mean Brier score among the main non-oracle methods on both targets, and the lowest centered parameter error among those with linear causal coefficients. ARM-CML reaches $67.64\%$ and $82.79\%$ accuracy, with Brier scores $0.505$ and $0.273$. Its $100$ repetitions use the same data-generating mechanism but independent draws from the previously retained baselines.

\begin{figure}[!htb]
\centering
\includegraphics[width=0.85\textwidth]{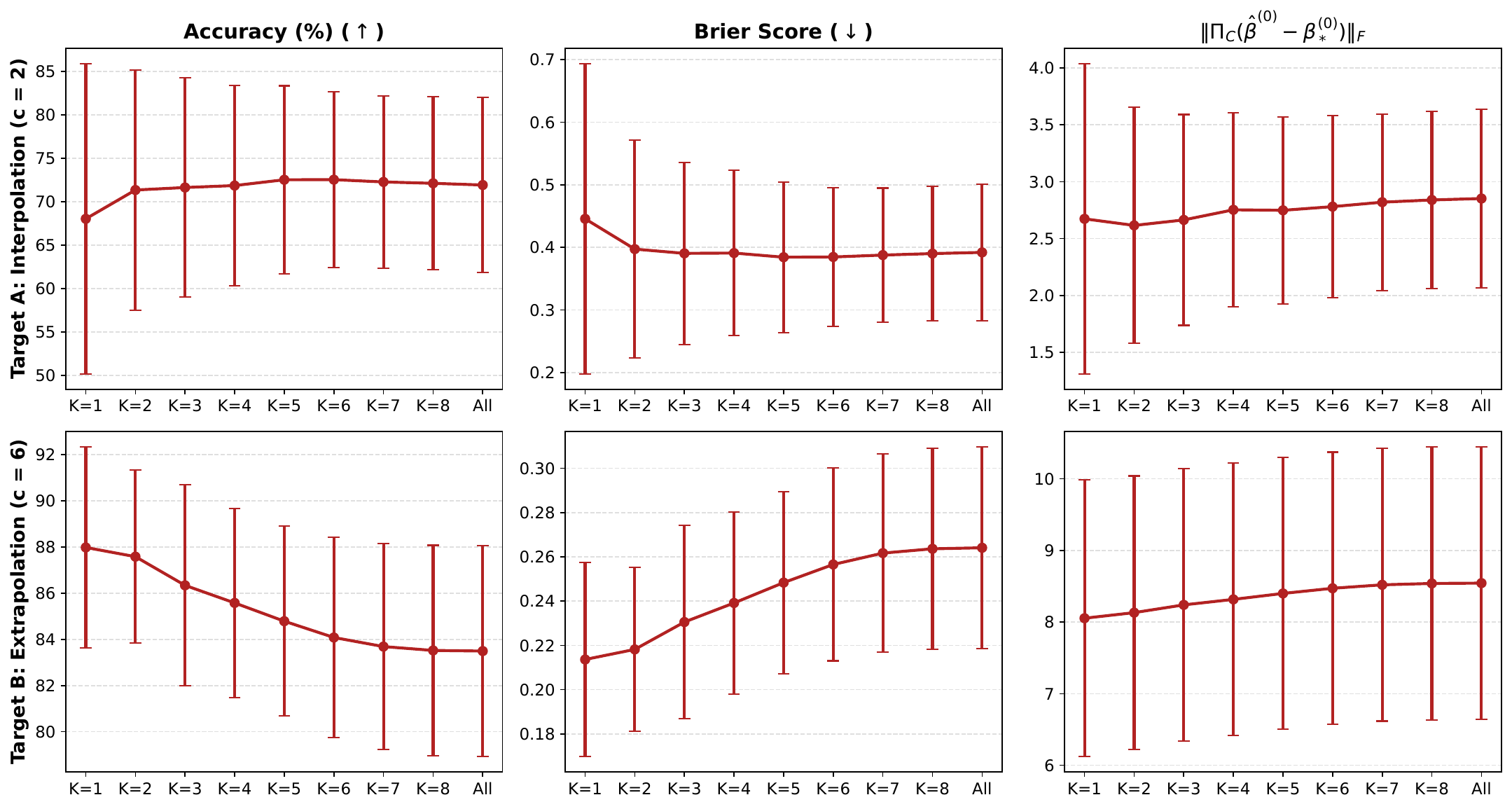}
\caption{Synthetic simulation KNN source-selection ablation over $K\in\{1,\ldots,8\}$ and ``All''. The main setting uses $K=4$, while ``All'' removes the KNN truncation and reweights all available source classifiers.}
\label{fig:supp_simulation_knn_ablation}
\end{figure}

Figure~\ref{fig:supp_simulation_knn_ablation} studies the local source-selection radius under the same covariance-shift style protocol. For the interpolation target, mean accuracy ranges from $71.34\%$ to $72.54\%$ across $K=2,\ldots,8$ and the ``All'' variant. For extrapolation, $K=1$ reaches $87.98\%$, compared with $83.49\%$ when all sources are reweighted. Small neighborhoods focus on the boundary source domains closest to the target mechanism. We retain $K=4$ as the main setting: it lies in the stable interpolation region and preserves local source selection for extrapolation.

\begin{figure}[!htb]
\centering
\includegraphics[width=0.85\textwidth]{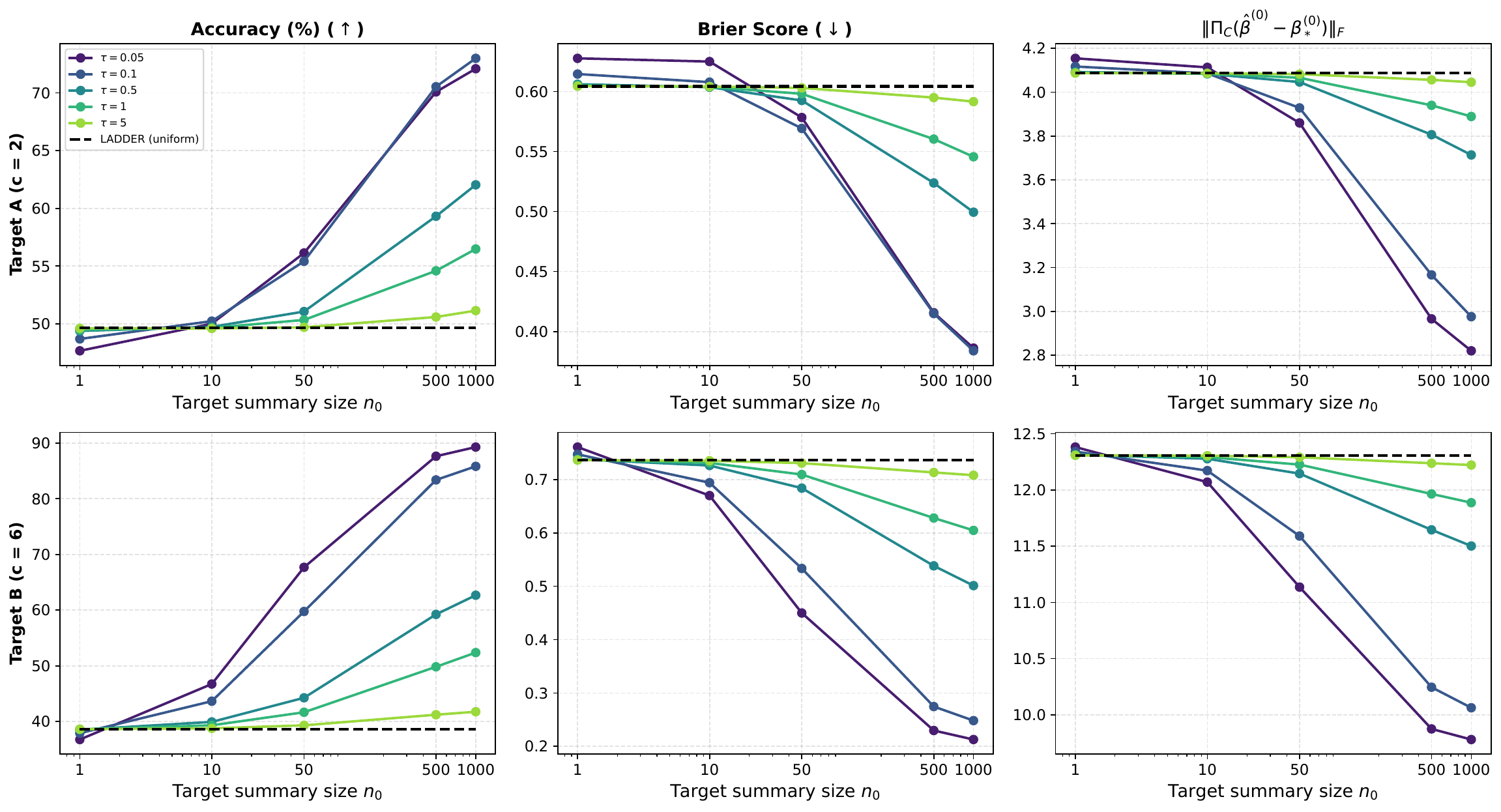}
\caption{Synthetic simulation ablation for target-summary size and temperature $\tau\in\{0.05,0.1,0.5,1,5\}$. Rows correspond to interpolation and extrapolation targets; columns report accuracy, Brier score, and centered parameter error. Curves show mean results over target-summary resamples pooled across five training repetitions. \method{} (uniform) uniformly averages all source classifiers.}
\label{fig:supp_simulation_navigation_ablation}
\end{figure}

The target-summary ablation uses five independent repetitions of data generation and source training. Within each repetition, it resamples subsets of the same $1000$ target covariates, using $5000$, $1000$, $1000$, $300$, and $300$ inference draws for $n_0\in\{1,10,50,500,1000\}$, respectively. It uses all nine source classifiers, held fixed across inference draws, and varies $\tau\in\{0.05,0.1,0.5,1,5\}$. These inference draws are not additional training repetitions.

Figure~\ref{fig:supp_simulation_navigation_ablation} varies $n_0$ and $\tau$ while holding the source classifiers and domain mechanisms fixed. Two of the eight style coordinates carry environment information; the remaining six contain high-variance nuisance variation. At $n_0=1$, mean accuracy remains close to uniform source averaging on both targets; at $n_0=10$ and $\tau=0.1$, it is already higher, with larger gains at larger $n_0$.

As $n_0$ increases, the target style summary becomes more informative for navigation. With $n_0=1000$ and $\tau=0.1$, mean accuracies are $73.00\%$ and $85.82\%$ on the two targets, compared with $49.64\%$ and $38.66\%$ for uniform source averaging. Smaller temperatures help select the relevant source classifiers, especially for extrapolation. Thus $\tau$ should be interpreted as a temperature parameter controlling how strongly \method{} trusts empirical style-distribution distances: a large temperature such as $\tau=5$ is a conservative fallback when target summaries are tiny, whereas smaller temperatures become useful once the target distribution is estimated with enough samples.

Table~\ref{tab:arm_sim_controls} compares the same fitted ARM-CML model using separate target contexts, a context pooled across both targets, and shuffled groups preserving group sizes and summary caps. Separate contexts improve extrapolation accuracy but reduce interpolation accuracy relative to pooled or shuffled contexts.

\begin{table}[!htb]\centering
\caption{ARM-CML simulation context controls (mean $\pm$ population SD, $100$ repetitions).}
\label{tab:arm_sim_controls}
\begin{tabular}{llcc}\toprule
Target & Context & Accuracy (\%) $\uparrow$ & Brier $\downarrow$ \\ \midrule
Interpolation & Split & $73.87 \pm 4.15$ & $0.420 \pm 0.066$ \\
Interpolation & True groups & $67.64 \pm 5.52$ & $0.505 \pm 0.084$ \\
Interpolation & Shuffled groups & $73.99 \pm 4.01$ & $0.419 \pm 0.065$ \\
Extrapolation & Split & $77.14 \pm 4.47$ & $0.362 \pm 0.072$ \\
Extrapolation & True groups & $82.79 \pm 2.98$ & $0.273 \pm 0.047$ \\
Extrapolation & Shuffled groups & $77.24 \pm 4.33$ & $0.360 \pm 0.070$ \\
\bottomrule\end{tabular}\end{table}

\paragraph{Additional style-conditioned pooled baselines.}
The main simulation includes PooledJoint as an additive linear classifier on $(\bm{Z}_c,\bm{Z}_s^{(e)})$. This instantiates the additive pooled class in Corollary~\ref{cor:pooled_misspec}; more expressive style-conditioned pooled predictors are examined below. To test whether the main simulation gap is caused only by this additive restriction, Table~\ref{tab:supp_style_pooled_ablation} and Figure~\ref{fig:supp_style_pooled_ablation} add three stronger pooled variants. Pooled-Interaction appends a domain-level style summary and explicit products between the style summary and $\bm{Z}_c$; Pooled-MLP uses a nonlinear pooled network on $\bm{Z}_c$, $\bm{Z}_s^{(e)}$, and the style summary; and StyleHyperNet maps the style summary to classifier parameters. All three baselines use the same source data and unlabeled target style samples as \method{}. Their domain summaries concatenate coordinate-wise means and log variances, with normalization statistics fitted on source domains. \method{} computes its quantile-based distances from the empirical target style samples.

\begin{table}[!htb]
\centering
\caption{Supplementary style-conditioned pooled ablation in the controlled SCM over $100$ repetitions. Accuracy is in \%; Param.\ is the centered parameter error. Pooled-MLP has no single linear target classifier, so parameter error is not reported.}
\label{tab:supp_style_pooled_ablation}
\fitwidetable{\begin{tabular}{llccc}
\toprule
Target & Method & Acc. $\uparrow$ & Brier $\downarrow$ & Param. $\downarrow$ \\
\midrule
Interpolation, $c=2$ & PooledJoint & $54.84\pm4.49$ & $0.561\pm0.037$ & $3.822\pm0.735$ \\
 & Pooled-Interaction & $61.74\pm19.60$ & $0.534\pm0.288$ & $2.975\pm1.734$ \\
 & Pooled-MLP & $58.21\pm18.52$ & $0.646\pm0.331$ & \textemdash \\
 & StyleHyperNet & $64.46\pm19.60$ & $0.505\pm0.300$ & $2.709\pm1.818$ \\
 & LADDER & $72.57\pm11.01$ & $0.383\pm0.119$ & $2.556\pm0.998$ \\
 & Oracle & $80.86\pm3.10$ & $0.267\pm0.041$ & $0.000\pm0.000$ \\
\midrule
Extrapolation, $c=6$ & PooledJoint & $41.94\pm4.36$ & $0.669\pm0.034$ & $10.559\pm1.823$ \\
 & Pooled-Interaction & $85.01\pm10.43$ & $0.214\pm0.138$ & $5.769\pm2.535$ \\
 & Pooled-MLP & $78.57\pm11.01$ & $0.322\pm0.179$ & \textemdash \\
 & StyleHyperNet & $84.05\pm11.60$ & $0.230\pm0.175$ & $4.286\pm2.454$ \\
 & LADDER & $85.74\pm3.75$ & $0.239\pm0.036$ & $8.258\pm1.726$ \\
 & Oracle & $91.88\pm1.66$ & $0.117\pm0.023$ & $0.000\pm0.000$ \\
\bottomrule
\end{tabular}}
\end{table}

\begin{figure}[!htb]
\centering
\includegraphics[width=0.85\textwidth]{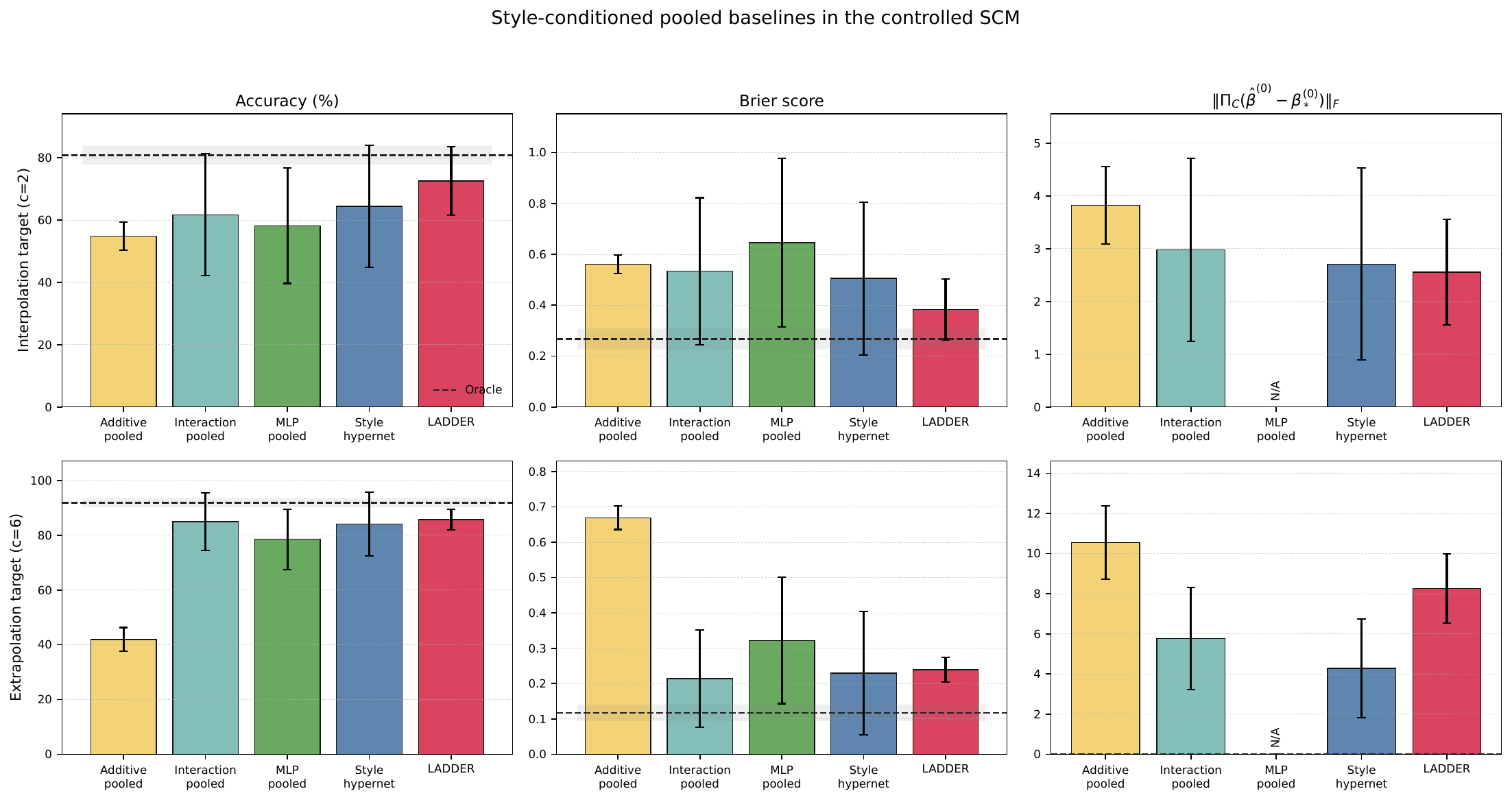}
\caption{Supplementary style-conditioned pooled ablation in the controlled SCM. Stronger pooled predictors close much of the additive PooledJoint gap, especially in extrapolation, but they also have substantially larger variance across repetitions. \method{} remains competitive in extrapolation and gives the strongest interpolation performance among non-oracle methods.}
\label{fig:supp_style_pooled_ablation}
\end{figure}

LADDER gives the best mean interpolation accuracy, Brier score, and parameter recovery among non-oracle methods. In extrapolation, it has the highest mean accuracy and lower accuracy variability than the expressive pooled alternatives, while Pooled-Interaction and StyleHyperNet have lower mean Brier scores and parameter errors.

\subsubsection{FMoW-WILDS}
\begin{table}[!htb]
\centering
\caption{FMoW-WILDS results over five random seeds. Accuracies are reported in \%; ID-to-test accuracy drop is reported in percentage points.}
\label{tab:supp_fmow_results}
\fitwidetable{%
\begin{tabular}{lcccc}
	\toprule
Method & Test $\uparrow$ & Worst $\uparrow$ & Avg. $\uparrow$ & Drop $\downarrow$ \\
\midrule
ERM & $46.68 \pm 0.96$ & $30.68 \pm 1.92$ & $45.49 \pm 1.30$ & $9.14 \pm 0.77$ \\
IRM & $46.68 \pm 0.54$ & $29.06 \pm 1.79$ & $45.39 \pm 0.74$ & $8.88 \pm 0.73$ \\
EQRM & $46.85 \pm 0.55$ & $28.89 \pm 1.61$ & $45.15 \pm 0.54$ & $9.09 \pm 0.71$ \\
DSNR & $45.52 \pm 0.87$ & $27.03 \pm 1.89$ & $43.93 \pm 0.87$ & $9.78 \pm 0.61$ \\
RDM & $47.15 \pm 1.19$ & $29.56 \pm 1.68$ & $45.22 \pm 1.36$ & $9.27 \pm 0.70$ \\
ARM-CML & $47.89 \pm 0.35$ & $29.53 \pm 1.49$ & $46.32 \pm 0.80$ & $\mathbf{8.56 \pm 0.20}$ \\
\midrule
\method{} (uniform) & $49.65 \pm 0.32$ & $30.46 \pm 1.24$ & $48.55 \pm 0.45$ & $8.69 \pm 0.22$ \\
\method{} ($K=1$) & $49.12 \pm 0.53$ & $30.71 \pm 0.97$ & $48.08 \pm 0.65$ & $8.70 \pm 0.71$ \\
\method{} (sample gate) & $49.98 \pm 0.31$ & $30.67 \pm 1.18$ & $48.67 \pm 0.58$ & $9.01 \pm 0.25$ \\
	\textbf{LADDER} & $\textbf{50.19} \pm \textbf{0.34}$ & $\textbf{31.46} \pm \textbf{0.74}$ & $\textbf{49.05} \pm \textbf{0.54}$ & $8.87 \pm 0.53$ \\
\bottomrule
\end{tabular}}
\end{table}

\begin{figure}[!htb]
\centering
\includegraphics[width=0.85\textwidth]{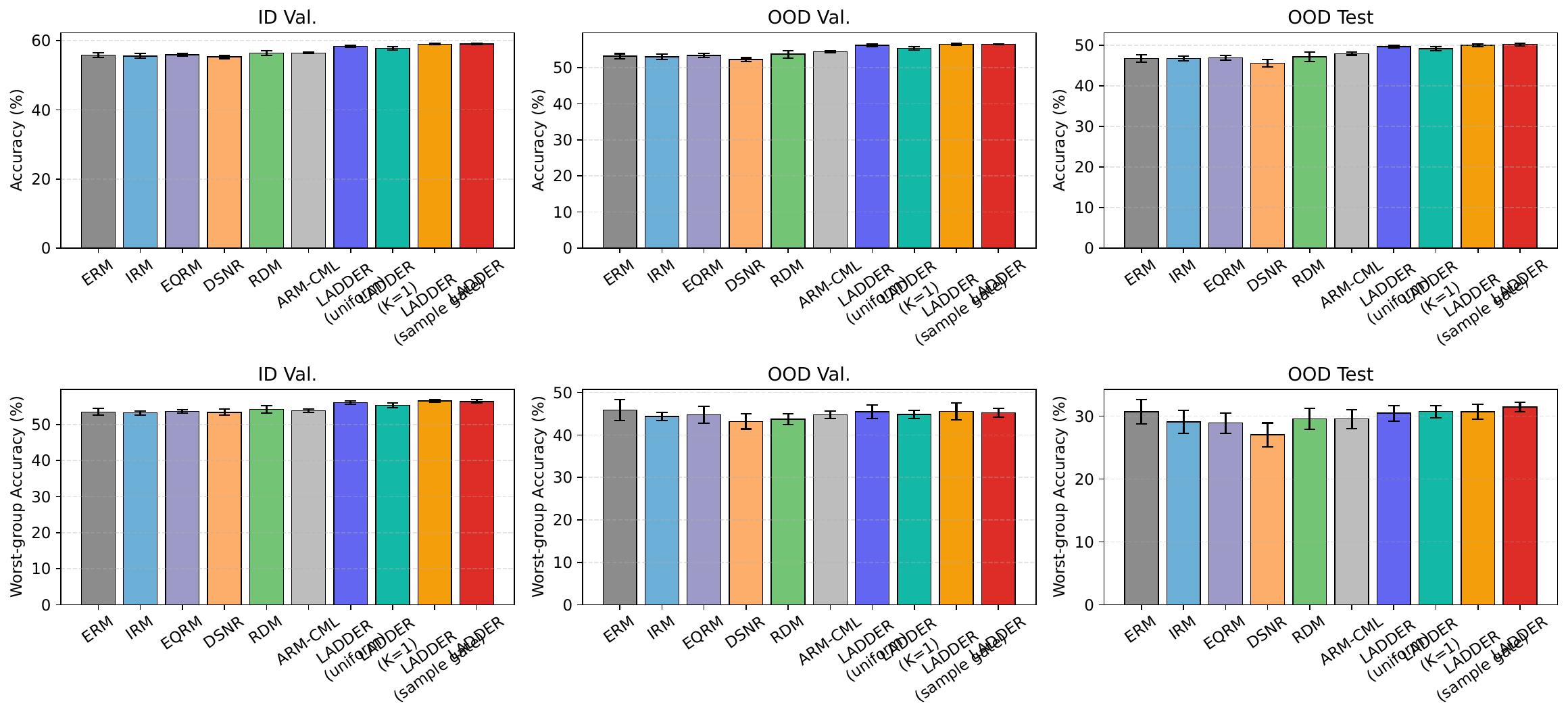}
\caption{FMoW-WILDS accuracy and worst-region accuracy on the ID validation, OOD validation, and OOD test splits.}
\label{fig:supp_fmow_accuracy}
\end{figure}

Figures~\ref{fig:supp_fmow_accuracy} and \ref{fig:supp_fmow_tradeoff_weights} provide the FMoW-WILDS diagnostic view. They report both aggregate accuracy and region-aware robustness, together with the source-classifier weights used by \method{} for the representative target split. Figure~\ref{fig:supp_fmow_knn_ablation} further isolates the effect of the KNN source-selection radius.

\begin{figure}[!htb]
\centering
\begin{tabular}{ccc}
\includegraphics[width=0.295\textwidth]{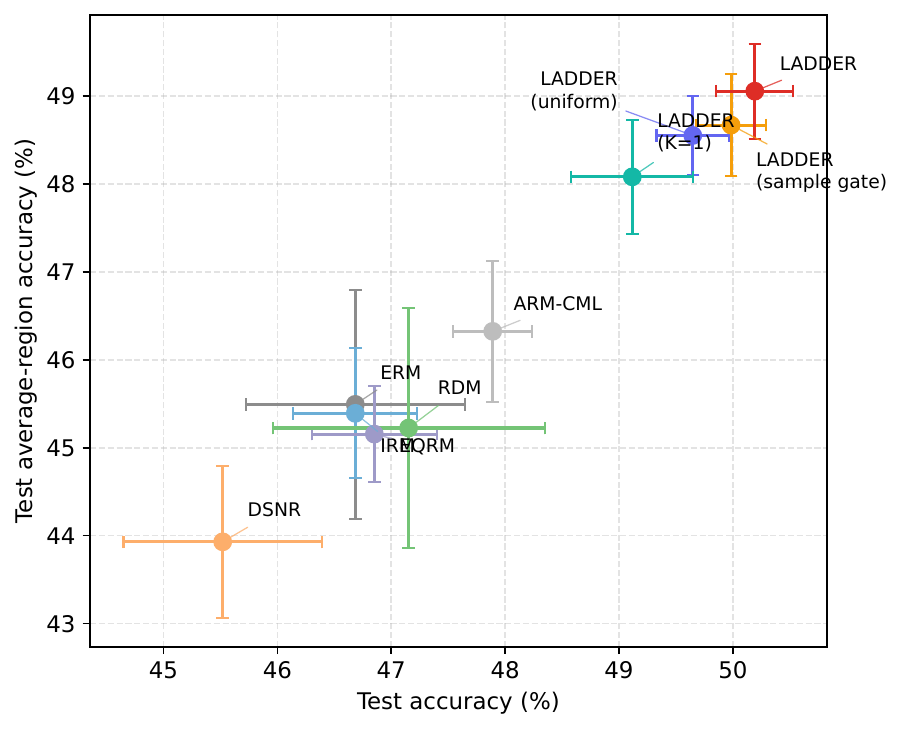} &
\includegraphics[width=0.295\textwidth]{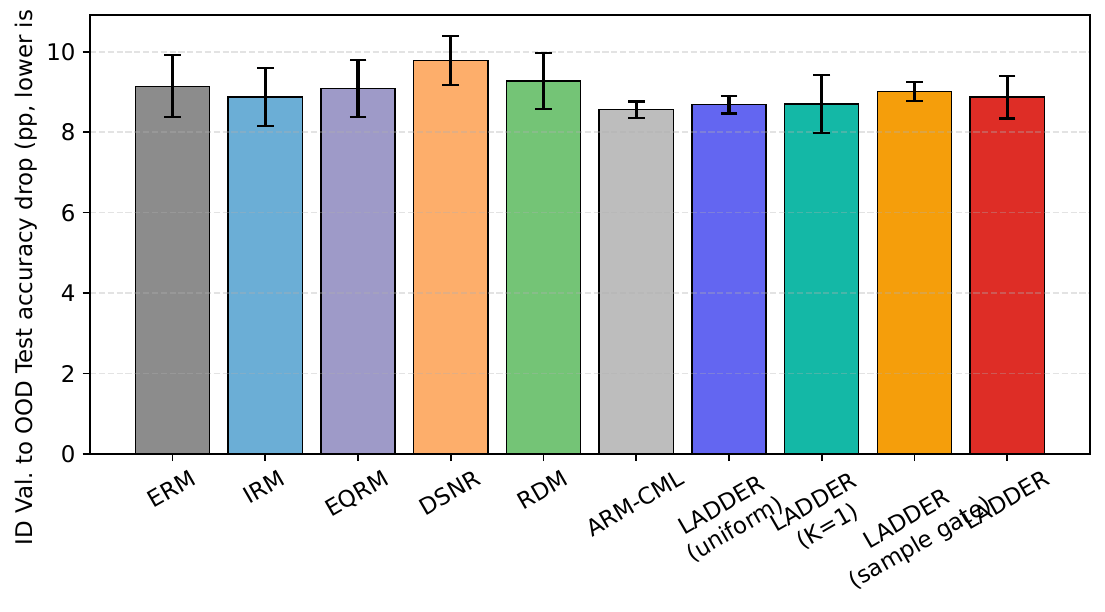} &
\includegraphics[width=0.295\textwidth]{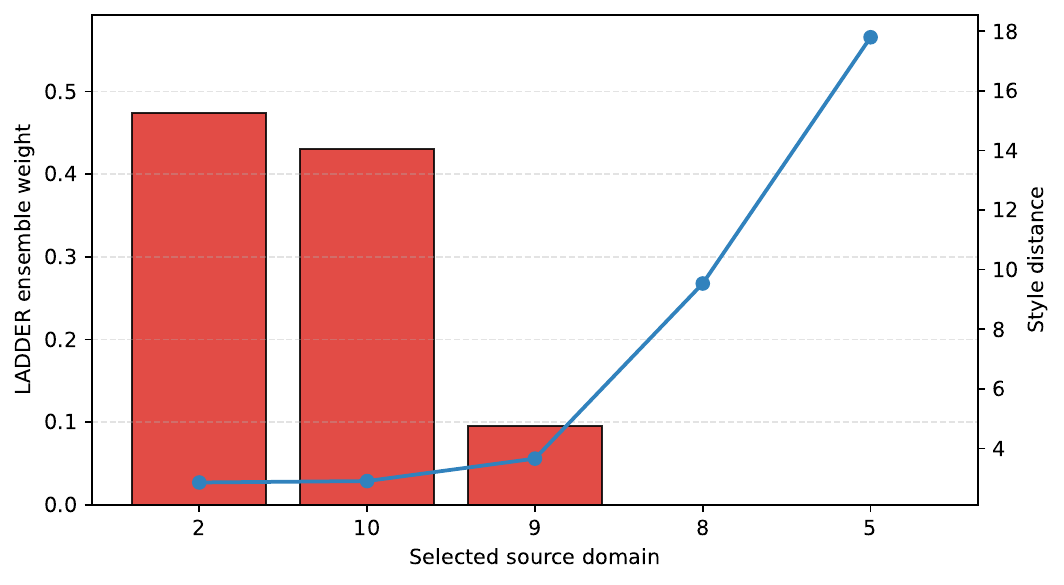}
\end{tabular}
\caption{FMoW-WILDS accuracy/average-region trade-off, OOD degradation, and representative \method{} source-classifier weights.}
\label{fig:supp_fmow_tradeoff_weights}
\end{figure}

\begin{figure}[!htb]
\centering
\includegraphics[width=0.85\textwidth]{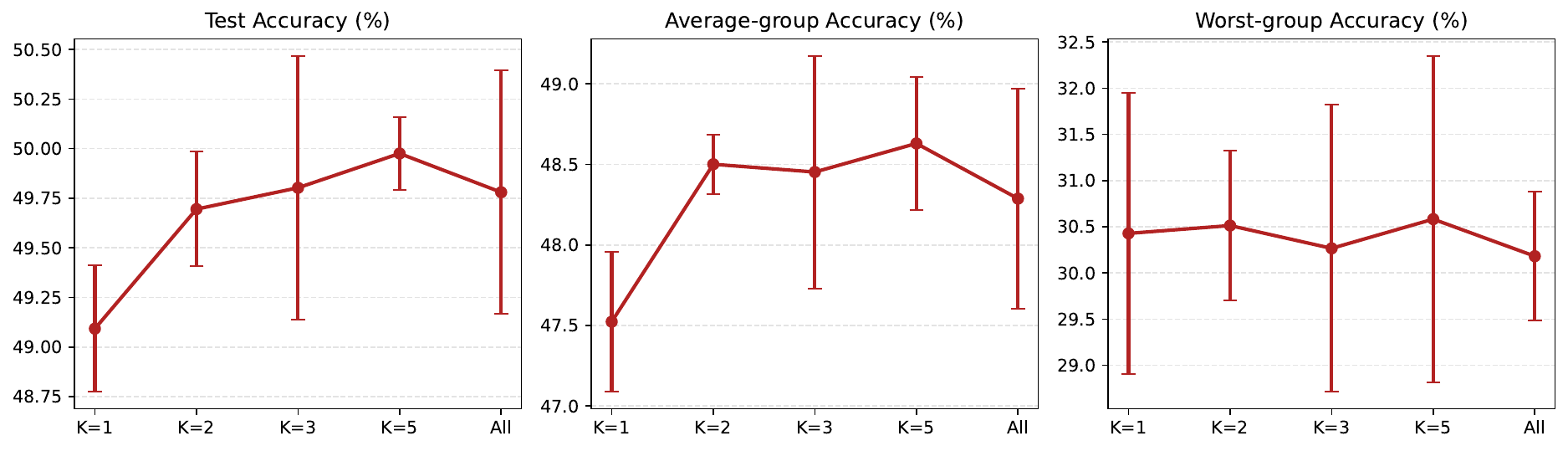}
\caption{FMoW-WILDS KNN source-selection ablation over five random seeds. The main setting uses $K=5$, while ``All'' removes the KNN truncation and reweights all retained source domains. The sweep includes $K\in\{1,2,3,5\}$ and shows that local source selection is stable across nearby neighborhood sizes. The sweep uses up to $8192$ covariates per source or target summary, compared with $1024$ in the main table.}
\label{fig:supp_fmow_knn_ablation}
\end{figure}

The FMoW-WILDS KNN ablation supports the inference-stage reweighting mechanism used by \method{}; its $K=1$ endpoint corresponds to hard nearest-source classifier routing, while the navigation ablation in simulation includes the uniform source-classifier average. Moving from the very local $K=1$ setting to moderate neighborhoods improves both test accuracy and average-region accuracy, and the main setting $K=5$ gives the strongest mean performance. Removing the neighborhood restriction remains competitive but slightly dilutes the nearest source classifiers, especially on the region-aware metrics.

\begin{figure}[!htb]
\centering
\includegraphics[width=0.76\textwidth]{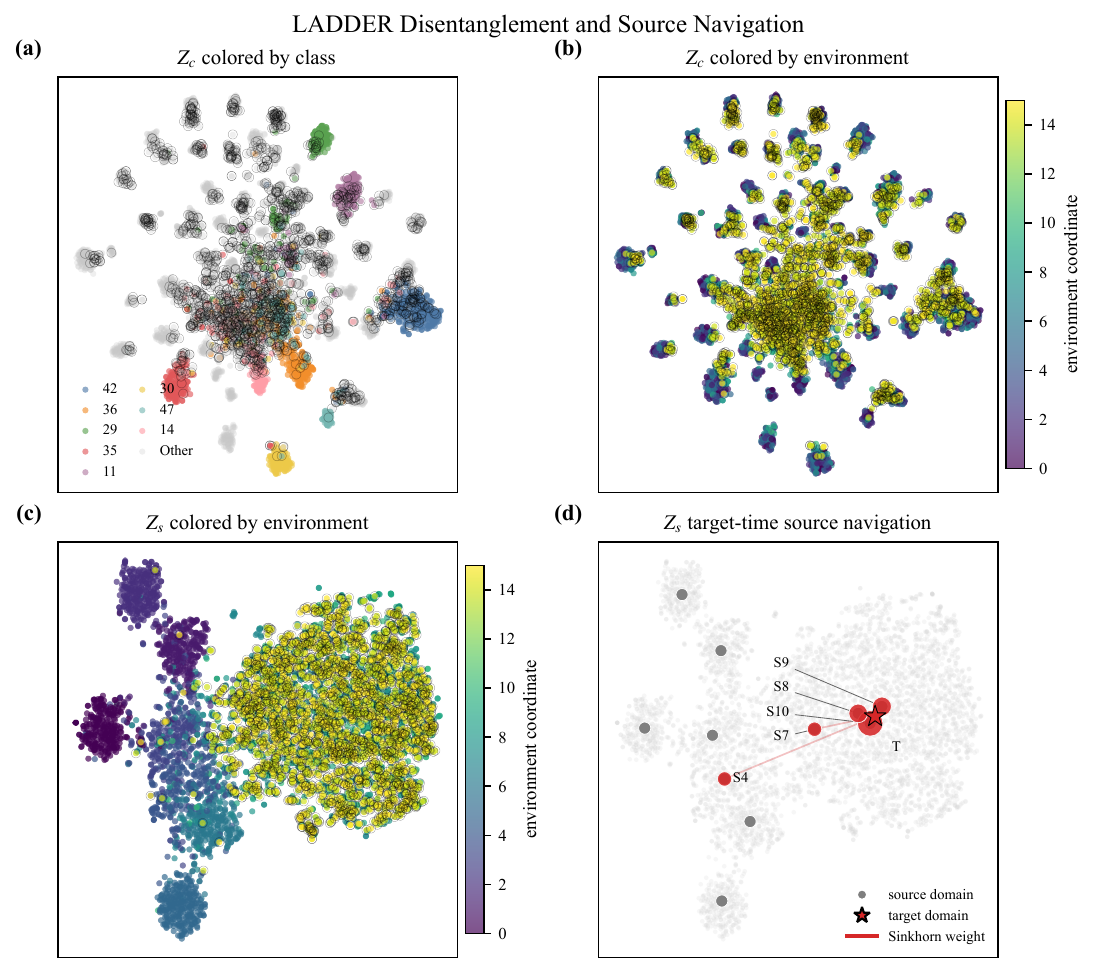}
\caption{Full FMoW-WILDS t-SNE visualization \citep{vandermaaten2008visualizing}. Panels (a,b) show the causal representation $\bm{Z}_c$ colored by class and domain, respectively; panels (c,d) show the style representation $\bm{Z}_s^{(e)}$ colored by domain and Sinkhorn source reweighting.}
\label{fig:supp_fmow_tsne_full}
\end{figure}

Figure~\ref{fig:supp_fmow_tsne_full} visualizes the domain structure in the style representation used for source reweighting. The style representation $\bm{Z}_s^{(e)}$ preserves domain geometry, while the causal representation $\bm{Z}_c$ remains the input used by the source-specific classifiers. The saved t-SNE coordinates are held fixed while navigation weights are recomputed using the FMoW inference settings: balanced debiased Sinkhorn distance with $p=1$, blur $0.05$, scaling $0.8$, $K=5$, and $\tau=0.5$. This visualization uses the fixed seed-$0$ feature cache, containing $3000$ source and $3000$ target examples; the displayed embedding contains $5000$ examples. Source summaries contain all $272$--$276$ cached examples per source year. One target summary pools the test split, including both target years, and uses $1024$ seed-sampled covariates. Style coordinates are standardized by the pooled source mean and population standard deviation, with a $10^{-6}$ standard-deviation floor. The target star is the mean of the displayed target coordinates. Thus the routing settings match the main FMoW protocol, while the visualization cache is separate from the five-seed evaluation.

\begin{figure}[!htb]
\centering
\includegraphics[width=0.9\textwidth]{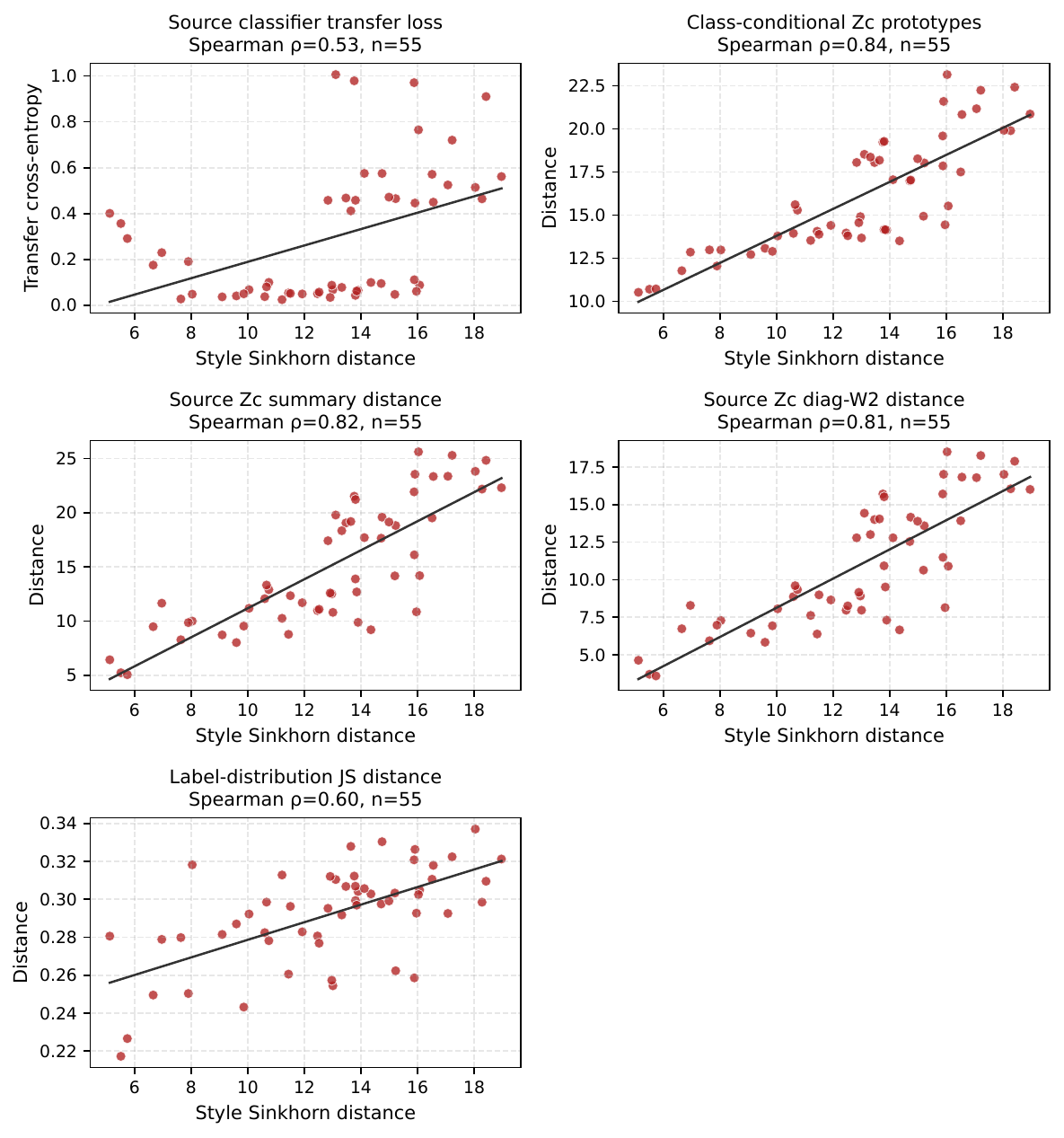}
\caption{FMoW-WILDS source-geometry diagnostic over all $55$ source-domain pairs. The horizontal axis is the debiased Sinkhorn distance between normalized style samples from the diagnostic run. The vertical axes report bidirectional source-classifier transfer loss, class-conditional $\bm{Z}_c$ prototype distance, two $\bm{Z}_c$ summary distances, and label-distribution distance.}
\label{fig:supp_fmow_source_geometry}
\end{figure}

Figure~\ref{fig:supp_fmow_source_geometry} tests whether the learned style distribution is only a domain identifier or also a quantitative signal for classifier reweighting. The second panel measures the distance between class prototypes in $\bm Z_c$. Each prototype is the mean representation of one class in one source domain; for a domain pair, we average the Euclidean distances between prototypes over classes with at least three cached examples in each domain. The first panel relates source-style distance to the symmetric transfer loss $d^{\mathrm{tr}}_{e,e'}=\frac{1}{2}(L_{e\rightarrow e'}+L_{e'\rightarrow e})$, where $L_{e\rightarrow e'}$ is the cross-entropy of the classifier trained on source $e$ when evaluated on source-domain $e'$ examples. The representative diagnostic uses cached features from the seed-$0$ diagnostic run, selected independently of the observed correlation value. Its $3000$ source examples cover $11$ domains with $272$--$276$ examples per domain. We standardize style coordinates using these source examples and compute balanced debiased Sinkhorn distances with $p=1$, blur $0.05$, and scaling $0.8$. The cached features and transfer losses come from the same diagnostic run; this cache is separate from the main five-seed evaluation and uses fewer samples per domain. Across the $55$ unordered source-domain pairs, the Sinkhorn style distance correlates with bidirectional source-classifier transfer loss (Spearman $\rho=0.53$) and class-conditional $\bm Z_c$ prototype distance ($\rho=0.84$). These pairwise correlations are descriptive because domain pairs share source domains and should not be read as independent-pair significance tests. The prototype diagnostic associates style proximity with similarity of class means in $\bm Z_c$, while the transfer-loss diagnostic provides complementary evidence about source-classifier transferability. Both diagnostics use source data with fixed encoders, classifiers, and normalization statistics.

A source-only probe diagnostic gives a complementary view. Probe features are seed-controlled subsamples with at most $4096$ examples per source year, standardized from the probe training split, and split into an $80/20$ source-only train/held-out split. Each probe is a linear Softmax head trained for $50$ epochs with AdamW (learning rate $10^{-2}$, weight decay $10^{-4}$, batch size $1024$). On held-out FMoW source examples, a linear probe predicts the year domain more accurately from $\bm{Z}_s^{(e)}$ than from $\bm Z_c$ ($63.6\pm0.4\%$ vs. $56.2\pm0.7\%$), while label probes remain high for both branches ($\bm Z_c$: $96.5\pm0.3\%$; $\bm{Z}_s^{(e)}$: $95.9\pm0.3\%$). We therefore do not interpret $\bm{Z}_s^{(e)}$ as a label-free nuisance. To control for class composition, we also train class-conditional domain probes on classes represented in at least $3$ source years with at least $20$ sampled examples per class-year pair ($46.8\pm0.8$ eligible classes; majority baseline $18.8\pm0.5\%$). Conditioned on class, $\bm{Z}_s^{(e)}$ still predicts the source year better than $\bm Z_c$ ($56.8\pm0.6\%$ vs. $41.3\pm1.0\%$), a $15.5\pm1.4$ point gap. These diagnostics support the operational claim that $\bm{Z}_s^{(e)}$ is useful for domain-level routing, but they do not imply statistical independence from labels.

\subsubsection{iWildCam-WILDS}\label{app:iwildcam_results}
\begin{table}[!htb]
\centering
\caption{iWildCam-WILDS results over five random seeds. Test accuracy, macro F1, and filtered average-group accuracy are reported in \%.}
\label{tab:supp_iwildcam_results}
\fitwidetable{%
\begin{tabular}{lccc}
	\toprule
Method & Test $\uparrow$ & Macro F1 $\uparrow$ & Filt. Avg. $\uparrow$ \\
\midrule
ERM & $42.26 \pm 5.48$ & $11.23 \pm 1.35$ & $33.29 \pm 7.61$ \\
IRM & $34.59 \pm 10.11$ & $6.20 \pm 1.08$ & $23.42 \pm 9.71$ \\
EQRM & $40.75 \pm 6.62$ & $10.45 \pm 1.37$ & $32.22 \pm 8.07$ \\
DSNR & $47.36 \pm 7.96$ & $12.34 \pm 1.12$ & $39.44 \pm 2.07$ \\
RDM & $39.36 \pm 7.67$ & $8.97 \pm 1.78$ & $31.09 \pm 5.37$ \\
ARM-CML & $54.08 \pm 1.03$ & $\mathbf{14.74 \pm 0.58}$ & $49.52 \pm 0.84$ \\
\midrule
\method{} (uniform) & $44.42 \pm 1.03$ & $4.53 \pm 0.25$ & $34.13 \pm 0.78$ \\
\method{} ($K=1$) & $57.96 \pm 2.73$ & $11.62 \pm 0.90$ & $48.41 \pm 0.44$ \\
\method{} (sample gate) & $58.00 \pm 3.05$ & $14.40 \pm 0.70$ & $48.16 \pm 1.23$ \\
\method{} (group gate) & $59.22 \pm 3.10$ & $12.00 \pm 0.86$ & $48.29 \pm 1.98$ \\
\method{} (split summary) & $41.19 \pm 3.76$ & $2.42 \pm 0.55$ & $36.22 \pm 1.80$ \\
	\textbf{\method{}} & $\textbf{61.71} \pm \textbf{1.76}$ & $11.61 \pm 0.65$ & $\textbf{50.01} \pm \textbf{0.84}$ \\
\bottomrule
\end{tabular}}
\end{table}

\begin{figure}[!htb]
\centering
\includegraphics[width=0.85\textwidth]{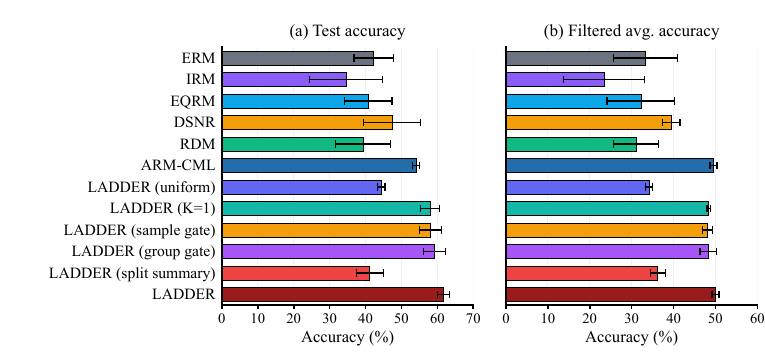}
\caption{iWildCam-WILDS test accuracy and filtered average-group accuracy.}
\label{fig:supp_iwildcam_accuracy}
\end{figure}

LADDER forms a separate target summary for each camera location, matching the domain definition in this experiment. Table~\ref{tab:supp_iwildcam_results} and Figures~\ref{fig:supp_iwildcam_accuracy} and~\ref{fig:supp_iwildcam_tradeoff_ablation} also report \method{} (split summary), which pools all target locations into one summary.

Table~\ref{tab:arm_real_controls} evaluates the same fitted ARM-CML model with whole-split, true-group, and shuffled-group contexts, preserving group sizes and summary caps (regions on FMoW; camera locations on iWildCam). Test accuracy changes little across these choices, and true location groups do not improve macro F1 or filtered average-group accuracy over shuffled groups on iWildCam.

\begin{table}[!htb]\centering
\caption{ARM-CML target-context controls (mean $\pm$ sample SD, five seeds; all metrics in \%).}
\label{tab:arm_real_controls}
\begin{tabular}{llccc}\toprule
Dataset & Context & Test $\uparrow$ & Worst / Macro F1 $\uparrow$ & Avg. / Filt. Avg. $\uparrow$ \\ \midrule
FMoW & Split & $47.89 \pm 0.35$ & $29.53 \pm 1.49$ & $46.32 \pm 0.80$ \\
FMoW & True groups & $47.92 \pm 0.35$ & $29.70 \pm 1.62$ & $46.34 \pm 0.81$ \\
FMoW & Shuffled groups & $47.90 \pm 0.34$ & $29.54 \pm 1.47$ & $46.32 \pm 0.80$ \\
iWildCam & Split & $54.06 \pm 0.86$ & $15.05 \pm 0.60$ & $49.98 \pm 0.65$ \\
iWildCam & True groups & $54.08 \pm 1.03$ & $14.74 \pm 0.58$ & $49.52 \pm 0.84$ \\
iWildCam & Shuffled groups & $54.07 \pm 0.86$ & $15.10 \pm 0.60$ & $49.99 \pm 0.62$ \\
\bottomrule\end{tabular}\end{table}

\begin{figure}[!htb]
\centering
\begin{tabular}{cc}
\includegraphics[width=0.415\textwidth]{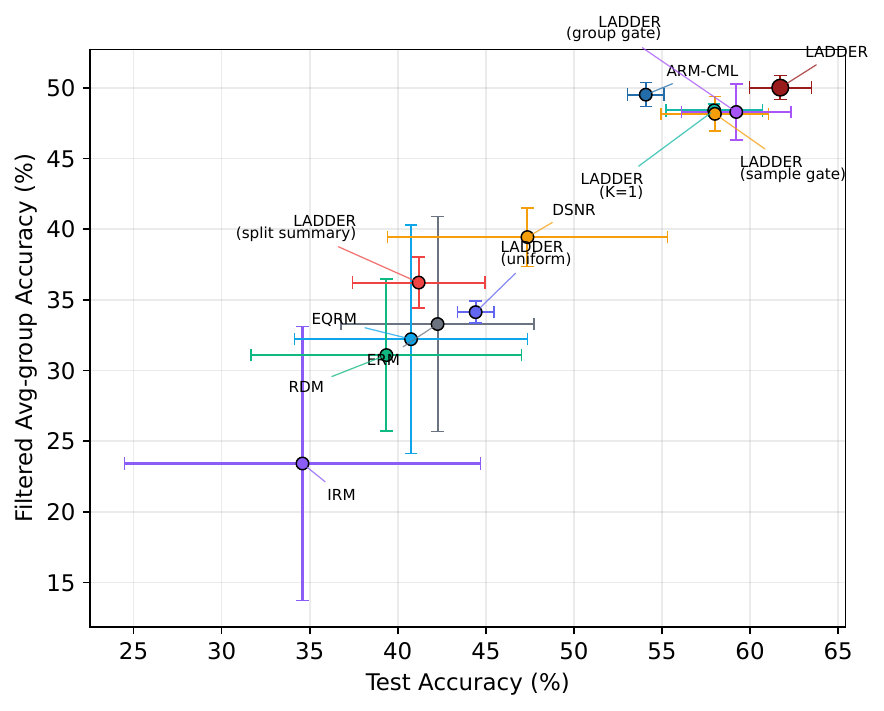} &
\includegraphics[width=0.415\textwidth]{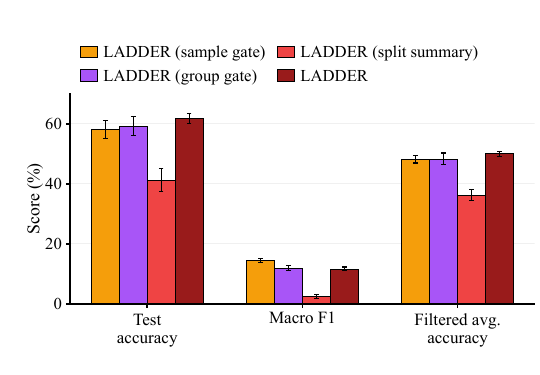}
\end{tabular}
\caption{iWildCam-WILDS accuracy/robustness trade-off (left) and routing/target-summary comparison on test accuracy, macro F1, and filtered average-group accuracy (right), including \method{} (sample gate), \method{} (group gate), \method{} (split summary), and \method{}.}
\label{fig:supp_iwildcam_tradeoff_ablation}
\end{figure}

\begin{figure}[!htb]
\centering
\includegraphics[width=0.85\textwidth]{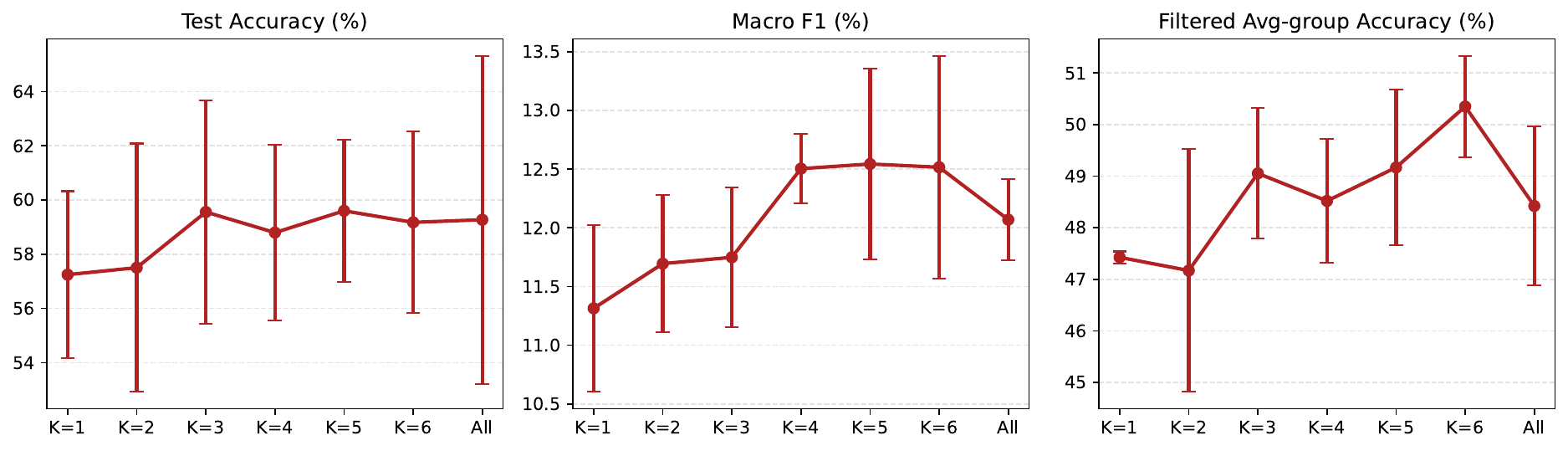}
\caption{iWildCam-WILDS KNN source-selection ablation over three random seeds. The main setting fixes $K=6$, while ``All'' removes the KNN truncation and reweights all retained source domains. Across $K\in\{1,2,3,4,5,6\}$, $K=5$ gives the highest test accuracy and macro F1, whereas $K=6$ gives the strongest filtered average-group accuracy. The sweep uses up to $8192$ covariates per target location; the main table uses five seeds and a cap of $1024$.}
\label{fig:supp_iwildcam_knn_ablation}
\end{figure}

The iWildCam-WILDS KNN ablation evaluates sensitivity to the neighborhood size. The main protocol fixes $K=6$, and the ablation shows that local source selection remains competitive across nearby choices while avoiding the robustness degradation caused by using all source classifiers.

\paragraph{iWildCam worst-location diagnostics.}
This subsection explains why the main text does not emphasize raw worst-location accuracy on iWildCam-WILDS. The target locations are extremely long-tailed, so the minimum over locations is often determined by locations with very few examples. We therefore report the five-seed diagnostic table below to make this choice explicit.

\begin{table}[!htb]
\centering
\caption{iWildCam-WILDS worst-location diagnostics over five random seeds. Raw worst-location accuracy takes the minimum over all target locations; filtered worst-location accuracy and zero-location counts use locations with at least $50$ samples.}
\label{tab:supp_iwildcam_worst}
\fitwidetable{%
\begin{tabular}{lccc}
	\toprule
Method & Raw worst $\uparrow$ & Filtered worst $\uparrow$ & Zero filt. locs. $\downarrow$ \\
\midrule
ERM & $0.00 \pm 0.00$ & $0.20 \pm 0.44$ & $1.4 \pm 1.5$ \\
IRM & $0.00 \pm 0.00$ & $0.00 \pm 0.00$ & $4.0 \pm 5.0$ \\
EQRM & $0.00 \pm 0.00$ & $0.98 \pm 1.39$ & $0.6 \pm 0.5$ \\
DSNR & $0.00 \pm 0.00$ & $1.96 \pm 1.55$ & $0.4 \pm 0.9$ \\
RDM & $0.00 \pm 0.00$ & $0.24 \pm 0.54$ & $1.4 \pm 0.9$ \\
ARM-CML & $0.00 \pm 0.00$ & $1.76 \pm 1.75$ & $0.4 \pm 0.5$ \\
\midrule
\method{} (uniform) & $0.00 \pm 0.00$ & $0.00 \pm 0.00$ & $2.8 \pm 0.8$ \\
\method{} ($K=1$) & $0.00 \pm 0.00$ & $0.00 \pm 0.00$ & $1.0 \pm 0.0$ \\
\method{} (sample gate) & $0.00 \pm 0.00$ & $1.96 \pm 1.83$ & $0.2 \pm 0.4$ \\
\method{} (group gate) & $0.00 \pm 0.00$ & $0.00 \pm 0.00$ & $1.8 \pm 1.3$ \\
\method{} (split summary) & $0.00 \pm 0.00$ & $0.00 \pm 0.00$ & $5.2 \pm 0.4$ \\
\method{} & $0.00 \pm 0.00$ & $0.95 \pm 1.70$ & $0.6 \pm 0.5$ \\
\bottomrule
\end{tabular}}
\end{table}

These diagnostics explain why the main paper emphasizes filtered average-group accuracy rather than worst-location accuracy. Even after filtering locations with fewer than $50$ examples, the worst-location metric remains close to zero for almost all methods, while the average over filtered locations provides a more stable view of target-location robustness. At the $50$-sample threshold, the test split retains $35$ locations and $42{,}477$ examples. As a threshold sensitivity check for \method{}, using thresholds $m_g\in\{25,50,100,200\}$ gives filtered average-group accuracies of $51.49\%$, $50.01\%$, $47.82\%$, and $51.57\%$, respectively. These values describe the sensitivity of \method{} to the filtering threshold; the retained location set changes with the threshold.

\subsubsection{Paired Routing Comparisons}
\begin{table}[!htb]
\centering
\caption{Paired-seed differences for protocol-matched routing controls. Each entry is the first method minus the second in percentage points; intervals are descriptive bootstrap percentile $95\%$ CIs over the five matched seeds. Positive values favor the first method except for ID-to-test drop, where negative values are better.}
\label{tab:supp_paired_routing}
\fitwidetable{%
\begin{tabular}{lllr}
	\toprule
Dataset & Comparison & Metric & Mean $\Delta$ [95\% CI] \\
\midrule
FMoW & \method{} -- \method{} (sample gate) & Test & $+0.21\ [-0.21,0.55]$ \\
FMoW & \method{} -- \method{} (sample gate) & Worst region & $+0.79\ [-0.07,1.85]$ \\
FMoW & \method{} -- \method{} (sample gate) & Average region & $+0.39\ [-0.29,0.91]$ \\
FMoW & \method{} -- \method{} (sample gate) & ID-to-test drop & $-0.14\ [-0.69,0.41]$ \\
\midrule
iWildCam & \method{} -- \method{} (group gate) & Test & $+2.50\ [0.50,4.49]$ \\
iWildCam & \method{} -- \method{} (group gate) & Macro F1 & $-0.38\ [-1.69,0.54]$ \\
iWildCam & \method{} -- \method{} (group gate) & Filtered average & $+1.72\ [-0.16,4.03]$ \\
\bottomrule
\end{tabular}}
\end{table}

\subsection{Implementation Details and Hyperparameters}\label{app:implementation_details}

The theoretical analysis uses $\mathcal{W}_1$ on style distributions as a clean population metric. In the main real-data experiments, we approximate the empirical $W_1$ geometry using the balanced debiased Sinkhorn divergence with Euclidean ground cost $c(\bm z,\bm z')=\|\bm z-\bm z'\|_2$. For each source domain and each target summary unit, we encode at most $m=1024$ covariates with the fixed style encoder using seed-controlled uniform random subsampling. Style samples are normalized by source-training statistics,
\[
\widetilde{\bm z}_s=(\bm z_s-\widehat\mu_{\rm src})/\max\{\widehat\sigma_{\rm src},10^{-6}\},
\]
where $\widehat\mu_{\rm src}$ and $\widehat\sigma_{\rm src}$ are computed only from cached source-training style samples and then applied unchanged to target covariates; the maximum and division are coordinate-wise. The implementation uses GeomLoss
\begin{quote}
\texttt{SamplesLoss(loss="sinkhorn", p=1, blur=0.05,}\\
\texttt{scaling=0.8, debias=True, backend="tensorized")}
\end{quote}
We use balanced OT because each fingerprint is a probability measure, debiasing to remove entropic self-bias, and $p=1$ so the ground cost matches the $W_1$ geometry. The tensorized backend is used for the main runs with $m\le1024$. The displayed theory is for exact empirical $W_1$ and does not include entropic regularization or numerical solver error. The simulation uses a separable empirical quantile proxy. For each style coordinate, it linearly interpolates empirical quantiles at $256$ equally spaced probabilities on $[0,1]$, averages the squared source--target differences over these quantiles, and sums over coordinates. Before source weighting, these distances are divided by the largest source--target distance plus $10^{-8}$. The reported runs have $1000$ source samples per domain, so the quantile grid has $256$ points even when the target summary is smaller. In all cases, the distance estimator is fixed before target evaluation and is computed only from unlabeled target covariates.

\begin{table}[!htb]
\centering
\caption{Implementation settings.}
\label{tab:supp_impl_hparams}
\begin{tabular}{>{\raggedright\arraybackslash}p{\dimexpr0.18\linewidth-2\tabcolsep\relax}>{\raggedright\arraybackslash}p{\dimexpr0.21\linewidth-2\tabcolsep\relax}>{\raggedright\arraybackslash}p{\dimexpr0.305\linewidth-2\tabcolsep\relax}>{\raggedright\arraybackslash}p{\dimexpr0.305\linewidth-2\tabcolsep\relax}}
	\toprule
Setting & Simulation & FMoW-WILDS & iWildCam-WILDS \\
\midrule
Backbone & Linear latent SCM & ImageNet ResNet-50 & ImageNet ResNet-50 \\
Input resolution & -- & $224\times224$ & $224\times224$ \\
Optimizer / scheduler & Adam / none & AdamW / none & AdamW / none \\
Learning rates & $10^{-2}$ & representation $10^{-4}$, source classifiers $10^{-3}$ & representation $10^{-4}$, source classifiers $10^{-3}$ \\
Weight decay & $10^{-2}$ & $10^{-4}$ & $10^{-4}$ \\
Epochs / steps & $150$ epochs & $25$ rep. epochs, $8$ head epochs, $800$ steps/epoch & $20$ rep. epochs, $6$ head epochs, $700$ steps/epoch \\
Batching & $1000$ samples/domain & batch $32$, eval batch $64$, $4$ domains/step & batch $32$, eval batch $64$, $4$ domains/step \\
Dimensions & $d_c=10,d_s=8$ & $d_c=512,d_s=128$ & $d_c=512,d_s=128$ \\
Loss weights & -- & $(\lambda_s,\lambda_a,\lambda_o,\lambda_r)=(0.15,0.03,5\!\times\!10^{-4},10^{-5})$ & same as FMoW \\
Source classifiers & source GLMs & frozen-encoder linear classifiers, CE & frozen-encoder linear classifiers, class-balanced CE \\
Source filtering & $E=9$ fixed & year domains, min size $512$ & largest $64$ locations, min size $512$ \\
Main $(K,\tau)$ & $(4,0.1)$ & $(5,0.5)$ & $(6,1.0)$ \\
Target summary & target latent samples & split-level style summary & location-level style summary \\
Style distance & coord. quantile $W_2^2$ proxy & debiased Sinkhorn, $p=1$, $m\le1024$ & debiased Sinkhorn, $p=1$, $m\le1024$ \\
\bottomrule
\end{tabular}
\end{table}

FMoW images are resized to $224\times224$. iWildCam uses random resized cropping and horizontal flipping during training, and resizing to 256 followed by a $224\times224$ center crop for evaluation; both datasets use ImageNet normalization.

Representation training uses batches containing equal numbers of examples from four sampled source domains. Classification losses use a mean reduction, with the weighted label-loss normalization for iWildCam specified below. The orthogonality loss averages squared entries of the minibatch cross-covariance matrix, and each branch's KL term is averaged over examples and latent coordinates before the two branch terms are added. The loss weights in Table~\ref{tab:supp_impl_hparams} use these averaging conventions.

The reported implementation uses $\lambda_a$ both as the outer adversarial-loss weight and inside the gradient-reversal layer, so the encoder-side reversed gradient is scaled by $\lambda_a^2=9\times10^{-4}$ for $\lambda_a=0.03$; this fixed protocol is used across all reported seeds and datasets. For iWildCam, class-balanced CE uses effective-number class weight $a_y=(1-\beta_{\rm cb})/(1-\beta_{\rm cb}^{n_y})$ with $\beta_{\rm cb}=0.999$, multiplying the ordinary CE for class $y$ \citep{cui2019classbalanced}. Here $n_y$ is the class count pooled over the retained source training set; the source classifiers reuse these weights. The weighted CE sums the weighted losses and divides by the sum of class weights in the minibatch.

We manually selected $K$ and $\tau$ using validation worst-region accuracy on FMoW and validation macro F1 on iWildCam, following the respective WILDS metrics, and kept them fixed for test evaluation. The reported main settings are $(K,\tau)=(4,0.1)$ in simulation, $(5,0.5)$ on FMoW, and $(6,1.0)$ on iWildCam. Simulation uses $K=4$ for the main comparison and reports $\tau\in\{0.05,0.1,0.5,1,5\}$ as a controlled sensitivity analysis over target-summary reliability. The earlier real-data Sinkhorn KNN sweeps use a summary sampling cap of $8192$ covariates per domain, compared with $1024$ in the main runs. Within each sweep, $K$ varies while the other settings are fixed: FMoW uses $K\in\{1,2,3,5,\mathrm{All}\}$ with $\tau=0.5$, and iWildCam uses $K\in\{1,2,3,4,5,6,\mathrm{All}\}$ with $\tau=1.0$. For the main real-data runs, the Sinkhorn blur is fixed to $0.05$ after source-only style normalization. Training follows the fixed source-data schedules in Table~\ref{tab:supp_impl_hparams}, and the resulting final models are evaluated on each split.

The real-data baseline settings use IRM penalty weight $0.1$ with $1000$ annealing updates, EQRM quantile $0.75$ with $200$ burn-in updates, and RDM weight $1$ with variance weight $0.1$ and $200$ annealing updates. RDM uses RBF factors $\{0.001,0.01,0.1,1,10\}$. The DSNR index has dimension $8$; its MLPs use hidden dimension $512$, ReLU activations, and dropout $0.1$.

ARM-CML uses a three-layer $5\times5$ convolutional context network with hidden width $64$ and $12$ output channels, concatenated with RGB inputs before the ResNet-50 predictor. The added input-channel kernels copy the pretrained RGB kernels cyclically. Each source episode contains $8$ examples from each of $4$ sampled domains (batch $32$). AdamW, learning rate $10^{-4}$, weight decay $10^{-4}$, the dataset-specific losses, and the $25\times800$/$20\times700$ training schedules match the respective real-data settings above. Target contexts use at most $1024$ unlabeled examples per split or location. In simulation, the context and prediction networks are two-hidden-layer MLPs of width $128$, with a $16$-dimensional context and source-only input standardization; AdamW uses learning rate $10^{-3}$ and weight decay $10^{-4}$. Training uses $150$ epochs of $45$ episodes, each containing $50$ examples from each of $4$ source domains. All ARM-CML results use the final source-trained model. The implementation is adapted from the authors' MIT-licensed code (release\_2021, commit \texttt{e4a0ce2}).

\subsection{Runtime and Storage}

\begin{table}[!htb]
\centering
\caption{Wall-clock runtime in minutes per seed from available runtime logs. Ordinary baseline rows use the retained original runtime logs; source-classifier routing rows use the Sinkhorn reruns over five seeds.}
\label{tab:supp_runtime}
\begin{tabular}{lcc}
	\toprule
Method & FMoW-WILDS & iWildCam-WILDS \\
\midrule
ERM & $59.0 \pm 1.4$ & $54.7 \pm 1.2$ \\
IRM & $60.7 \pm 1.2$ & $57.0 \pm 1.1$ \\
EQRM & $58.3 \pm 1.0$ & $54.1 \pm 0.7$ \\
DSNR & $60.6 \pm 1.3$ & $56.9 \pm 1.1$ \\
RDM & $60.6 \pm 1.1$ & $57.8 \pm 3.8$ \\
ARM-CML (training only) & $99.4 \pm 0.9$ & $79.5 \pm 2.6$ \\
\midrule
\method{} (uniform) & $85.2 \pm 1.4$ & $102.4 \pm 5.4$ \\
\method{} ($K=1$) & $85.2 \pm 1.0$ & $104.2 \pm 3.3$ \\
\method{} (sample gate) & $85.0 \pm 1.0$ & $101.3 \pm 3.1$ \\
\method{} (group gate) & -- & $101.2 \pm 3.1$ \\
\method{} & $85.2 \pm 1.0$ & $104.1 \pm 3.2$ \\
\method{} (split summary) & -- & $115.5 \pm 5.1$ \\
\bottomrule
\end{tabular}
\end{table}

The reported LADDER runtimes include representation training, source-classifier fitting, style-summary construction, and target evaluation. ARM-CML reports the sum of its recorded epoch training times; its original evaluation times were not saved, so this row is not an end-to-end runtime comparison.

\begin{table}[!htb]
\centering
\caption{Approximate routing storage beyond the shared representation model: source-classifier weight matrices and cached style samples, using float32, $d_c=512$, $d_s=128$, and $m=1024$ samples per source domain. Classifier biases are omitted.}
\label{tab:supp_storage}
\begin{tabular}{lccc}
	\toprule
Dataset & \shortstack{Classifier\\weights} & \shortstack{Style\\fingerprints} & \shortstack{Total routing\\storage} \\
\midrule
FMoW-WILDS ($E=11,C=62$) & $1.33$ MiB & $5.50$ MiB & $6.83$ MiB \\
iWildCam-WILDS ($E=64,C=182$) & $22.75$ MiB & $32.00$ MiB & $54.75$ MiB \\
\bottomrule
\end{tabular}
\end{table}

For linear source classifiers and $m$ cached style samples per source domain, these weight matrices and fingerprints occupy
$4E C d_c+4E m d_s$ bytes. The target-domain style summary is a temporary prediction buffer and is not part of persistent model storage. Source covariates are not stored after the cached style samples are constructed.

\FloatBarrier
\end{document}